\documentclass[twoside,11pt]{article}
\pdfoutput=1

\usepackage[preprint]{jmlr2e}

\usepackage[title]{appendix}
\usepackage[ruled,vlined]{algorithm2e}
\usepackage{hyperref}
\usepackage{svg}
\usepackage{xcolor}
\usepackage{natbib}
\usepackage{amsmath}
\usepackage{bm}
\usepackage{amssymb}
\usepackage{amsfonts}
\usepackage{mathrsfs}
\usepackage{upgreek}
\usepackage{xcolor}
\usepackage{enumitem}
\usepackage{url}
\usepackage{graphicx}
\usepackage{mathtools}
\usepackage[normalem]{ulem}
\usepackage{quiver}

\newcommand{\underformula}[2]{ \underset{#2}{\underbrace{#1}} }

\DeclareMathOperator*{\ed}{\text{ED}}

\DeclareMathOperator*{\poi}{\text{Poisson}}
\DeclareMathOperator*{\gam}{\text{Gamma}}

\DeclareMathOperator*{\var}{\mathbb{V}}
\DeclareMathOperator{\V}{\mathbb{V}}

\DeclareMathOperator*{\cov}{\text{Cov}}
\DeclareMathOperator*{\simiid}{\overset{\text{iid}}{\sim}}
\DeclareMathOperator*{\eqdef}{\overset{\text{def}}{=}}
\DeclareMathOperator{\E}{\mathbb{E}}

\newcommand{\discrepancy}{\text{d}}
\newcommand{\T}{\text{T}}
\newcommand{\target}{\T^{*}}

\usepackage{letltxmacro}
\LetLtxMacro{\originaleqref}{\eqref}
\renewcommand{\eqref}{Eq.~\originaleqref}

\newcommand{\z}{\theta}
\newcommand{\Z}{\Theta}
\newcommand{\ZSPACE}{\Uptheta}
\newcommand{\YSPACE}{\Upupsilon}
\definecolor{orcidlogocol}{HTML}{A6CE39}
\renewcommand{\z}{z}
\renewcommand{\Z}{Z}
\renewcommand{\ZSPACE}{\mathscr{Z}}
\renewcommand{\YSPACE}{\mathscr{Y}}

\usepackage[normalem]{ulem}

\jmlrheading{0}{2022}{00-00}{07/22}{00/00}{000}{Eliezer de Souza da Silva, Tomasz~Kuśmierczyk, Marcelo~Hartmann and Arto~Klami}

\ShortHeadings{Prior Specification for Bayesian Matrix Factorization}{da Silva et al} %
\firstpageno{1}

\begin{document}
\title{Prior Specification for Bayesian Matrix Factorization via Prior Predictive Matching %
}
\author{\name Eliezer~de~Souza~da~Silva \email eliezer.souza.silva@ntnu.no \\
        \addr Department of Computer Science\\ Norwegian University of Science and Technology, Trondheim, Norway
        \AND
        \name Tomasz~Kuśmierczyk \email tomasz.kusmierczyk@helsinki.fi \\
        \addr
        Department of Computer Science \\
        University of Helsinki, Helsinki, Finland
        \AND
        \name Marcelo~Hartmann \email marcelo.hartmann@helsinki.fi \\
        \addr
        Department of Computer Science \\
        University of Helsinki, Helsinki, Finland
        \AND
        \name Arto~Klami \email arto.klami@helsinki.fi \\
        \addr
        Department of Computer Science \\
        University of Helsinki, Helsinki, Finland
}

\editor{}

\maketitle

\begin{abstract}%
The behavior of many Bayesian models used in machine learning critically depends on the choice of prior distributions, controlled by some hyperparameters that are typically selected by Bayesian optimization or cross-validation. This requires repeated, costly, posterior inference. We provide an alternative for selecting good priors without carrying out posterior inference, building on the prior predictive distribution that marginalizes out the model parameters. 
We estimate virtual statistics for data generated by the prior predictive distribution and then optimize over the hyperparameters to learn ones for which these virtual statistics match target values provided by the user or estimated from (subset of) the observed data.
We apply the principle for probabilistic matrix factorization, for which good solutions for prior selection have been missing. We show that for Poisson factorization models we can analytically determine the hyperparameters, including the number of factors, that best replicate the target statistics, and we study empirically the sensitivity of the approach for model mismatch. We also present a model-independent procedure that determines the hyperparameters for general models by stochastic optimization, and demonstrate this extension in context of hierarchical matrix factorization models.
\end{abstract}

\section{Introduction}\label{sec:intro}

Bayesian machine learning (ML) builds on the idea of  general-purpose hierarchical probabilistic models, such as mixture models, topic models or matrix factorization techniques, that can be used for various learning tasks. Often these models have large number of latent variables, such as the cluster indicators for mixture models or the topic proportions for topic models, for which the prior distributions are chosen predominantly by computational convenience.
The choice of the particular priors, as well as other hyperparameters like the number of factors or components, has notable impact on the overall performance of these models, but cannot be determined by classical statistical modeling principles: the priors may not have intuitive interpretation and for complex hierarchical models the relationship between the priors and data is poorly understood, ruling out subjective prior knowledge.
For example, the behavior of the hierarchical Poisson matrix factorization model of \citet{gopalan2015scalable} depends on seven hyperparameters (six for defining the priors and one for the number of latent factors) in a non-trivial manner.

The hyperparameters are typically chosen heuristically or by an iterative process that explicitly evaluates the quality of multiple choices.
The search can be automated with Bayesian optimization \citep{DBLP:conf/NeurIPS/SnoekLA12}, but evaluating the quality is costly. It is typically based on some proxy of the marginal likelihood, such as variational lower bound or leave-one-out cross validation \citep{DBLP:journals/sac/VehtariGG17a}, or directly on the performance in a downstream task, such as recommendation \citep{DBLP:conf/wcgo/GaluzziGCPA19}.
Both require carrying out posterior inference for every considered set of hyperparameters,
adding significant computational burden and increasing the overall training time by orders of magnitude.
Furthermore, the result is only optimal for the chosen measure and inference method, unnecessarily tying model specification with inference.

To overcome this, we turn attention to the statistical literature on the \emph{prior predictive distribution} (PPD) -- the marginal distribution of observables before seeing any data. The PPD is routinely used during the statistical modeling pipeline in form of prior predictive checks, to qualitatively access whether the model and the priors are reasonable \citep{schad:2019,gabry2019visualization,gelman2020bayesian}. PPD has also been used for prior elicitation, to convert knowledge an expert has on the properties of data into prior distributions \citep{kadane:1980,akbarov:2009,hartmann2020}. 
We turn those ideas into a tool for automatic learning of hyperparameters, by directly optimizing for a good match between  \textit{virtual statistics} of the PPD and statistics of the data\footnote{In order to distinguish from statistics of the observed data, we use the term \textit{virtual statistics} to refer to summary quantities calculated for hypothetical data sampled from PPD, inspired by the phrase  \textit{virtual counts} used sometimes for the hyperparameters in count-data models.}. 
The tool can be used in two ways: (1) the target statistics are provided by the expert (user) as prior knowledge on data, or (2) the target statistics are estimated from (subset of) the actual data. The former is related to use of PPD for prior elicitation \citep{kadane:1980}, extended here for practical use with Bayesian ML models with large number of latent variables, whereas the latter is similar to empirical Bayes \citep{Casella:1985}.

The proposed \emph{prior predictive matching} approach, described in Section~\ref{sec:concept}, finds good hyperparameters without needing posterior inference. When the true data generating process is within the assumed model family, the approach provides hyperparameters that are optimal with respect to the selected statistics, and we show empirically that the method is robust to small model misspecification. If the data fits poorly the assumed model family, the approach may return unreasonable choices, which can be interpreted as sign of model mismatch and a need for remodeling. For instance, we show how priors not accounting for non-homogeneous row and column counts in recommender engine setups result in clear underestimation of the number of factors, which could be fixed by changing to a formulation that is consistent with the observed margins \citep{yildirim2021}.

The principle of prior predictive matching is generally applicable for all statistical models used in machine learning that specify a sampling distribution. In this work we introduce it in the context of probabilistic matrix factorization models, focusing in particular on models for count data. This provides a tangible technical context for the work and enables showcasing how the specific instances of the general principle can provide particularly elegant solutions for practical models of interest.
First, we consider scenarios where certain moments of the PPD can be expressed analytically. We can then compute the virtual statistics in closed form and analytically solve for optimal hyperparameters. The method can be interpreted as a specific instance of the method of moments \citep{casella, pawitan} since we determine the parameters by matching the moments, but applied for learning the hyperparameters as an intermediate step in the full modelling process, rather than directly as means of approximate inference (once the hyperparameters are fixed, we still want to perform standard posterior inference). The approach facilitates computationally efficient selection of the hyperparameters, demonstrated in this work for count matrix factorization models for which we can set both the prior parameters and the number of factors in closed form. Previously, the latent dimensionality could be set automatically only for Gaussian matrix factorization, based on the analytic marginal likelihood by \citet{Bouveyron2019}, and now we can do it also for count models, although based on selected statistics rather than the marginal likelihood.

Closed-form analytic expressions are computationally ideal, but their derivation is tedious already for fairly simple models. To address this, we also provide a model-independent stochastic algorithm that uses sampling to compute the virtual statistics and stochastic gradient optimization to learn the hyperparameters. It only requires us to be able to sample from PPD and is applicable for all models with reparameterizable priors \citep{figurnov2018implicit, mohamed2019monte}. The method is loosely related to several recent methods directly learning a prior for flexible neural network -based models \citep{tomczak2018vae,Klushyn2019priors,nalisnick2018learning,nalisnick2021predictive,ducontrolled}, but our goal is specifically in determining the hyperparameters of existing priors and we obtain the solution without considering the observed data itself. We demonstrate also this approach in context of Bayesian matrix factorization models, this time considering a more complex hierarchical model for which we do not have analytic expressions for the moments.

\section{Motivation: Priors for Bayesian MF}
\label{sec:motivation}

Bayesian matrix factorization (BMF) is an important class of Bayesian ML models used, e.g., in recommender engines \citep{10.5555/2981562.2981720}, for dimensionality reduction \citep{10.1145/2433396.2433438, 10.1145/860435.860485}, community detection \citep{PhysRevE.83.066114}, and modeling relationships between data modalities \citep{Klami13jmlr}.
Importantly, it is a family for which the prior distributions are difficult to specify, as will be clarified in Section~\ref{sec:onpriors}.
We start by characterizing two concrete models building on Poisson distribution, for which the effect is emphasized~\citep{DBLP:journals/cin/Cemgil09}.

\textbf{Poisson Matrix Factorization} (PMF)~\citep{DBLP:journals/cin/Cemgil09,DBLP:conf/NeurIPS/GopalanCB14} with latent dimensionality $K$ specifies a generative model for a matrix $ \mathbf{Y} =\{ Y_{ij} \} \in \mathbb{R}^{N \times M}$, with each entry $Y_{ij}$ following a Poisson distribution with rate $\theta_{ik}\beta_{jk}$, a product of latent factors $\theta_{ik}$ indexed by the rows and  $\beta_{jk}$ indexed by the columns.

Each latent variable follows a prior $F(\mu,\sigma^2)$, parameterized here using mean $\mu$ and standard deviation $\sigma$
\begin{align}\label{eq:genmf}
    \theta_{ik} &\simiid F(\mu_{\theta},\sigma_{\theta}^2), \quad   \beta_{jk} \simiid F(\mu_{\beta},\sigma_{\beta}^2), \nonumber \\
     Y_{ij} & \simiid \poi \left(\sum_{k = 1}^K \theta_{ik}\beta_{jk}\right).
\end{align}
The majority of the PMF literature assumes the priors to be gamma distributions (using shape-rate parameterization, $F(\mu_{\theta},\sigma_{\theta}^2)=\gam(a,b)$ and $F(\mu_{\beta},\sigma_{\beta}^2)=\gam(c,d)$, with $\mu_{\theta}=\frac{a}{b}$, $\sigma_{\theta}^2 = \frac{a}{b^2}$, $\mu_{\beta}=\frac{c}{d}$ and $\sigma_{\beta}^2 = \frac{c}{d^2}$) for efficient posterior inference, but we use the more general notation to extend the analysis to all scale-location priors.

The priors and the number of factors $K$ control the \textit{sparsity} and \textit{magnitude} of the latent representation~\citep{DBLP:journals/cin/Cemgil09}, via the expected mean and variance of the rates. However, these effects are hard to separate from each other and many practitioners are unaware of the implications of their choices. As a practical example, already the common choice of independent priors for $\theta$ and $\beta$ (used also in our work) encodes the prior assumption that the rows (and, equivalently, the columns) are exchangeable and hence also that the margin counts are equal in expectation. Allowing for vastly varying row (and/or column) sums would hence require switching to a more general allocation models \citep{yildirim2021,cemgil2019} or more structured priors. In practice, the quality of the prior choices often becomes apparent for the practitioner only in context of the observed data \textit{a posteriori}.

\textbf{Compound Poisson Matrix Factorization} (CPMF) \citep{BasbugE16,simsekli2013} extends PMF by incorporating an additive exponential dispersion model (EDM) \citep{JorgensenEDM} in the observation model, while keeping the Poisson-Gamma factorization structure: 
\begin{align}
    \theta_{ik} &\sim F(\mu_{\theta},\sigma_{\theta}^2), \quad \beta_{jk} \sim F(\mu_{\beta},\sigma_{\beta}^2) \nonumber \\ 
    Y_{ij} & \sim \ed(w,\kappa N_{ij}), \quad N_{ij} \sim \poi(\sum_{k = 1}^K \theta_{ik}\beta_{jk}),
\end{align}
where we have $p(Y_{ij}| N_{ij} ; w,\kappa )=\text{exp}(Y_{ij}w-\kappa N_{ij}\psi(w))h(Y_{ij},\kappa N_{ij})$, $\E[Y_{ij} | N_{ij} ; w,\kappa ]=\kappa N_{ij} \psi'(w)$ and $\var[Y_{ij} |  N_{ij}; w,\kappa]=\kappa N_{ij} \psi{''}(w)$~\footnote{We denote $\psi'(w)=\frac{d \psi}{dw}$, $\psi''(w)=\frac{d^2 \psi}{dw ^2 }$ }, and
$n_{ui}$ is a Poisson distributed latent count. $\ed( w,\kappa N_{ij} )$ represents a distribution from the family of exponential dispersion model, with natural parameter given by $w$ and dispersion given by $\kappa N_{ij}$, and the particular distribution determined by the base log-partition function $\psi(w)$ and base-measure $ h(Y_{ij},\kappa N_{ij} )$. This model family includes Normal, Poisson, Gamma, Inverse-Gamma, and many other distributions (see Table 1 in \citealt{BasbugE16}).

The data generating distribution is influenced by both the chosen EDM distribution and the hyperparameters, now including also $\kappa$ and $w$, and a precise a priori reasoning about their joint effect is beyond feasible even for well-versed practitioners. An intuitive view of this model is that it allows us to decouple the sparsity or dispersion from the response model (controlled by the choice of distribution to be compounded). In this sense $\kappa$ would give an indication about variability of the responses, while $w$ would be related to the natural parameterization of the response distribution.
Determining specific values for these parameters to achieve desired or expected characteristics for the data is, however, difficult.

\textbf{Generic Bayesian Matrix Factorization} is a generalized template model, representing the family of matrix factorization models with distributions for the priors $F(\mu_{\theta},\sigma_{\theta}^2)$ and $ F(\mu_{\beta},\sigma_{\beta}^2) $, and observation model $ F_Y$. 
\begin{align}\label{eq:generic_pmf}
    \theta_{ik} &\sim F(\mu_{\theta},\sigma_{\theta}^2), \quad \beta_{jk} \sim F(\mu_{\beta},\sigma_{\beta}^2) \nonumber \\ 
    Y_{ij} & \sim F_Y(\sum_{k = 1}^K \theta_{ik}\beta_{jk}), \text{ with } \E[Y_{ij}]=\sum_{k = 1}^K \theta_{ik}\beta_{jk}.
\end{align}

The choice of specific parametric distribution for each part of the model leads specific instantiations of matrix factorization, for example if parameters and observations are normally distributed, with $F_Y = \mathcal{N}$,  $F(\mu_{\theta},\sigma_{\theta}^2)=\mathcal{N}(\mu_{\theta}=0,\sigma_{\theta}^2)$, and $F(\mu_{\beta},\sigma_{\beta}^2)=\mathcal{N}(\mu_{\beta}=0,\sigma_{\beta}^2)$, we obtain the classic probabilistic matrix factorization \citep{10.5555/2981562.2981720}; with Poisson distributed observations and Gamma distributed latent variables, $F_Y = \poi$,  $F(\mu_{\theta},\sigma_{\theta}^2)=\gam(a,b)$ and $F(\mu_{\beta},\sigma_{\beta}^2)=\gam(c,d)$, we obtain (PMF)~\citep{DBLP:journals/cin/Cemgil09,DBLP:conf/NeurIPS/GopalanCB14}. We use this generic formulation to obtain results that are valid for any choice of parametric distributions for prior and observation distributions.

\subsection{On Priors for BMF}
\label{sec:onpriors}

Despite the vast literature on BMF models and their inference algorithms, many of the papers developing or applying BMF models treat the prior choice lightly and resort to fairly heuristic choices.
For example, \citet{Brouwer2017c} present a compendium of BMF models with wide range of likelihoods and priors, but despite focusing on small data applications still set the hyperparameters
by either trying values from a regular grid or setting them to fixed values based on claims of insensitivity.
Similarly, \citet{gopalan2015scalable} solved the problem of choosing  the hyperparameters for hierarchical PMF by setting all of them to $1$ or $0.3$ to encourage sparsity, and \citet{BasbugE16} used a combination of empirically tested and heuristically chosen hyperparameter values for CPMF. Finally, \citet{DBLP:journals/pami/TanF13} optimized for the latent dimensionality $K$, but 
used ah-hoc values for other hyperparameters. Importantly, we stress that these examples should not be seen as weakness in these particular works, but rather as examples of a common practice motivating our research -- we have also published articles where we followed the same convention. We also point out that some works do take the prior choice much more seriously. In particular, \citet{cemgil2019} discusses the prior choices in detail in context of a broader class of count allocation models and provides valuable information not only for PMF but also for other models.

We argue that the common practice of heuristic choices is not because BMF models are particularly insensitive to the priors, but because selecting them is not easy. 
In many cases the PMF model is considered as a generic model family that is not necessarily motivated by well-understood generative description for the data, which makes specifying subjective prior knowledge difficult.
For example, often the only prior information for setting the variance parameters $\sigma^2$ of \eqref{eq:genmf} would be based on what kind of values have worked before when applying the model for other data sets, rather than an expert being able to somehow quantify a real subjective knowledge based on domain knowledge. Furthermore, the hyperparameters  are typically not even identifiable due to the latent variables relating to data only via the product $\theta_i^T \beta_j$. In general, this is a characteristic of hierarchical, high-dimensional and complex Bayesian models, where the interplay between prior specification and the final model properties is difficult to intuit aprioristically and can only be understood in connection to the likelihood and the predictions that come from it, as is argued by~\citet{DBLP:journals/entropy/GelmanSB17}.

For machine learning applications, various hyperparameters are today often chosen using a global optimization routine, such as Bayesian optimization \citep{DBLP:conf/NeurIPS/SnoekLA12} or a simple grid search.
Global optimization for BMF, howerer, is hard because (1) training/validation split is non-trivial for structured data, (2) posterior inference is slow for large data, and (3) the optimization surface is difficult. Figure~\ref{fig:prior_posterior} illustrates the last point for the PMF model by evaluating the predictive quality of the mean-field variational approximation
as measured by PSIS-LOO \citep{DBLP:journals/sac/VehtariGG17a} for a range of hyperparameter choices (see Section~\ref{sec:experiment_bo} for details). There are large regions of inappropriate choices and the sharp border between those and the feasible region makes global optimization hard. Furthermore, as illustrated on the right-most bottom plot in Figure~\ref{fig:prior_posterior} (an example of a 1D slice plotted using three different scalings for the y-axis), the optimization surface characteristics depend on the inspection scale; the small neighborhood with optimal scores is lost at coarses scales and is difficult to find with most optimization strategies. In particular, grid-based evaluation routines are likely to miss the good hyperparameters completely.

In Section~\ref{sec:model_specific}, we will describe analytic solution for learning the hyperparameters for this model by matching the virtual statistics of PPD with target values. The example already illustrates how that method, which here does not require any computation besides simple equations, finds a solution surface within the feasible region, selecting also the latent dimensionality automatically. It does not give the hyperparameters optimal for this specific evaluation metric and inference algorithm, nor should it. Instead, the result captures essential properties of the data with the PPD, resulting in an appropriate starting point for posterior inference.

\begin{figure}[t]
    \centering
      \includegraphics[width=0.48\columnwidth]{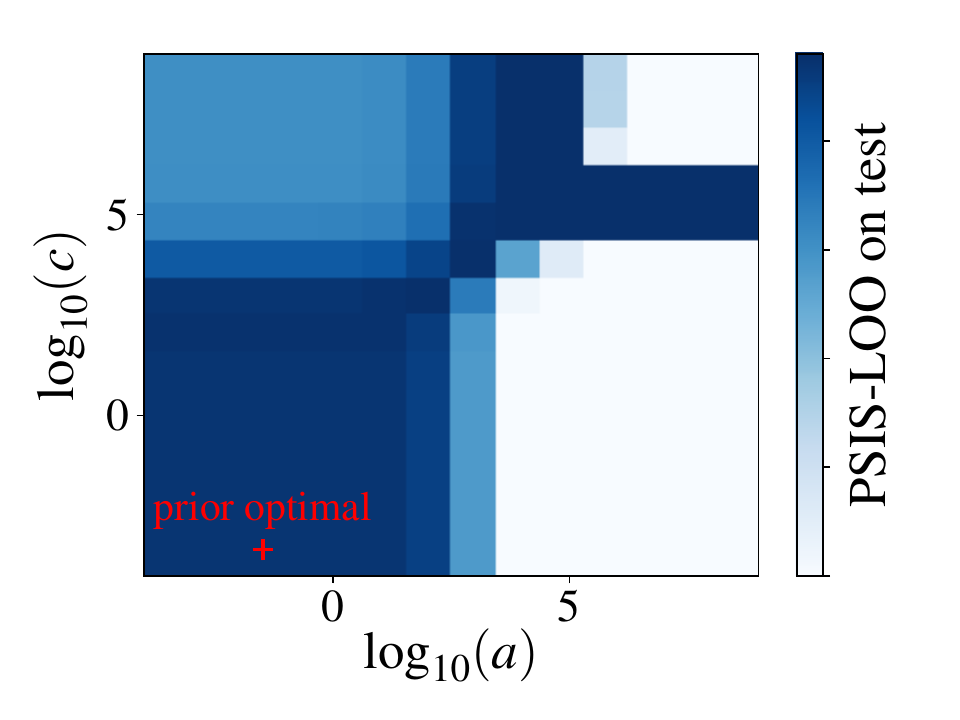}
      \includegraphics[width=0.48\columnwidth]{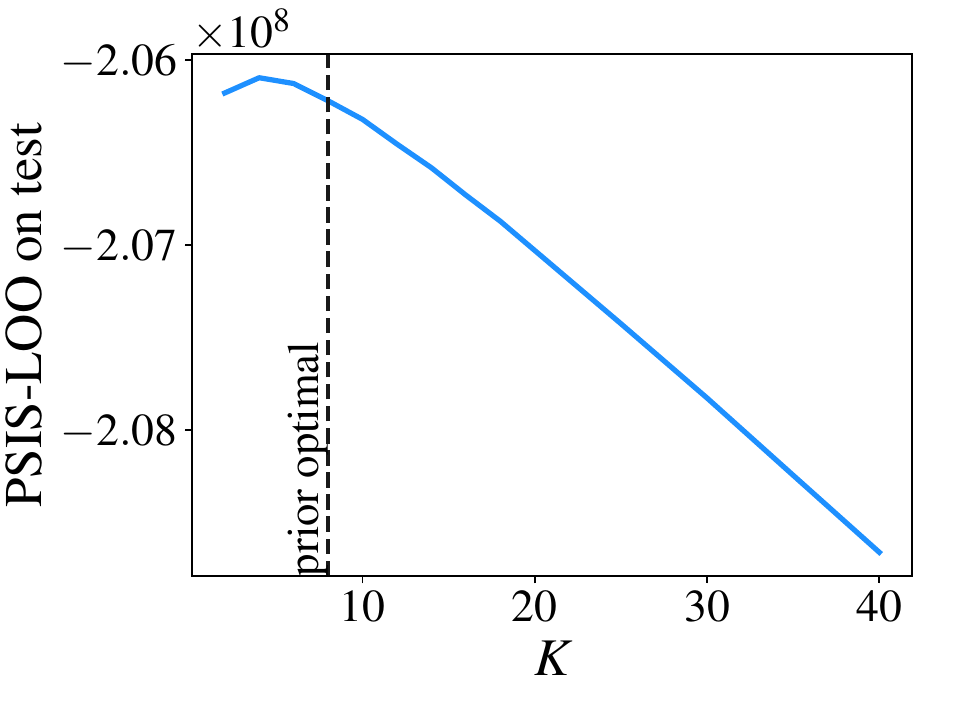}
      
    \includegraphics[width=0.48\textwidth]{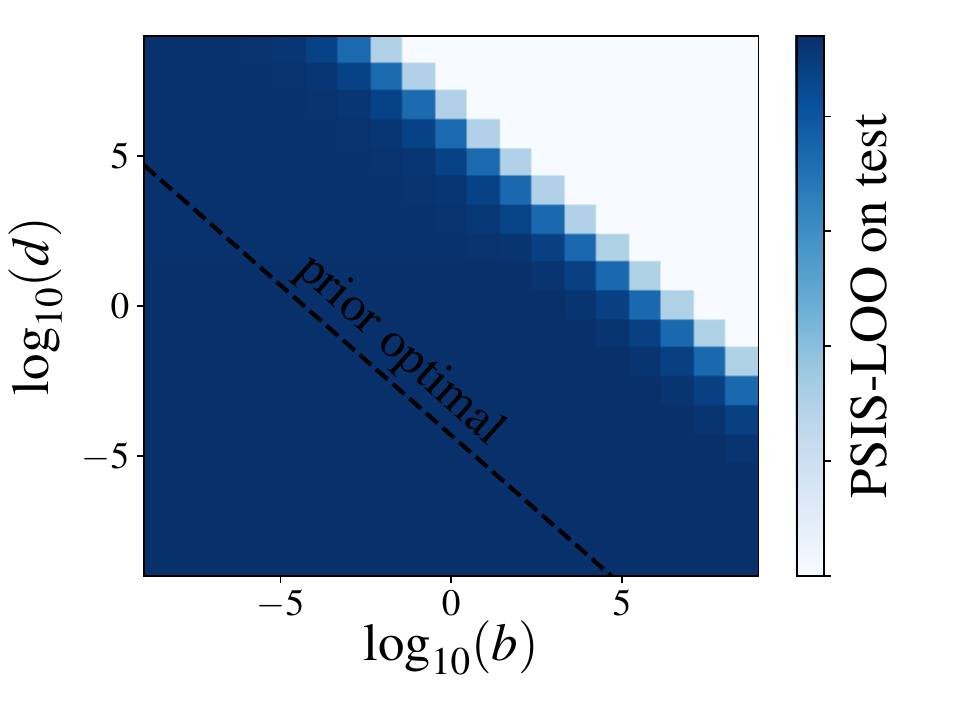}
    \includegraphics[width=0.48\textwidth]{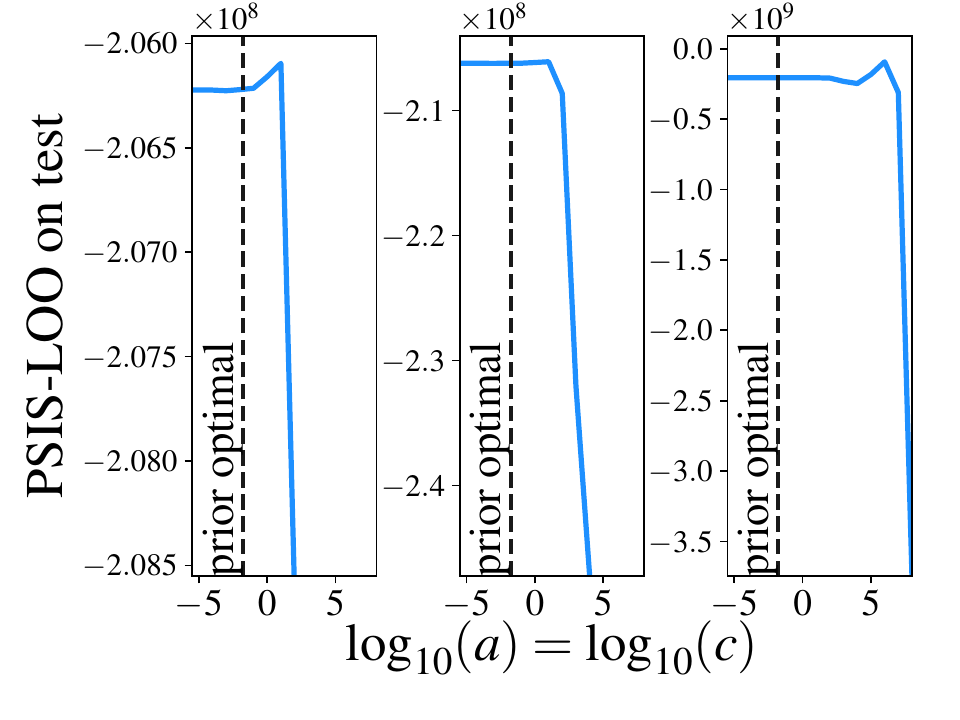}
      
    \caption{Illustration of difficulty of selecting good priors for Poisson matrix factorization, evaluated by predictive quality of a variational approximation on the \texttt{hetrec-lastfm} dataset. We show 2D (left) and 1D (right) slices of the loss surface in the five-dimensional hyperparameter space, with all other values fixed to prior optimal ones.
    The proposed prior predictive matching approach provides closed-form solution (indicated by "prior optimal"), including the latent dimensionality $K$ (top right), within the area of reasonable values.
    }   
    \label{fig:prior_posterior}
\end{figure}

\section{Prior Predictive Matching}
\label{sec:concept}

Having illustrated the challenges with BMF, we now proceed to provide a new method for determining hyperparameters of any Bayesian machine learning model, 
building on the idea of matching virtual statistics of prior predictive distribution as explained in detail below. The principle is presented from the perspective of general Bayesian models, but in this work it is demonstrated  only in the context of matrix factorization.

\subsection{General Idea}

Our goal is to select good hyperparameters $\lambda$ for a probabilistic model $p(Y,Z; \lambda)$, where $Z$ denotes actual model parameters and latent variables collectively, without directly computing the posterior quality of any particular model fit. That is, we want to avoid costly and potentially difficult global optimization requiring the %
selection of specific evaluation criterion for the quality of the final solution
as well as
training/validation split for the data.
Instead, we prefer to optimize an overall match between the model and the data characteristics.

To achieve this, we consider the prior predictive distribution
\[
p(Y;\lambda) = \int p(Y|Z;\lambda)p(Z;\lambda) dZ,
\]
which integrates out the parameters, and search for hyperparameters for which it matches the data distribution well. 
PPD is typically used for validating prior and modeling choices as part of the statistical modeling pipeline \citep{schad:2019, gelman2020bayesian}, often by visual comparison of prior predictive samples and the data, e.g., so that large deviation between the two is interpreted as an indication that the model should be modified~\citep{gabry2019visualization}. We extend the idea 
to automate the prior choice, by optimizing for $\lambda$ for which virtual statistics of PPD match sufficiently well with either prior knowledge of the user or empirical statistics for the available data.
Our goal is specifically to find a point estimate for $\lambda$, to be used for defining the prior for subsequent posterior inference and additional processing steps, and in this work we ignore potential direct prior information on $\lambda$ itself, assuming all choices are equally likely \emph{a priori}.

Note that for some models, such as Gaussian processes with conjugate likelihoods, the PPD can be expressed analytically and the optimal solution is then obtained by directly maximizing the marginal likelihood for the observed data. This approach is, however, restricted to exponential family models with conjugate priors, or requires significant model-specific effort that does not generalize over even minor variants of the model; see \citet{Bouveyron2019} for derivation of marginal likelihood for a specific Gaussian MF model. Our interest is in more general model classes where the PPD itself cannot be expressed in close form. Instead, we assume only that we can either (a) compute some low-order statistics of the PPD analytically, or (b) draw samples from the PPD. Our key contribution is providing an automatic process for selecting the hyperparameters for these broader classes of models in a computationally efficient manner. The practical method, described next, is related to the method of moments (see Section~\ref{sec:relatedwork} for more details), but is applied here for the purpose of learning the hyperparameters as an intermediate step in the modelling workflow, rather than alternative for posterior inference.

\subsection{Method}
\label{sec:usecases}

The gist of our proposed method is to search for $\lambda$ such that the PPD $p(Y;\lambda)$ and the true data distribution $p^{\dagger}(Y)$ or user's prior beliefs about $p^{\dagger}(Y)$ (when following strictly the principles of Bayesian modeling framework) match as well as possible.\footnote{We use $^{\dagger}$ to denote to the true distribution for clarity, but in practice only work on empirical estimates of statistics computed from a sample $Y$ or user-provided targets for the same statistics.} We quantify the match using a collection of statistics $\T$ that capture the essential properties of the data, for example in form of central moments. The goal is to find $\lambda$ such that the virtual statistics $\hat \T_{\lambda}$ of the PPD match some target statistics $\target$. In ideal case, we find the optimal match where $\hat \T_{\lambda} = \target$. We use the phrase \emph{virtual statistic} for $\hat \T_{\lambda}$ to emphasize that it does not correspond to any particular observed data, but can instead be thought of as the corresponding statistic computed for a hypothetical -- or virtual -- data set sampled from PPD.

This general formulation depends on two elements: (1) the choice of the statistics $\T$ (and associated discrepancy measure) used for evaluating the match, and (2) the choice of the specific target statistic values $\target$. Together they define the optimality.
Importantly, these two objects are fundamentally linked with each other. On one hand, a richer set of statistics $\T$ leads to hyperparameter choice likely to be good in broader set of applications, but at the same time implies the need to be more careful when providing the target values $\target$, while often making computation more difficult as well.
On the other hand, very simple statistics, such as only the mean of the data, are typically not sufficient for identifying a single optimal choice, but are already useful since they provide a surface of equally good choices and help ruling out nonsensical options.
Finally, let us note that
computational algorithms solving for $\lambda$ are agnostic to how $\target$ were obtained, but to clarify the broad scope of the developed machinery we explain three common use-cases with different way of defining $\target$:
\begin{enumerate}
    \item \textbf{Principled statistician:} Following the strict Bayesian principle, the target statistics may be provided by a domain expert, in form of the expected values for the statistics. When used in this form, the proposed method essentially becomes a prior elicitation method; the expert provides subjective information on what is to be expected regarding the data, and this is used for indirectly defining the prior over the model parameters, similar to e.g. \citet{kadane:1980} and \citet{hartmann2020}. Importantly, the expert only needs to provide statistics of the data and not of the model parameters, not necessarily needing to understand the whole role of the model in detail.
    \item \textbf{Held-out validation:} A somewhat more pragmatic approach is to %
    use actual observed statistics $\T$ of a separate validation data as the target values $\target$.
    For example, in the case of a recommender engine we might use a subset of the users and items to estimate the target statistics $\target$ and find $\lambda$ for which the virtual statistics of the PPD best match the observed ones. After this, this data subset is discarded and the remaining data is used for posterior inference and possible further computation steps with the hyperparameters fixed to the selected ones.
    \item \textbf{Automatic prior specification:} Finally, the method can also be used in a fashion where we use the observed statistics $\T$ of all available data $Y$ as the targets $\target$,
    loosely following the concept of empirical Bayes \citep{Casella:1985}. While this breaks the fundamental idea of specifying the priors independent of the data, the statistics we use in practice are of low dimensionality and only characterize the data on a rough level. Consequently, we expect many practitioners to be comfortable in using the tool also in this manner.
    Prior predictive checks are routinely used for manually tuning the priors so that the support of the PPD roughly matches the data \citep{schad:2019,gabry2019visualization,gelman2020bayesian}, and our approach can be interpreted as automated procedure for this if using simple statistics like mean and variance.
    Note that this reasoning would no longer hold if using very rich statistics, in the extreme case directly using individual data entries so that $\target=Y$, but in this work we only consider problems where $\T$ consists of a few low-order moments.
\end{enumerate}

Throughout this work we use moments such as mean and variance as the statistics $\T$, since they lead to computationally efficient algorithms applicable for reasonably broad model families, but the method would work for other choices like quantiles or extreme values as well. In particular, we go through details of two practical algorithms for different scenarios, demonstrated in the context of Bayesian MF models. In Section~\ref{sec:model_specific}, we first look at cases for which we can compute certain low-order moments of the PPD analytically and hence can find a closed-form expression for $\lambda$ corresponding to the optimal match $\hat \T_{\lambda} = \target$. This is ideal in terms of computation, but restricted in scope to specific models and statistics. Hence, in Section
~\ref{sec:model_independent}, we proceed to provide a
 general-purpose algorithm applicable to broader family of models and statistics, formulated as explicit optimization of a discrepancy measure between the PPD and target statistics, using sampling-based estimates for the virtual statistics and stochastic gradient-descent (SGD) optimization. The method is applicable for all continuous hyperparameters and requires only ability to sample from PPD, but alternative forms of optimization could be considered to extend support also for discrete hyperparameters. We demonstrate this algorithm for hierarchical Bayesian MF models, but it could be applied also outside matrix factorizations.

\section{Matching Moments for PMF and CPMF} \label{sec:model_specific}

PMF as specified in \eqref{eq:genmf} allows us to derive analytic expression for certain moments of the PPD, to be used as virtual statistics. If we denote by $Y$ a virtual data matrix following the PPD, we can compute the  
mean $\E[Y_{ij} ; \lambda ]$, the variance $\var[Y_{ij} ; \lambda ]$, and the correlations $\rho[Y_{ij}, Y_{tl}  ; \lambda ]$ in closed form. Here the hyperparameters are
$ \lambda = \{K, \mu_{\theta},\sigma_{\theta}^2, \mu_{\beta},\sigma_{\beta}^2 \}$, and we drop the explicit depedency on $\lambda$, writing $\E[Y_{ij}] := \E[Y_{ij} ; \lambda ]$. The detailed derivations, provided in the Appendix, build primarily on the laws of total expectation, variance and covariace for marginalizing out $\theta$ and $\beta$.

\begin{proposition}\label{prop:moments}
For any entry of the virtual data matrix $\mathbf{Y} =\{ Y_{ij} \} \in \mathbb{R}^{N \times M}$, the mean and variance is given by:
\begin{align}
    \E[Y_{ij}] &= K\mu_{\theta}\mu_{\beta} \label{eq:pmf_closedform_exp} \\
    \var[Y_{ij}] &= K[\mu_{\theta}\mu_{\beta}+ (\mu_{\beta}\sigma_{\theta})^2 +(\mu_{\theta}\sigma_{\beta})^2+ (\sigma_{\theta}\sigma_{\beta})^2]
    \label{eq:pmf_closedform_var}
\end{align}
\end{proposition}

\begin{proposition}\label{prop:moments_2}
For any pair of entries $ Y_{ij} $ and $Y_{tl}$ of   matrix $\mathbf{Y}$, their correlation is given by:
\begin{align}
     \rho[Y_{ij},Y_{tl}] &= \begin{cases} 
0,\text{if }i \neq t \And j \neq l\\
1,\text{if }i = t \And j = l \\
\rho_1 ,\text{if }i = t \And j \neq l \\
\rho_2 ,\text{if }i \neq t \And j = l
\end{cases}
\end{align}

With $\rho_1 = \frac{K(\mu_{\beta}\sigma_{\theta})^2}{\var[Y_{ij}]} $ and $\rho_2 =\frac{K(\mu_{\theta}\sigma_{\beta})^2}{\var[Y_{ij}]} $.
\end{proposition}

These results indicate exactly how the hyperparameters of the model directly affect the virtual statistics derived from PDD, here denotated by $ \E[Y_{ij}]$, $\var[Y_{ij}]$ and $\rho[Y_{ij},Y_{tl}]$.
Given Propositions \ref{prop:moments} and \ref{prop:moments_2} and some target values for the moments, we can directly solve e.g. for the number of latent factors $K$. Denoting
$ \tau = 1-(\rho_1+\rho_2)$, we obtain:
  \begin{align}
      K  &= \frac{ \tau\var[Y_{ij}]-\E[Y_{ij}]}{\rho_1\rho_2} \left( \frac{\E[Y_{ij}]}{\var[Y_{ij}]} \right)^2 \label{eq:latent_var_est}.
  \end{align}
We also obtain formulas for relationships between the means and standard deviations of the priors that can be used to set e.g. the Gamma hyperparameters $a$, $b$, $c$ and $d$.
\begin{align}
    a &= \frac{\rho_2 \var[Y_{ij}] }{\tau\var[Y_{ij}]-\E[Y_{ij}]}   \\
    c &=  \frac{\rho_1 \var[Y_{ij}] }{\tau\var[Y_{ij}]-\E[Y_{ij}]}   \\
    bd &= \frac{\E[Y_{ij}]}{\var[Y_{ij}]}\sqrt{\frac{ac}{\rho_1 \rho_2}}.\label{eq:hyper_bd}
\end{align}

For the CPMF model we obtain the following result 
  \begin{align}
K  &= \frac{\tau\var[Y_{ij}]-\left(\kappa \psi'(w)+\frac{\psi{''}(w)}{\psi'(w)} \right) \E[Y_{ij}]}{\rho_1\rho_2} \left( \frac{\E[Y_{ij}]}{\var[Y_{ij}]} \right)^2 \label{eq:latent_var_est_2}
  \end{align}
for the latent factors and the following relationships for other terms:
\begin{align*}
    \sigma_\theta \sigma_\beta &= \frac{\var[Y_{ij}]}{\E[Y_{ij}]\kappa \psi'(w)}\sqrt{\rho_1 \rho_2} \\
    \E[Y_{ij}] &= \kappa \psi'(w) K\mu_{\theta}\mu_{\beta} \\
    \var[Y_{ij}] &=\kappa \psi{''}(w) K\mu_{\theta}\mu_{\beta}+[\kappa\psi'(w)]^2K[\mu_{\theta}\mu_{\beta}
     +(\mu_{\beta}\sigma_{\theta})^2+(\mu_{\theta}\sigma_{\beta})^2+(\sigma_{\theta}\sigma_{\beta})^2].
\end{align*}
The derivations are provided in Appendix in Propositions 8, 9, 10 and 13.

The generic observation model $Y_{ij} \sim F_Y(\sum_{k=1}^{K} \theta_{ik} \beta_{jk})$, with  $\E[Y_{ij} | \theta,\beta] = \sum_{k=1}^{K} \theta_{ik} \beta_{jk}$ can be analyzed using the same techniques. In this case we obtain the following equation for the dimensionality of the latent factors:
\begin{align}
  K  &= \frac{ \tau\var[Y_{ij}]-\E[\var(Y_{ij}|\theta, \beta)]}{\rho_1\rho_2} \left( \frac{\E[Y_{ij}]}{\var[Y_{ij}]} \right)^2 \label{eq:generic_solution}.
\end{align}
The term $\E[\var(Y_{ij}|\theta, \beta)]$ is model dependent with each distinct choice of a distribution for the observation model $F_Y$ having a different functional form for conditional variance $\var(Y_{ij}|\theta, \beta)$. For example, in the case of probabilistic matrix factorization model with a Gaussian observation model $ Y_{ij} \sim  \mathcal{N}(\sum_{k=1}^{K} \theta_{ik} \beta_{jk},\sigma_Y^{-1})$, we obtain $\E[\var(Y_{ij}|\theta, \beta)]=\sigma_Y^2$; while in the case of Poisson matrix factorization with observation model $ Y_{ij} \sim  \poi(\sum_{k=1}^{K} \theta_{ik} \beta_{jk})$, we would obtain $\E[\var(Y_{ij}|\theta, \beta)]=\E[Y_{ij}]$. Therefore, we obtain different analytical formulas based on Eq.~\ref{eq:generic_solution} for each of these models. The analysis presented here can be applied to virtually any matrix factorization models, with any combination of distributions for the priors $F(\mu_{\theta},\sigma_{\theta}^2)$ or $F(\mu_{\beta},\sigma_{\beta}^2)$, and observation model $F_Y$ (with finite mean and non-zero variance).

\textbf{Observations:} the above equations provide closed-form expressions that determine the priors in terms of the virtual statistics of PPD matching the target statistics. They can be computed instantaneously, bypassing the need for optimizing $\lambda$, and provide %
a computationally efficient way of automatically determining the number of factors for PMF and CPMF. The result is not necessarily optimal for any particular task, especially when the data does not follow the model well, but as illustrated experimentally in Section~\ref{sec:experiments} tends to be a good choice when there is small model mismatch and for the dataset analyzed.

\subsection{Empirical Estimates for the Moments}

\begin{algorithm}[t]
    \SetAlgoLined
    \SetKwInOut{Input}{input}\SetKwInOut{Output}{output}
    \Input{observed matrix $\mathbf{Y} =\{ y_{i,j} \} \in \mathbb{R}^{N \times M}$, number of samples for the estimator $S$}
    \Output{$\widetilde{\rho_1}$ and $\widetilde{\rho_2}$ }
     Initialize arrays $\mathbf{a} =\{ a_{i,j} \} \in \mathbb{R}^{S \times 2}$ and $\mathbf{b} =\{ b_{i,j} \} \in \mathbb{R}^{S \times 2}$\;
     \For{ $ s \in \{1, \cdots , S\} $}{
      Sample $ i \sim \text{Unif}( \{1, \cdots ,N \} ) $\;
      Sample $ j_1 \sim \text{Unif}( \{1, \cdots , M \} ) $\;
      Sample $ j_2 \sim \text{Unif}( \{1, \cdots , M \} \setminus \{ j_1 \}  ) $\;
      $a_{s,1} \leftarrow y_{i,j_1}$\;
      $a_{s,2} \leftarrow y_{i,j_2}$\;
      
      Sample $ j \sim \text{Unif}( \{1, \cdots ,M \} ) $\;
      Sample $ i_1 \sim \text{Unif}( \{1, \cdots , N \} ) $\;
      Sample $ i_2 \sim \text{Unif}( \{1, \cdots , N \} \setminus \{ i_1 \}  ) $\;
      $b_{s,1} \leftarrow y_{i_1,j}$\;
      $b_{s,2} \leftarrow y_{i_2,j}$\;
     }
     Calculate and return the Pearson correlation of columns of $\mathbf{a}$ and $\textbf{b}$\;
     \caption{Empirical correlations $\widetilde{\rho_1}$ and $\widetilde{\rho_2}$ for a observed matrix $\mathbf{Y} =\{ y_{ij} \} \in \mathbb{R}^{N \times M}$}
     \label{alg:corr}
    \end{algorithm}

As explained in Section~\ref{sec:usecases}, one way of using the method is based on matching the virtual statistics with the true statistics of the observed data.
For MF models, we only have a single matrix representing one (often partial) observation, and hence need to estimate the statistics by averaging over the rows and columns of one matrix instead of averaging over multiple matrices.
The mean and variance can be easily estimated over the independent matrix entries, but it is not immediate nor intuitive how to compute the correlations. One remark is that there are two values of correlations, one for rows and another one for columns, hence we can levarage this property of the model to make an estimator. To estimate the correlation we derived an estimator, described in Algorithm~\ref{alg:corr}, that samples two elements from the same row or column and uses them to calculate the correlation of elements sharing a row or column index, independent of the specific index.

\section{Gradient-based Approach} \label{sec:model_independent}

Deriving analytic expressions for even simple models and moments is tedious, error-prone, and often impossible. For users willing to give up the convenience and robustness of analytic expressions, we next provide an optimization-based alternative that does not require model-specific derivations. We will later demonstrate this algorithm in context of a hierarhical variant of Poisson MF by \citet{gopalan2015scalable}, as a concrete example of a matrix factorization model for which analytic moments are not known.

\subsection{Formulation as an Optimization Problem}

Instead of directly equating the target and virtual statistics, we solve for $\lambda$ that minimizes a discrepancy measuring the difference between the requested quantities $\target$ (either expert's prior expectation or empirical estimate) and the virtual statistics $\hat \T_{\lambda}$ of the PPD
\begin{equation}
\min_\lambda \discrepancy\left(\target, \hat T_{\lambda}\right),
\label{eq:discrepancy}
\end{equation} 
where for brevity we write $\hat \T_{\lambda}$ instead of the complete 
$\hat T (\E[g(Y);\lambda])$.
As before, $\T$ is a collection of statistics (e.g. central moments) defining which aspects are used for learning the hyperparameters. Intuitively, a richer set allows more accurate prior specification, but requires more careful choice of discrepancy as well.
For example, to match at the same time an expected value $\text{E}^*$ and variance $\text{V}^*$ we can use
$\discrepancy := (\text{E}^* - \E[Y])^2 + (\text{V}^* - (\E[Y^2]-\E[Y]^2))^2$,
where $g(Y) = (Y, Y^2)$
and $\hat \T(E_1, E_2) = (E_1, E_2-E_1^2)$. 
The result of the optimization problem naturally depends on the choice of the discrepancy and -- in the case of multiple target statistics -- the weighting of the different statistics. However, when the data follows the assumed model and the target statistics are achievable, we can typically reach $\hat T_{\lambda} = T$ and the choice only influences the computational efficiency of the optimizer, not the result itself. For other cases, the user needs to select an apppropriate measure and weighting that reflects their preferences.

The process builds on repeatedly drawing samples from PPD to estimate the virtual statistics $\hat{\T}_{\lambda}$, and solving for $\lambda$ using some iterative algorithm. Importantly, this only requires the model code providing the samples and the target statistics $\target$, and hence allows solving for the priors as part of the modeling pipeline without needing to consider any particular data.

\subsection{Differentiable Moments' Estimators}
\label{sec:differentiating_moments}

We optimize \eqref{eq:discrepancy} with 
stochastic gradient descent, using
Monte Carlo approximation~\citep{mohamed2019monte} for the prior predictive moments and automatic differentiation with reparameterization gradients~\citep{figurnov2018implicit, hartmann2020} wherever available and REINFORCE (log derivative trick)~\citep{williams1992simple} can be used elsewhere.
For gradient-based optimization %
we require
that $\discrepancy(\cdot)$ and $\hat \T$ are differentiable w.r.t their arguments, and 
that we can propagate gradient
 $\nabla_\lambda$
 through $\E[g(Y)]$.

Even though our main interest is in BMF models, we directly derive the stochastic optimization for a somewhat more general family of hierarchical Bayesian models consisting of $L$ layers of latent variables $\Z_l$.
The procedure is based on recursively applying the law of total expectation. 
The unconditional expectation of $g(Y)$ can be obtained by integrating out latent variables $Z$,
but 
since an analytical form of it is not available, 
we proceed by performing a numerical approximation,
where each of the integrals over latent variables $Z_1, \dots Z_l \dots Z_L$ is replaced by a sum over samples from respective (conditional) distributions.
An estimate of 
the required gradient $\nabla_\lambda \E[g(Y)]$ is then obtained by propagating estimates of the gradients
$\nabla_\lambda \E[g(Y)|\z_l]$ and $\nabla_\lambda \log p(\z_l| \dots; \lambda)$ backward through the computation graph.

The detailed description and derivation of the algorithm for estimating the gradients, which requires somewhat tedious notation and may not be of interest for readers primarily interested in BMF models, is provided in Supplement~\ref{sec:generalmethod}. We then show in Supplement~\ref{sec:model_indep_pmf} how both PMF and a hierarchical Poisson factorization (HPF) by \citet{gopalan2015scalable} can be expressed as instances of this general structure. For the HPF model we do not have analytic expressions for the moments and hence need this more general algorithm.

\section{Related Work}
\label{sec:relatedwork}

The way we use the prior predictive distribution for hyperparameter optimization is, to our best knowledge, novel. The technical elements and the overall goal are, however, related to several seemingly distinct concepts, briefly outlined here.

\textbf{Method of moments.}
Method of moments (MoM) refers to inferring model parameters by equating (in our terminology) virtual moments
with sample moments, as an alternative to maximum likelihood or Bayesian inference \citep{casella, pawitan, keith}. Even though our method involves similar expressions, the approaches are fundamentally different in their goals. MoM is typically used for learning a point estimate of the parameters ($Z$ in our notation) of a probabilistic model based on observed data and the theoretical moments for the model, whereas our goal is to set values of hyperparameters $\lambda$ and the formulation can be used for any statistics and targets that may not even be estimated from any data. Importantly, in our case we proceed to conduct standard Bayesian inference for $Z$ with the chosen $\lambda$.
The practical algorithms are also different. MoM is typically used either for simple models with easy analytic expression for the moments, or as generalized method of moments relying on asymptotic normality \citep{casella, pawitan}. Our derivations for PMF and CPMF are contributions for the MoM literature in itself, providing equations for computing the virtual moments for a non-trivial model class, and the model-independent algorithm using stochastic gradient optimization may be useful also in other contexts.

\textbf{Prior elicitation.}
By following the Bayesian paradigm, the choice of prior should be conducted independent from the observed data \citep{garthwaite:2005, ohagan:2006, mikkola2022}. When priors are difficult to specify, \emph{prior elicitation} can be used to convert expert opinions expressed e.g. using graphical interfaces to prior distributions. This is often done by eliciting information about the parameters, but e.g. \citet{kadane:1980} and \citet{akbarov:2009} have studied prior elicitation using prior predictive distributions.
When the target statistics are provided by the user, our approach can be seen as prior elicitation following their principles but using specific type of user-specified information. Furthermore, they could only elicit priors for over-simplistic models with Gaussian likelihoods, whereas we provide tools 
for more general models with potentially large number of latent variables or hierarchical structure. More recently, \citet{hartmann2020} followed similar lines as our paper by developing model-independent prior elicitation method based on PPD, but modelled directly the probability of the outcome data using a Dirichlet process model, to focus on accounting for the expert uncertainty in the numerical specification of the probability.

\textbf{Learning flexible priors.} For highly flexible models, such as variational autoencoders (VAE) or Bayesian neural networks (BNN), the question of the prior choice is more poorly defined than for classic Bayesian models. Several authors have proposed directly learning a flexible prior for such cases, typically by optimizing the same objective that is used for inference. For instance, \citet{tomczak2018vae} and \citet{Klushyn2019priors} learnt mixture priors for VAEs by optimizing the variational objective over the prior as well, and \citet{nalisnick2018learning} and \citet{ducontrolled} learnt parametric priors for BNNs to satisfy specific properties (e.g. rotational invariance or monotonicity) for the model predictions. The most closely related work in this direction is \citet{nalisnick2021predictive}, who proposed \emph{predictive complexity priors} to learn a prior such that the model output best matches that of a reference model. Their reference model can be interpreted as playing a role similar to our target statistic, providing an alternative way for the analyst to provide subjective information on how the model is expected to work, and their algorithm minimizing Kullback-Leibler divergence between the outputs shares some commonalities with our model-independent solution.

Our approach has two main differences to these methods directly optimizing for the prior. First of all, our primary interest is in determining hyperparameters for a specific hierarhical Bayesian models with clearly defined priors, rather than learning an arbitrary prior distribution to optimize the predictive performance. The more fundamental difference, however, is that all of these approaches use the observed data itself while fitting the prior, often by explicitly maximizing the final learning objective, whereas we only consider the PPD of the model. Even if using the observed data to determine the target statistics $\target$, the process for determining the priors still does not involve posterior inference in contrast to these approaches.

\textbf{Empirical Bayes.}
When the target statistics are computed directly from observed data, rather than provided by the user based on domain expertise, the method resembles \emph{empirical Bayes} (EB) methods that use the data to form the prior \citep{Casella:1985}. While this is misaligned with rigorous Bayesian principles, EB is commonly used for simplifying modeling tasks, e.g. by fixing hyperparameters to their maximum marginal likelihood values. Our approach shares the conceptual motivation and can be seen as practical extension of EB to models for which classical EB would be difficult. In context of BMFs, \citet{wang2018empirical} recently provided EB solution for a variational approximation of the model. This approach uses flexible prior families, but is focused on Gaussian likelihoods.

\textbf{Global optimization.}
The approach provides an alternative to global optimization of hyper-parameters with techniques such as Bayesian optimization \citep{DBLP:conf/NeurIPS/SnoekLA12}. The core difference is that they directly optimize for some specific criterion and posterior inference algorithm, and the result will not be optimal for other choices. Our result naturally depends on the choice of the statistics (and the discrepancy measure), but is agnostic of the eventual inference algorithm and the utility of the downstream task -- it can be used for any  modeling task and as such better follows the Bayesian modeling paradigm. Global optimization also has clear computational disadvantage in requiring repeated posterior analysis, though strategies for reducing the cost by using subsets of data have been developed \citep{Fabolas}. For models with analytic solution our approach is immediate, but also the model-independent algorithm has in our experiments been faster than any method relying on posterior inference. 

Finally, we would like to emphasize that our method takes advantage of the fact that generative models define a sampling distribution, and hence is restricted to such models. That is, we can only carry out hyperparameter optimization for generative probabilistic models, in contrast to Bayesian optimization that can be applied for determining hyperparameters of arbitrary machine learning models. Relying on this additional property is also why were are able to outperform the more generic method.

\textbf{Likelihood-free inference.}
Our approach also relates to likelihood-free inference (or Approximate Bayesian Computation) for posterior inference of models for which the likelihood cannot be evaluated in closed form but can be sampled from \citep{Marin2012,Lintusaari2016}. The prior predictive distribution has these properties, which means our method can be interpreted as likelihood-free inference for the hyperparameters and we can borrow suggestions for discrepancy measures from the literature in that field (e.g. \citealt{Beaumont2010,Lintusaari2016}). A core difference, however, is that we do not seek for a distribution over the hyperparameters but instead prefer a point estimate; the result will typically be used as input for subsequent modeling stages and we do not expect the user to be willing to specify further hierarchical priors for these parameters. This allows us to use direct gradient-based optimization, in constrast to model-based approximations \citep{Gutmann2017} or inefficient MCMC samplers \citep{marjoram2003} used for likelihood-free inference.

\textbf{Spectral methods.}
For various latent variable models, such as topic models, spectral methods find globally optimal parameters by matching low-order moments of data while integrating out the  latent variables \citep{anandkumar2012}. We also use low-order moments (or other statistics) as inputs and search for point estimates -- for hyperparameters while marginalizing out the parameters -- but otherwise the techniques are very different. Spectral methods require as inputs accurate estimates for very fine-grained moments (e.g. for topic models the joint probability of all possible triples of words, which already for a vocabulary of 10,000 words corresponds to $10^{12}$ values), whereas we use a small number of global moments -- just four in case of PMF -- that can be provided as crude estimates.

\section{Experiments}
\label{sec:experiments}

Here we present two main experimental results demonstrating the potential of our method, followed by a set of technical examinations illustrating its properties and behavior. 
We first verify the approach by showing in Section \ref{sec:analytic} that for PMF the analytic solution can retrieve the true generating prior for artificial data. 
Then, in Section~\ref{sec:experiment_bo} we compare hyperparameters selected by prior predictive matching against the ones found by Bayesian optimization applied directly for a posterior quality measure.
Finally, Section \ref{sec:technical} presents several technical illustrations, including an analysis of sensitivity of the approach to model mismatch or various kinds and a convergence study for the gradient-based algorithm.
 Throughout the experiments we use capital letters to refer to different hyperparameter settings, defined in Tables~\ref{tab:pmf_initializations}, \ref{tab:hpf_initializations}, \ref{tab:priors_initialization} and \ref{tab:distributions_params} in Appendix.
The code used for all experiments is available in public repository\footnote{\url{https://github.com/zehsilva/prior-predictive-specification}}.

\subsection{Analytic Solution for PMF and CPMF}
\label{sec:analytic}

\subsubsection{Number of factors for count data}

\begin{figure}[t]
    \centering
      \includegraphics[width=0.48\columnwidth]{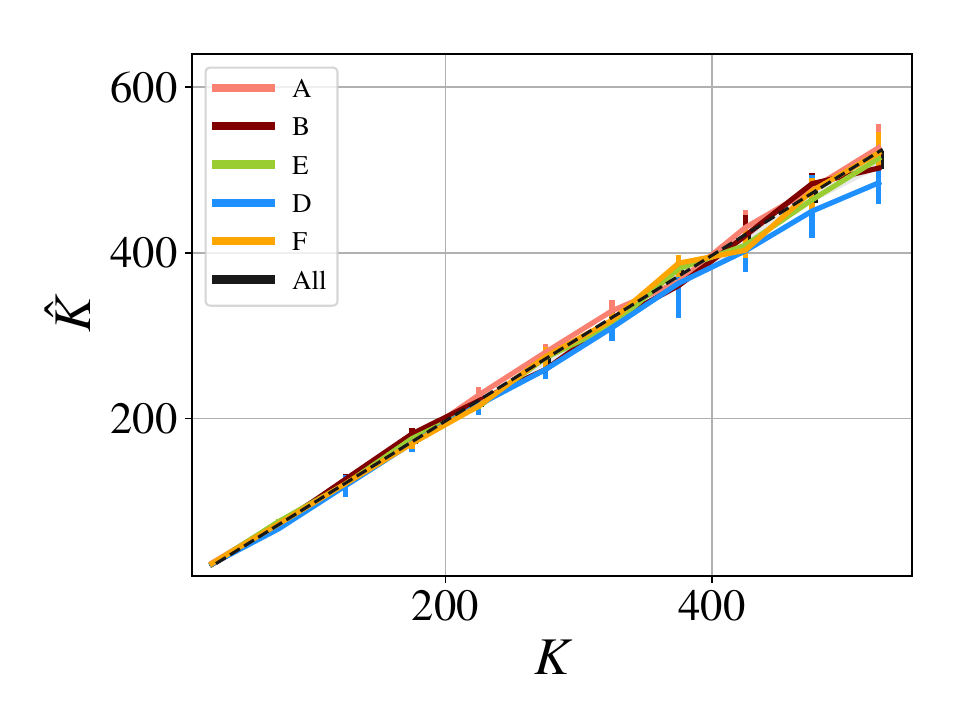}
      \includegraphics[width=0.48\columnwidth]{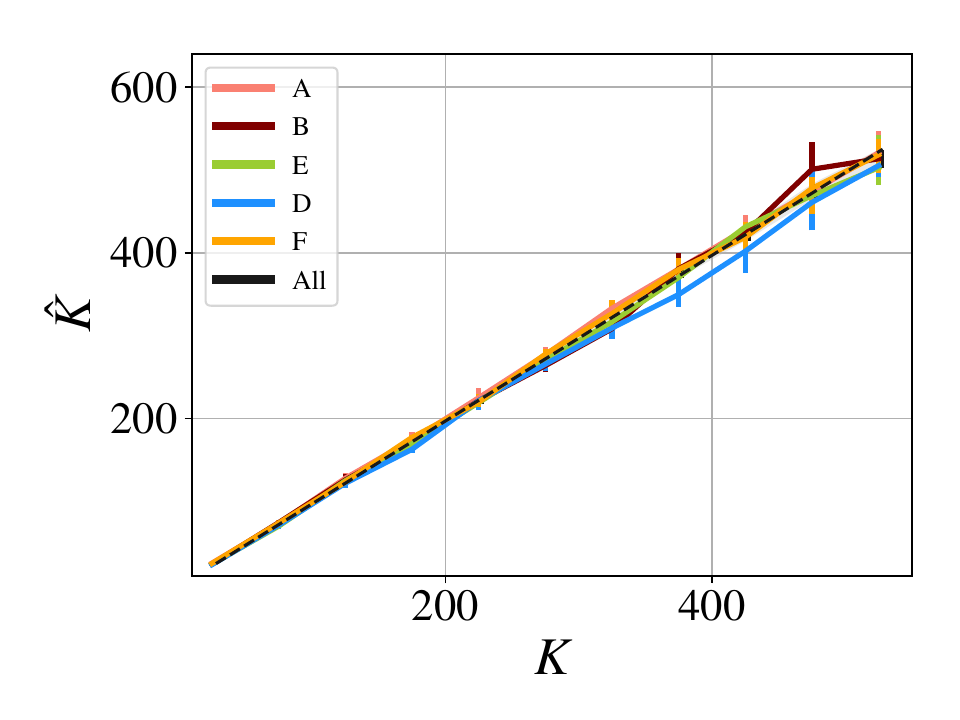}
    \caption{Prior predictive matching provides estimates $\hat K$ for each true latent factor dimensionality $K$ and prior configurations (colored lines), as analytic expression of empirical moments for both Poisson MF (left) and Compound Poisson MF (right). The lines indicate the median over 30 replications and the shaded are corresponds to $95\%$ confidence intervals.}   
    \label{fig:model-based}
\end{figure}

\begin{figure}[t]
    \centering
    \includegraphics[width=0.32\textwidth]{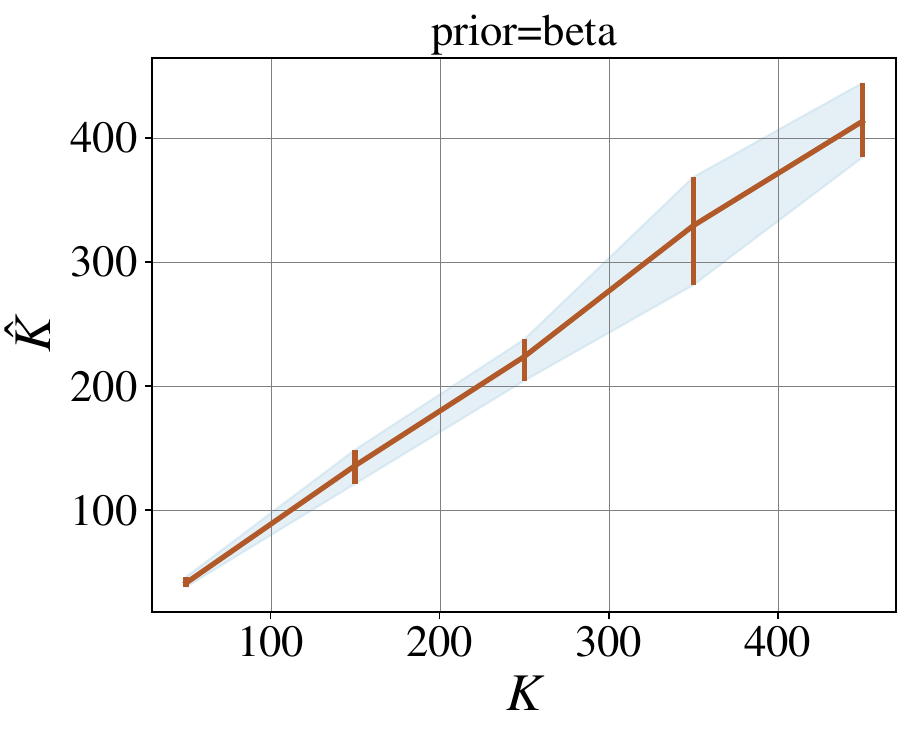}
    \includegraphics[width=0.32\textwidth]{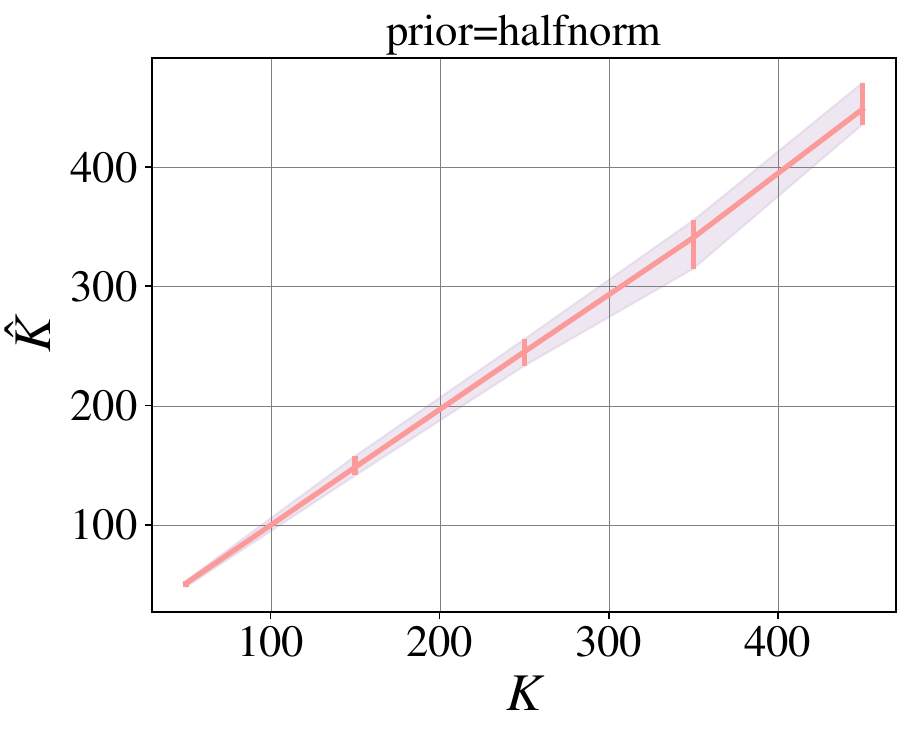}
    \includegraphics[width=0.32\textwidth]{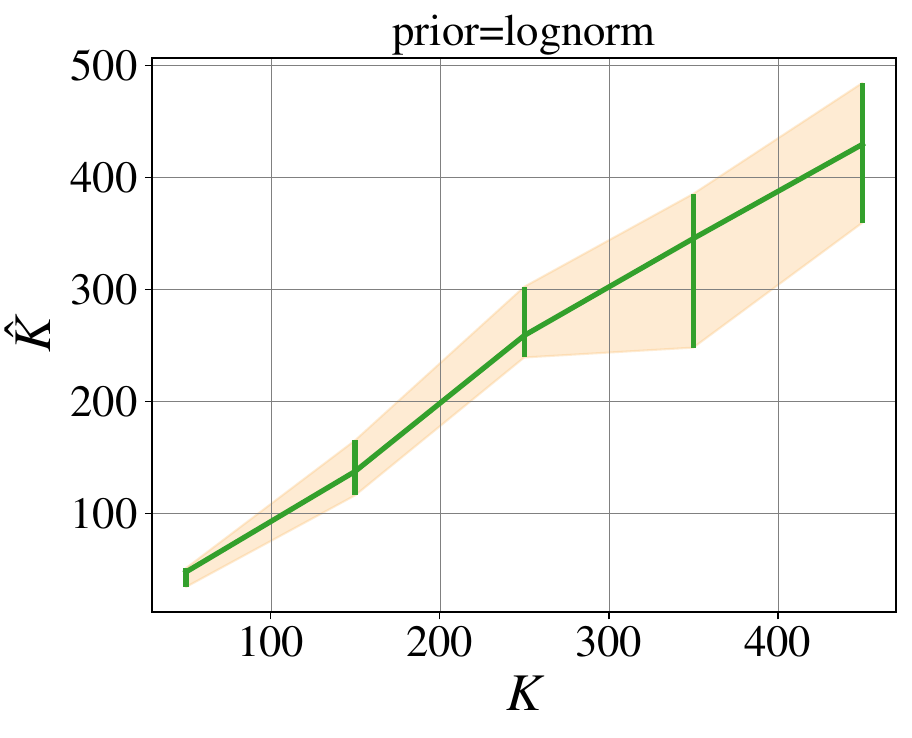} 
    \caption{Comparison of the latent factor dimensionality estimator $\hat K$ for different choices of prior distribution for PMF: Beta (left), Half-Normal (center) and Log-Normal (right). 
    The y-axis represent the estimated latent factor dimensionality $\hat K$ with 95\% confidence intervals around the median, and the x-asis represent the true dimensionality $K$.
    }        
    \label{fig:diffpriors}
\end{figure}

For PMF and CPMF, \eqref{eq:latent_var_est} and \eqref{eq:latent_var_est_2} provide analytic expressions for hyperparameters given the target moments. We demonstrate them in an empirical Bayes scenario, where estimates of the observed data are used as targets, to show that the results are robust to  estimation errors.
We sample a data matrix (of size {$10^3 \times 10^3$}) from the model for 30 scenarios where the true hyperparameters (denoted by $\lambda^*$) are set at different values. We repeat this for a range of values for the true $K$, and for each data compute the empirical estimates required for estimating the number of factors using \eqref{eq:latent_var_est} and \eqref{eq:latent_var_est_2}. Figure~\ref{fig:model-based} shows the estimates %
accurately match the ground truth when the data follows the model, for both PMF and CPMF with observation model $Y_{ij} \sim \sum_{i=1}^{N_{ij}} \mathcal{N}(1,1)$. For both cases we used gamma distributions as the priors.

The formulation of the model in \eqref{eq:genmf} allows for generic prior distributions with hyperparameters mean $\mu$ and variance $\sigma^2$, and the resultings equations are valid for multiple choices of prior distribution consistent with the model constraints (for example of non-negativity of the rate of the Poisson). We repeated the above analysis using different prior distributions for PMF: Beta, Half-Normal and Log-Normal. For each prior distribution we fixed a set of valid hyperparameters (the values used for each the configuration of prior distribution are in Table~\ref{tab:priors_initialization}) . Figure~\ref{fig:diffpriors} shows the results for these choices of priors, verifying that the result holds irrespective of the prior.

\subsubsection{Different observation models}\label{sec:differentmodels}

To demonstrate the method on other forms of observed data, we next evaluate the analytic solutions of \eqref{eq:generic_solution} for the latent dimensionality of the generic MF model of \eqref{eq:generic_pmf}. The empirical validation is performed using Gumbel, Laplace and Normal distributions for the observation model, with Log-Normal, Half-Normal, Beta and Gamma priors for each of observation distributions. The location parameters of the observation distributions were set such that the conditional mean is linear with respect to the latent factors $\E[Y_{ij} | \theta, \beta]=\eta_{ij}=\sum_{k=1}^{K}\theta_{ik}\beta_{jk}$, while the expected conditional variance is a hyperparameter of the model that is related to the scale parameters of the distribution (details in Appendix~\ref{append:experiments}). For example, for the Gumbel distribution with location and scale parameters $\mu_{ij}$ and $\sigma$, we have $Y_{ij} \sim \text{Gumbel}(\mu_{ij},\sigma)$, and given that $E[Y_{ij} | \theta, \beta]=\eta_{ij}=\mu_{ij}+\gamma\sigma$, the location parameter is set to $\mu_{ij}=\eta_{ij}-\gamma\sigma$,\footnote{The constant $\gamma$ used for the mean of the Gumbel distributions is the Euler–Mascheroni constant, with value $\approx 0.577$} and the resulting conditional expected variance is $\E[\var(Y_{ij} | \theta, \beta)]=\frac{\pi^2}{6}\sigma^2$.

\begin{figure*}[t]
    \includegraphics[width=1.0\textwidth]{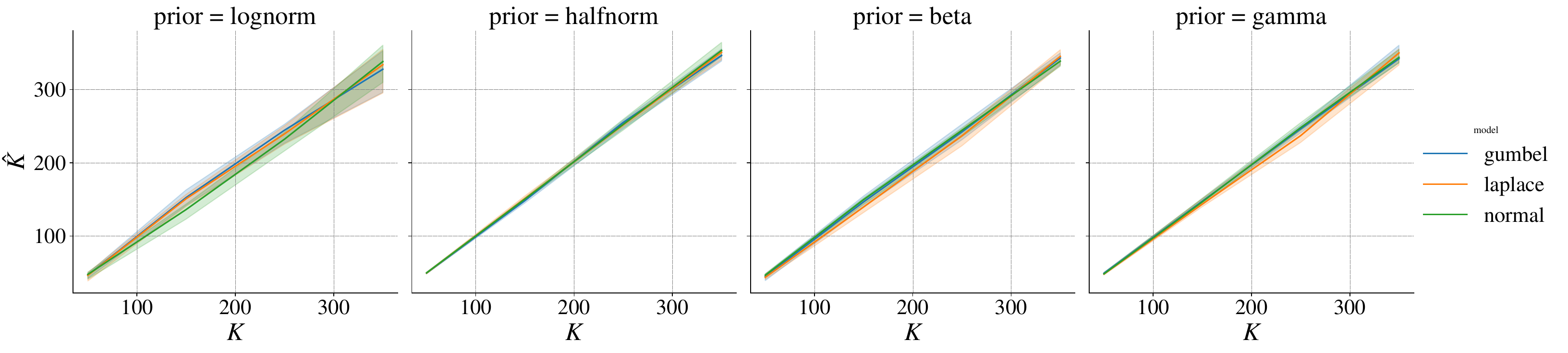}
    \caption{Comparison of the latent factor dimensionality estimator $\hat K$ for generic Bayesian matrix factorization models, validated using different choices of observation models (Gumbel, Laplace and Normal) and different choices of prior distribution: Log-Normal, Half-Normal, Beta and Gamma. 
    The y-axis represent the estimated latent factor dimensionality $\hat K$ with 95\% confidence intervals around the median, and the x-asis represent the true dimensionality $K$.
    }        
    \label{fig:diffmodels}
\end{figure*}

For each combination of observation and prior distributions, we vary the hyperparameters (details in Appendix~\ref{append:experiments} and Table~\ref{tab:priors_initialization}) and the true  dimensionality of the latent vectors $K \in \{50, 150, 250, 350\}$, with 20 runs for each unique configuration. We sample a data matrix (of size {$10^3 \times 10^3$}) for the each configuration of the experiment, calculating empirical estimates for the mean, variance and correlations, and using those estimates to estimate the latent dimensionality $\hat K$.
Figure~\ref{fig:diffmodels} shows the estimated dimensionality $\hat K$ (in the y-axis) for each true latent dimensionality $K$ (aggregated over multiple runs and hyperparameters configurations for each prior and observation distribution). The conclusion is that the method retrieves the true latent dimensionality well irrespective of the likelihood and prior choices.
Figure~\ref{fig:diffmodels_rel} shows the relative error $\frac{\hat{K}-K}{k}$ of the latent dimensionality estimator, which allows a more detailed analysis of the effect of the prior and observation model choice. For some combinations the variance of the estimator is higher than for others, even though all have similar mean performance.
In particular, the Log-Normal prior displays higher variability in general, while the variability of the relative error for the Beta prior decreases with the true latent dimensionality $K$.

\begin{figure*}[t]
    \centering
    \includegraphics[width=0.7\textwidth]{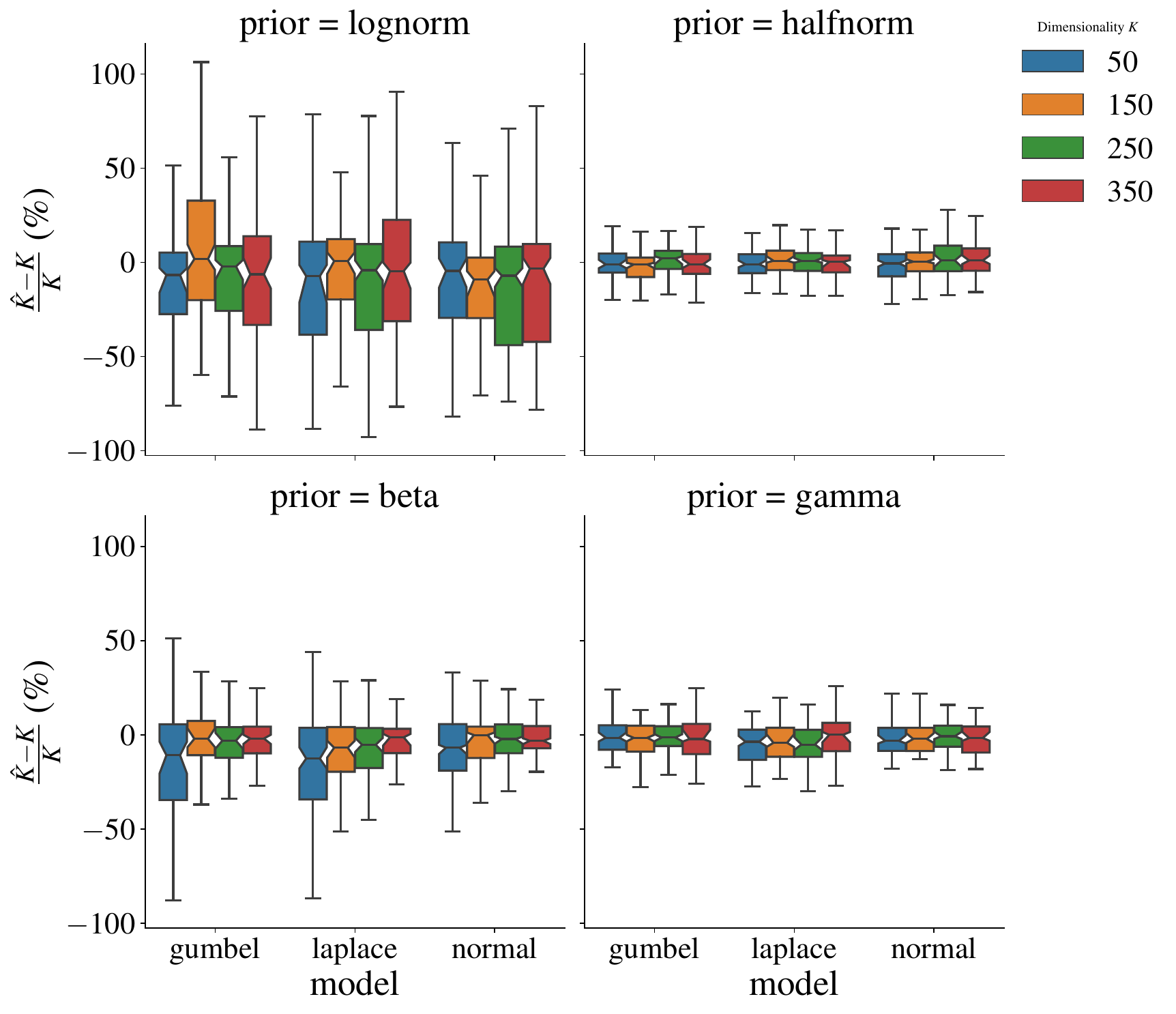}
    \caption{Boxplots of the relative error $\frac{\hat K-K}{K}$ analyzing the latent factor dimensionality estimator $\hat K$ for generic Bayesian matrix factorization models, validated using different choices of observation models (Gumbel, Laplace and Normal) and different choices of prior distribution: Log-Normal, Half-Normal, Beta and Gamma. 
    The y-axis represent relative error, displayed in a collection of boxplots showing the median and quartiles, with color code for different true latent dimensionality $K \in \{50, 150, 250, 350\}$, observation model in the x-axis and organized in rows and colums for different priors (Log-Normal in top-left, Half-Normal in top-right, Beta in bottom-left and Gamma and bottom-right).
    }        
    \label{fig:diffmodels_rel}
\end{figure*}

\subsection{Posterior Quality}
\label{sec:experiment_bo}

The main use for the approach is as part of a modeling process where we eventually carry out posterior inference using the selected hyperparameters. The quality of the solution can hence only be evaluated by inspecting how the final model performs. We do this by fitting a PMF model to the \texttt{user-artists} data of the~\texttt{hetrec-lastfm} dataset~\citep{Cantador:RecSys2011}\footnote{\url{ http://files.grouplens.org/datasets/hetrec2011}}, using an efficient implementation of coordinate ascent variational inference\footnote{
We extended \url{http://github.com/dawenl/stochastic_PMF}
to support sparse data. 
}
 fitted to randomly selected 90\% subset of the data. We evaluate the quality of the model using the PSIS-LOO criterion~\citep{DBLP:journals/sac/VehtariGG17a}. PSIS-LOO is here considered as an example of a typical task-agnostic metric a practitioner would be likely to use, but the experiment is not sensitive to the specific choice.

Figure~\ref{fig:prior_posterior} already illustrated the quality surface via explicit enumeration of the hyperparameter choices in a regular grid, shown as slices of the five-dimensional surface where all remaining parameters were fixed to the optimal ones provided by our method. Some prior choices have better PSIS-LOO scores -- for example, slightly smaller $K$ would be better -- but the crucial observation is that prior predictive matching provides sufficiently good solution in an instant. It is also important to understand that these results depend on the specific inference algorithm and evaluation metric used, and the optimal solution would change -- possibly by a lot -- if the variational approximation was replaced, e.g., with MCMC and PSIS-LOO by a another metric. Hence, exactly matching whatever choice happens to be optimal here would not even be correct.

\begin{figure}[t]
    \centering
    \includegraphics[width=0.48\columnwidth]{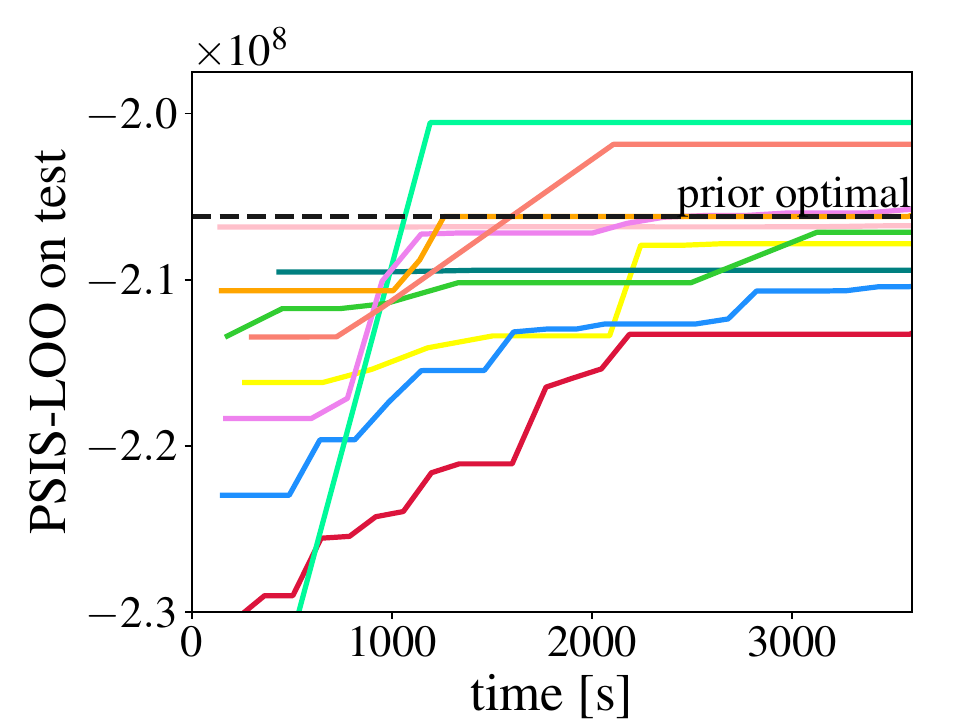}
    \includegraphics[width=0.48\columnwidth]{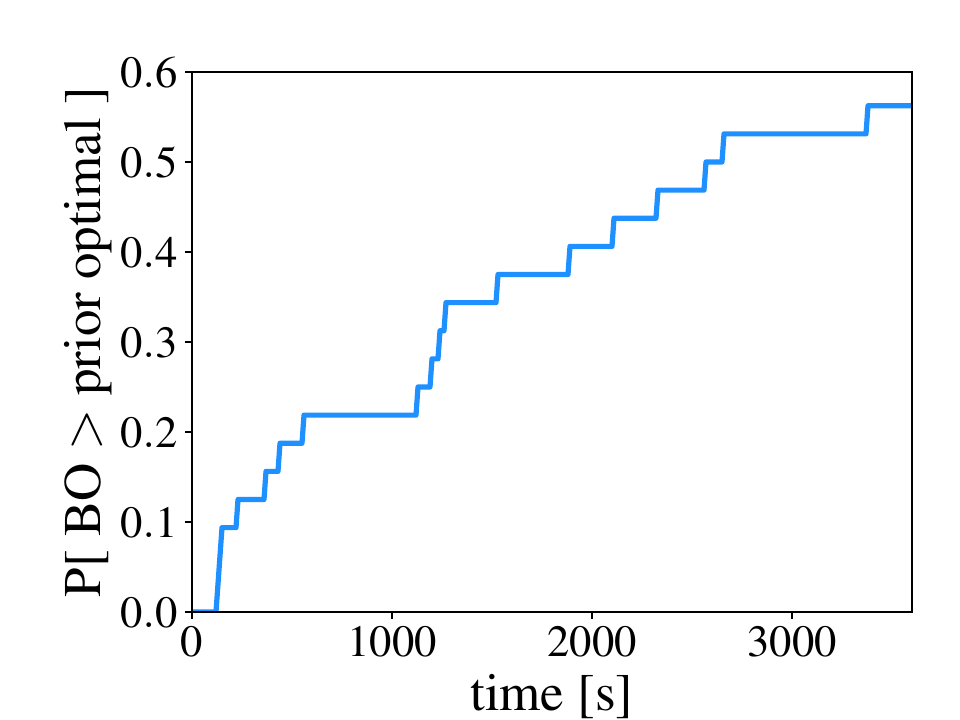}    
    \caption{
Comparison between Bayesian optimization runs (BO; colored lines) and the proposed method (dashed line) for optimization of five-dimensional hyperparameter for Poisson MF. BO eventually provides better PSIS-LOO (by directly optimizing for it), but even after an hour almost half of the runs are worse than our analytic, immediate, solution (right).
    }
    \label{fig:bo}
\end{figure}

Figure~\ref{fig:bo} compares the solution with the alternative of running a global optimization algorithm to figure out the hyperparameters, for the same model, data, inference algorithm and evaluation metric as above. We optimize for $\lambda$ using the robust Bayesian optimization (RoBO) tool~\citep{klein-bayesopt17} (with default options for parameters specified below), and compare results of multiple runs with the one provided by our approach in an instant. 

BO requires specifying a bounding box of feasible values, which is not trivial for PMF and has notable impact on the performance. %
For fairly carefully selected search space of 
$\mu_\theta, \sigma_\beta, \sigma_\theta, \sigma_\beta \in [10^{-4}, 10^2]$ (optimized on logarithmic scale)
and $K \in \{1..100\}$, BO is typically able to improve on the result of prior predictive matching. However, this takes on average roughly an hour even for this simplified scenario where efficient variational approximation can be fit to the fairly small data (on average) in a minute, and again we remind that BO is directly optimizing for this particular measure for the chosen inference algorithm. 
For more complex models taking e.g. hours for posterior inference, BO would be unfeasible. Even if we would then need to switch from analytic solutions to the model-independent algorithm for prior predictive matching, experiments in Section~\ref{sec:convergence_experiment} indicate that our approach is still likely to be more efficient; for PMF the algorithm converged in seconds, which can be directly compared with the times reported here for BO.
Finally, we remark that the prior optimal choice always outperforms the initial estimate of BO and BO typically requires tens of iterations to reach its quality, which suggests that it is also likely to outperform various naive heuristics and default values for the hyperparameters.

\subsection{Technical Validations}
\label{sec:technical}

After demonstrating that our approach can be successfully applied for selecting hyperparameters and priors in  practical scenarios, we proceed to technical experiments examining in detail the sensitivity of the approach to violation of the model assumptions (Section~\ref{sec:mismatch}) and testing properties of the gradient-based algorithm (Sections~\ref{sec:convergence_experiment}). Additional technical experiments on estimators and parameterization of the prior are provided in Appendices \ref{sec:parametrization} and \ref{sec:experiment_bias_variance}.

\subsubsection{Sensitivity to Model Mismatch}
\label{sec:mismatch}

In most applications, the data does not follow any model in the assumed model family. Since we compute the virtual statistics conditional on the model, it is unclear how well the approach works when the model mismatch is severe. We conducted three experiments to evaluate this, by controlling the amount of mismatch on artificial data while assuming the PMF model and using the analytic expression \eqref{eq:latent_var_est} for setting the number of factors $K$. 

We first consider two forms of model mismatch that relate to the distribution of the individual entries. For both experiments we generated data with $K \in \{25, 50, 75, 100, 125, 150\}$ for two hyperameter configurations and generated $20$ realizations for each configuration, and we summarize the results using relative error $\frac{\hat K - K}{K}$ of the estimates $\hat K$ for different true $K$.

\begin{figure*}[t]
    \centering
    \includegraphics[width=0.45\textwidth]{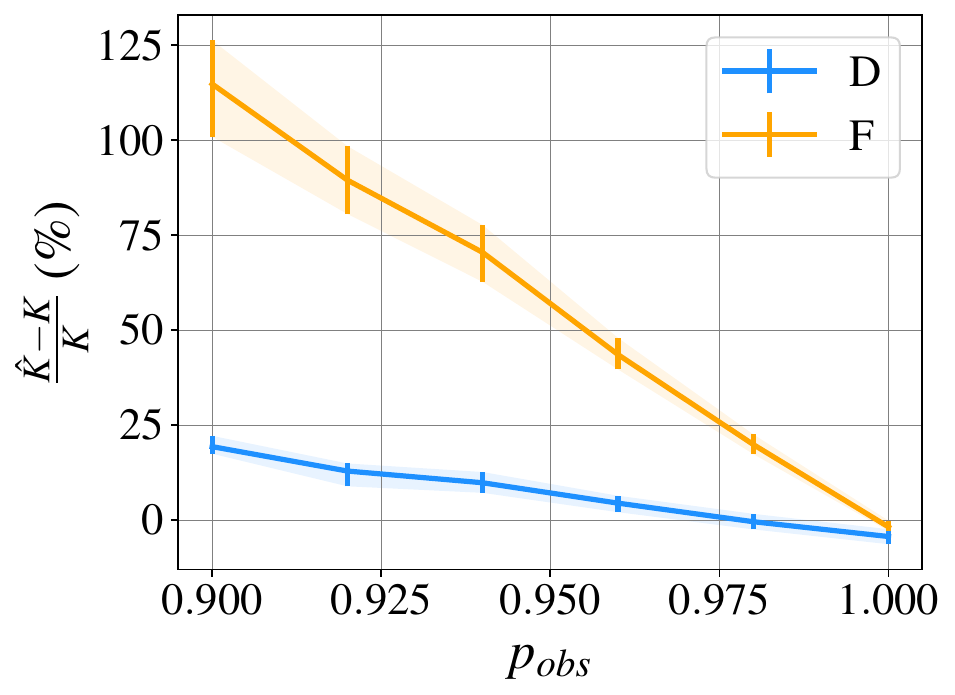}
    \includegraphics[width=0.45\textwidth]{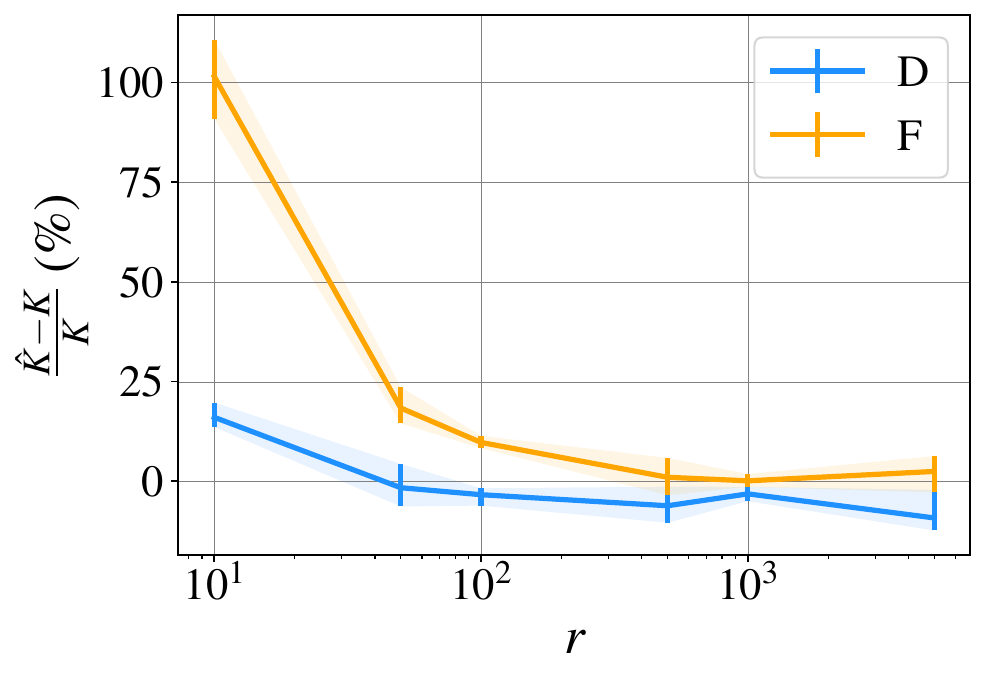}
    \caption{Sensitivity to model mismatch on zero-inflated (left) and overdispersed (right) data.
    For both cases increasing model mismatch (smaller $p_{obs}$ or $r$) increases the error monotonically, implying the approach is robust for small model mismatch but may give misleading results if the assumed model family fits the data very poorly.
    The y-axis represent the relative error $\frac{{\hat K} - K}{K}$ with 95\% confidence intervals around the median, and D and F refer to two different true hyperparameter configurations.
    }        
    \label{fig:sensitivity}
\end{figure*}

\paragraph{Sparsity} 
A typical mismatch with count-matrix models relates to sparsity; the data has more zeroes than expected under the model. To investigate the effect of this, we independently sample for each entry $Y_{ij}$ (sampled from PMF) a Bernoulli variable $X_{ij} \sim \text{Ber}(p_{obs})$ controlling whether the entry is observed, so that our final observation is given by $\tilde{Y}_{ij} = X_{ij} \times Y_{ij} $.
Decreasing $p_{obs}$ increases model mismatch, and Figure~\ref{fig:sensitivity} (left) shows this also increases relative error in the estimate $\hat K$, but the decline is graceful. For some configurations (blue) the relative error stays below 25\% even after dropping 10\% of observations. Note that for this kind of model mismatch the retrieved hyperparameter is consistently larger than the true one; additional components are required to explain the increased variance caused by the excess zeroes.

\paragraph{Overdispersion}
The second experiment considers another prototypical model mismatch, a scenario where a too simple model class is used.  We use PMF as the model, but generate the data so that the Poisson likelihood is replaced with negative binomial (NB) with varying rate of overdispersion that PMF cannot account for.
More specifically, we sample data from NB distribution so that the conditional mean is given by the MF, but the variance is controlled with overdispersion parameter $r$, using the parameterization $\text{NB}(r,p)$ where $r$ is the number of failures until the experiment is stopped and $1-p$ is probability of failure.
Using the notation of Eq.~\ref{eq:genmf}, our data is hence generated by $\eta_{ij} = \left(\sum_{k = 1}^K \theta_{ik}\beta_{jk}\right) $ and $Y_{ij} \sim \text{NB}(r,\frac{\eta_{ij}}{\eta_{ij}+r})$. This implies that $\E[ Y_{ij}  | \eta_{ij}; r ] =  \eta_{ij} $ and $\var[ Y_{ij}  | \eta_{ij} ; r ] =  \eta_{ij} + \frac{\eta_{ij}^2}{r}$, so the smaller the $r$, the more overdispersed the distribution is. We vary $r \in \{ 10, 50, 100, 500, 1000, 5000 \}$, and Figure
~\ref{fig:sensitivity} (right) shows that again the procedure is relatively robust for the mismatch.

Both experiments indicate that the analytic expression for determining $K$ is robust only for small model mismatch, and that the relative error grows drastically as a function of the mismatch. Here both examples correspond to overdispersed data compared to the assumed model -- which is often the case -- which results in overestimating the number of factors. In general, it is difficult to anticipate how specific kinds of model mismatches influence the results, but as a general rule we suggest interpreting odd hyperparameters as signs of potential mismatch worth investigating in more detail. For both of these examples the problem could be solved by switching to an appropriate model, explicitly modeling missing data in the first case and using negative binomial as likelihood in the latter. For both cases we would also need to switch to formulas, adapting the general result of Eq.~\ref{eq:generic_solution}.

\paragraph{Inconsistent margins}
As mentioned in Section~\ref{sec:motivation}, placing independent priors on the factors induces equal margin sums in expectation. More specifically, the expected row-sum is $\E[\sum_i Y_{ij}] = NK\mu_\beta\mu_\theta$ and the expected column-sum is $\E[\sum_j Y_{ij}] = MK\mu_\beta\mu_\theta$. Often the margins of real data sets are, however, quite non-uniform. We already showed in Section~\ref{sec:experiment_bo} that the overall configuration of the hyperparameters can be good even in the recommender engine setups that typically have very non-uniform margins, but it is worthwhile to study the sensitivity to this specific form of model mismatch also in scenarios that can be analysed in more detail.

We consider a case where each row is consistently element-wise multiplied with a positive vector $\bm{\gamma}=[\gamma_1, \ldots, \gamma_M]$, with $\gamma_j \geq 1$, and again focus on the estimator for $K$. In this case, given $\bm{\eta} =\{ \eta_{ij} = \sum_{k = 1}^K \theta_{ik}\beta_{jk} \}$, the conditional mean and variance is  $\E[Y_{ij} | \theta \beta]= \var[Y_{ij} | \theta \beta]= \gamma_j\eta_{ij}$, and the marginal mean and variance $\E[Y_{ij}] = e_j = K\gamma_j\mu_\theta\mu_\beta$ and $\var[Y_{ij}] = v_j = e_j+\gamma_j^2\var[\eta_{ij}]$ with $\var[\eta_{ij}]=K((\mu_{\beta}\sigma_{\theta})^2+(\mu_{\theta}\sigma_{\beta})^2+ (\sigma_{\theta}\sigma_{\beta})^2)$. We observe that each column $j$ have an expected sum controlled by $\gamma_j$, which also influences the degree of overdispersion, and the parameter $\gamma_j$ can be chosen to enforce a desired value the expected column-sum. Furthermore, we can calculate the correlations between $Y_{ij}$ and $Y_{tl}$, resulting in the following expressions for the non-trivial cases:
for the fixed row $i=t$ and different columns $\rho(Y_{ij},Y_{tl})=\rho_{jl}^{(1)}=\frac{K\gamma_j\gamma_l}{\sqrt{v_jv_l}}(\mu_\beta\sigma_\theta)^2$, and for fixed columns $j=l$ and different rows  $\rho(Y_{ij},Y_{tl})=\rho_{j}^{(2)}=\frac{K\gamma_j^2}{v_j}(\mu_\theta\sigma_\beta)^2$. In summary, the mismatch here corresponds to both larger expectation and variance, as well as a distinct correlation structure and column-wise variation which is not present in the PMF. 

We perform an experimental validation by setting the scaling matrix $\bm{\gamma}=[1, \ldots, \gamma_{max}] \in \mathbb{R}^{M}_+$ and varying $\gamma_{max} \in \{1, \ldots, 8 \} $. We use $K=50$, executing 5 runs of the experiment for each prior (Beta, Gamma, Half-Normal and Log-Normal) with different configurations of their respective hyperparameters (detailed in appendix Table~\ref{tab:priors_initialization}), while using the standard equations for PMF to estimate the necessary quantities. Figure~\ref{fig:consitent_rows_rescale} shows that the equation for $K$ consistently underestimates the true value, with worse results as $\gamma_{max}$ increases. This is expected, since we added multiple sources of deviation from the original model assumptions, resulting in equations for the moments that are varying with each column index for mean and variance, and each pair of columns for the correlations.

It is worth noting that this particular form of mismatch could be corrected for by deriving a new estimator for $K$, using the values for $\E[Y_{ij}]$, $\var[Y_{ij}]$ and $\rho(Y_{ij},Y_{tl})$ described above, and new empirical estimates of those quantities. Alternatively, it is possible to recover the original terms of the equations derived for PMF by rescaling the observations with the formulas $\hat Y_{ij}^{(1)}=\frac{Y_{ij}}{\gamma_j}$ and $\hat Y_{ij}^{(2)}=\frac{Y_{ij}}{\gamma_j^2}$, and calculated adjusted empirical estimates for the mean $\E[\hat Y_{ij}^{(1)}]$, variance $\var[\hat Y_{ij}^{(1)}]-\E[\hat Y_{ij}^{(2)}]+\E[\hat Y_{ij}^{(1)}]$ and covariance $\cov(Y_{ij}^{(1)}, Y_{tl}^{(1)})$. In the Appendix~\ref{app:mismatch} these estimates are discussed in more detail, and we show an example of adjusting for model mismatch in this case of column-wise multiplicative factors.

\begin{figure*}[t]
    \centering
    \includegraphics[width=0.4\textwidth]{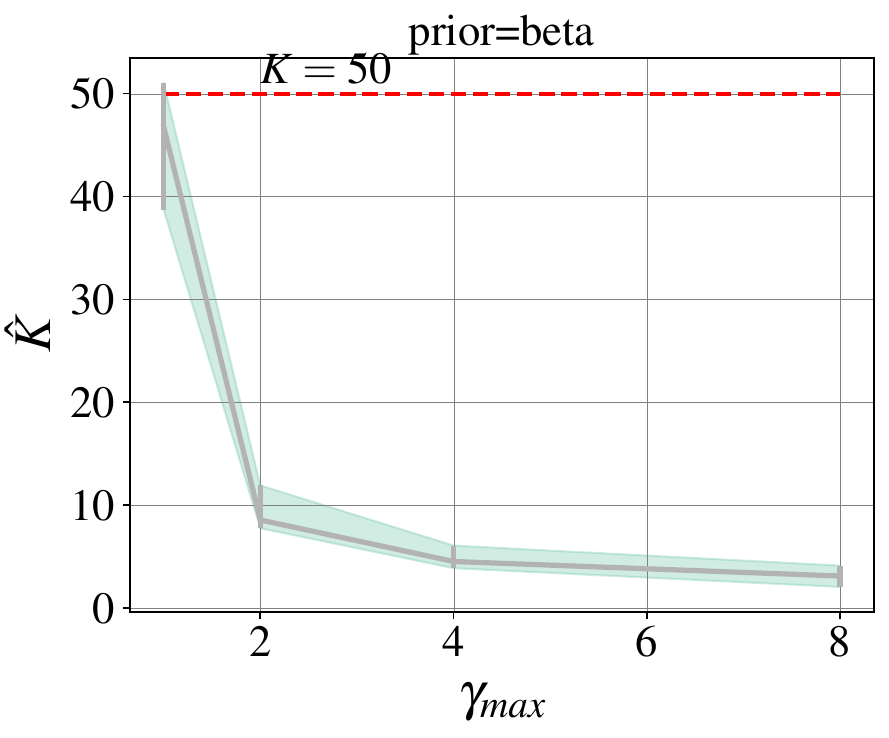}
    \includegraphics[width=0.4\textwidth]{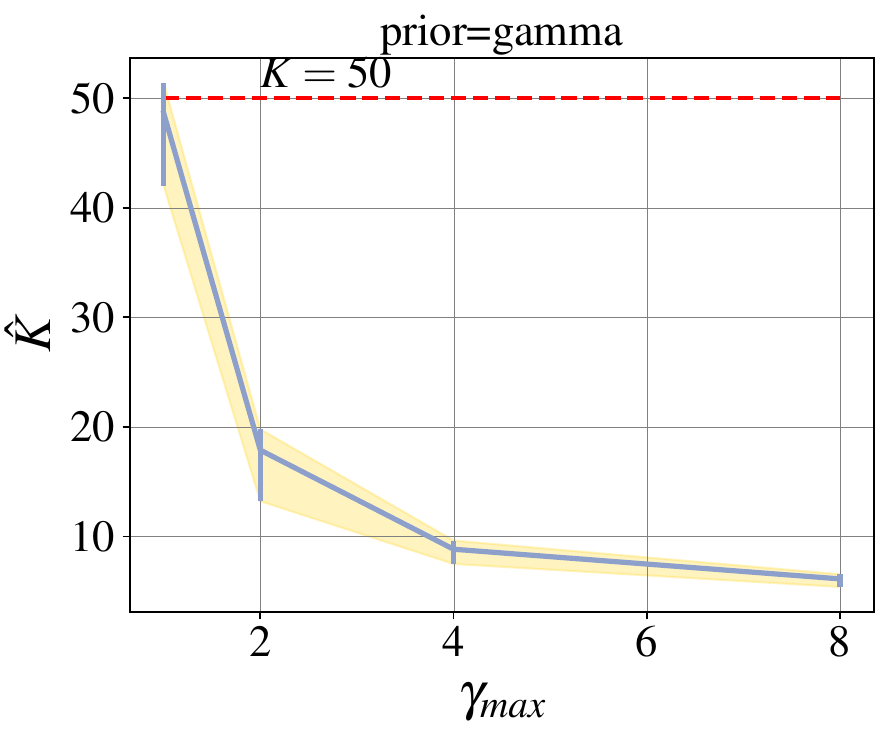} \\
    \includegraphics[width=0.4\textwidth]{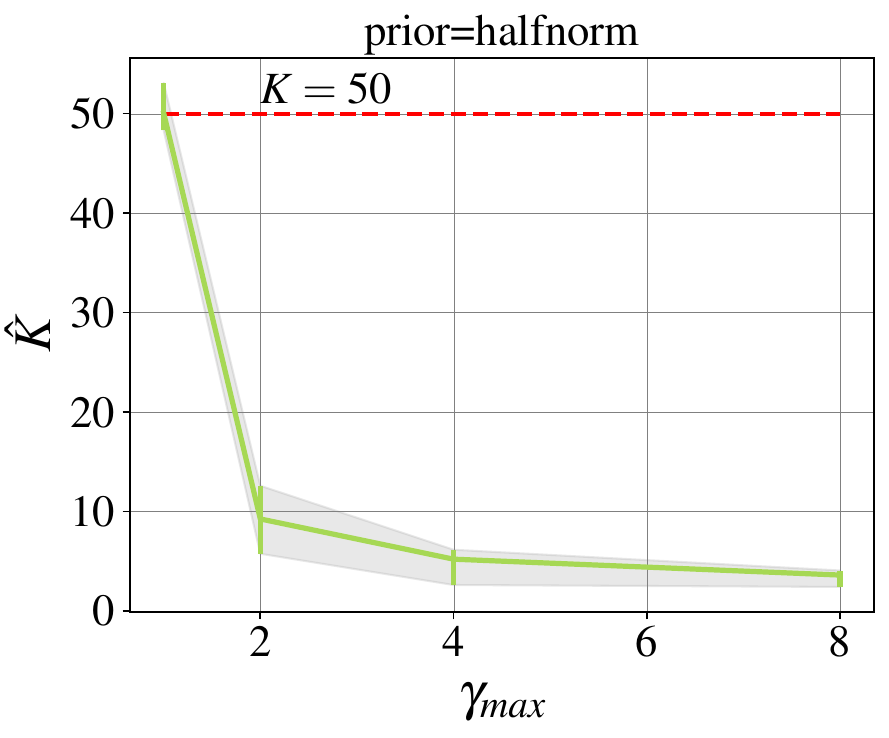}
    \includegraphics[width=0.4\textwidth]{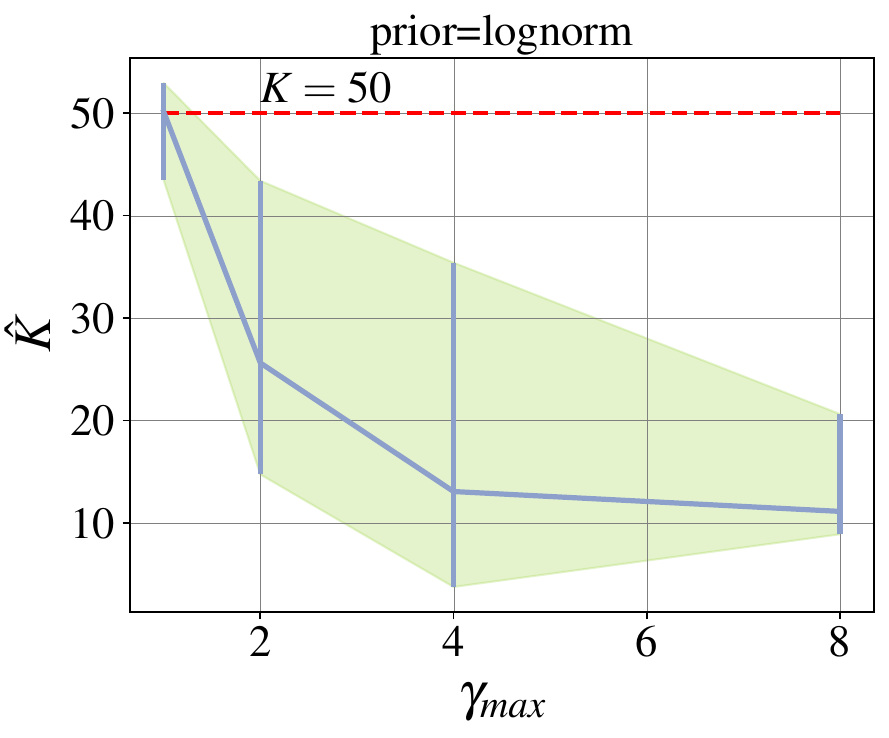}
    \caption{Sensitivity of the  estimated latent dimensionality $\hat K$ when each row of the latent rate matrix $\bm{\eta} =\{ \eta_{ij} = \sum_{k = 1}^K \theta_{ik}\beta_{jk} \} \in \mathbb{R}^{N \times M}_+$ is consistently multiplied element-wise by $[1, \ldots, \gamma_{max}]$. This procedure induces a resulting count-matrix with distinct column-sum counts, where higher column-index display higher values. Each graph was generated using a fixed true dimensionality $K=50$, and different prior distributions: Beta (top left), Gamma (top right), Half-Normal (bottom left) and Log-Normal (bottom right). Results represent median over 5 runs for each prior configuration and its 95\% confidence interval.
    }        
    \label{fig:consitent_rows_rescale}
\end{figure*}

\subsubsection{Stochastic Algorithm Convergence}
\label{sec:convergence_experiment}

\begin{figure*}[t]
    \centering
    \includegraphics[width=0.32\textwidth]{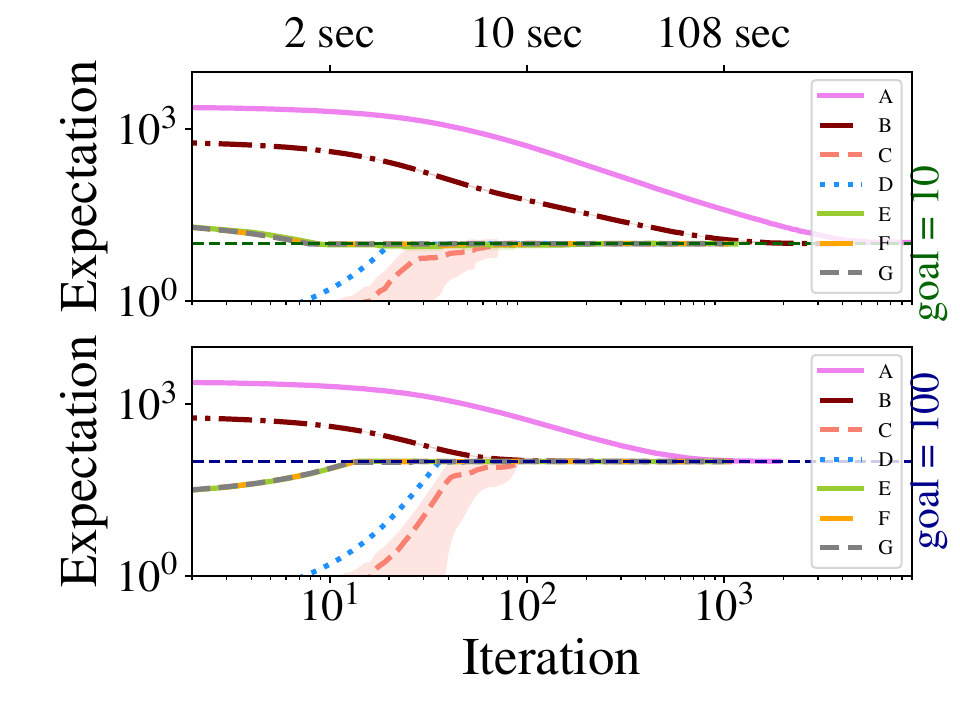}
    \includegraphics[width=0.32\textwidth]{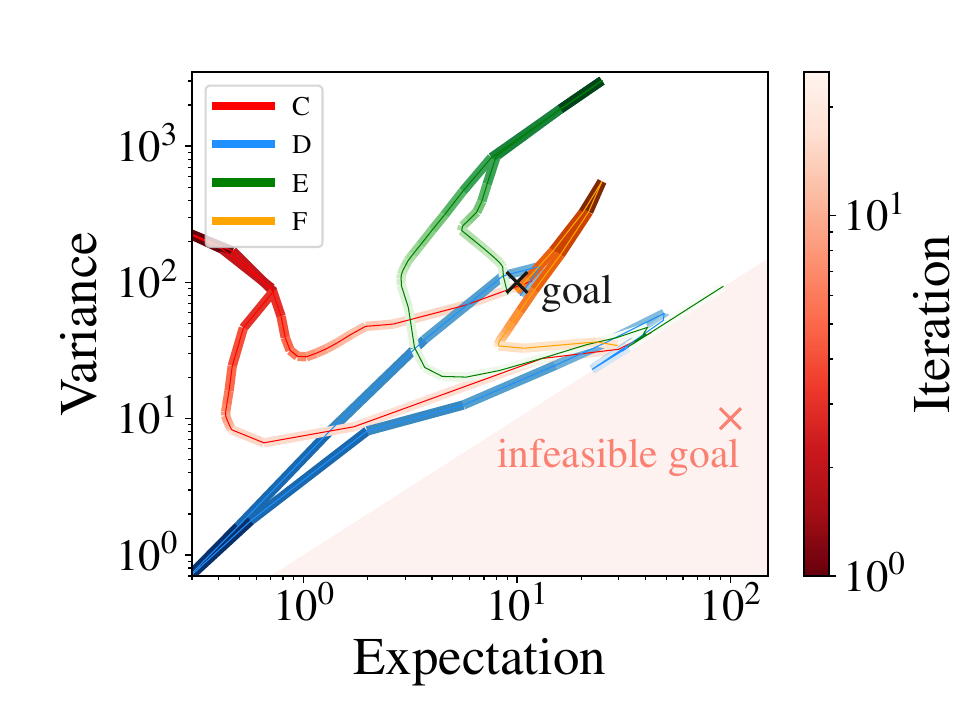}
    \includegraphics[width=0.32\textwidth]{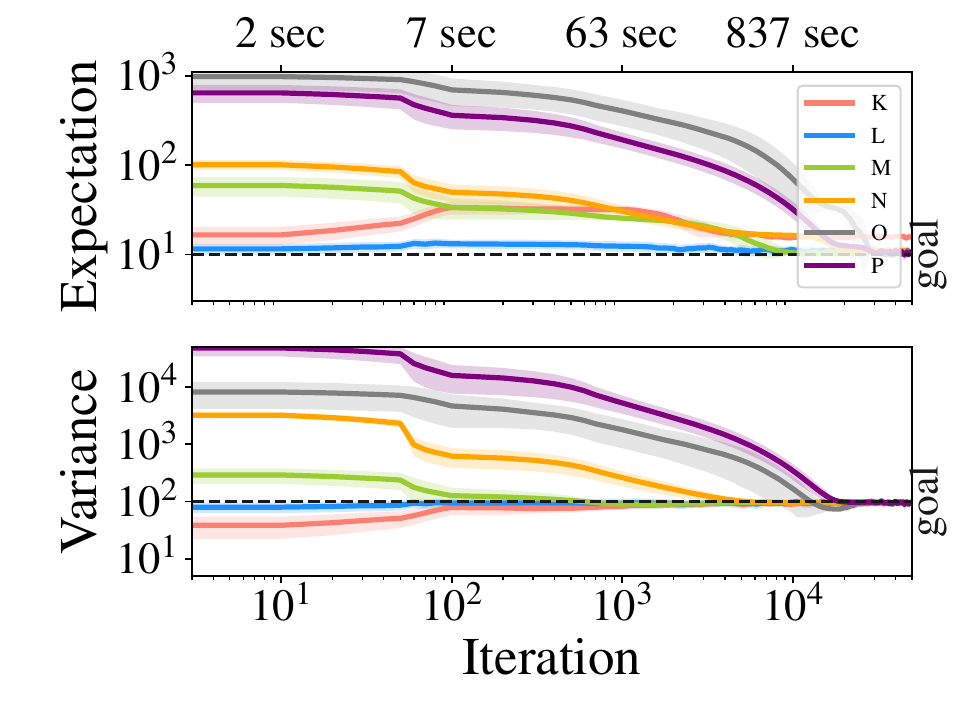} 
    \caption{Convergence of the model independent algorithm for PMF (left and middle) and HPF (right) starting from different initializations.
    }        
    \label{fig:convergence}
\end{figure*}

The stochastic algorithm of Section~\ref{sec:model_independent} can be used for richer model families, but as iterative algorithm requires more careful analysis. We study its properties via convergence speed and sensitivity to initial conditions for two models in the prior elicitation scenario (target values provided by the user), and illustrate an important property regarding scenarios where the target values cannot be satisfied.

Before demonstrating the method in practice, we note that
the algorithm itself
relies on Monte Carlo estimates for the virtual statistics. The statistics are not necessarily accurate. Hence, we empirically evaluate the bias and variance of estimators in Appendix~\ref{sec:experiment_bias_variance}, showing that for estimating the mean the bias is typically negligible but that for the variance estimator the bias can be large for small values of $\var[Y]$.
Furthermore, convergence of gradient-based optimization algorithms depends on how the optimization space is structured. Therefore, we compare two alternative parametrizations for PMF in Appendix~\ref{sec:parametrization}, showing that good model parametrization (e.g. in terms of means and variances) may simplify the solved optimization problem.

We first apply the algorithm for PMF, to demonstrate we can replicate the behavior of the analytic solution (but using fixed $K=25$, since the algorithm currently does not support discrete hyperparameters). Figure~\ref{fig:convergence} (left) shows convergence plots for two target values for the mean $\hat{\E}[Y]$ and seven initial values (Table~\ref{tab:pmf_initializations} in the Appendix), demonstrating that we find the optimal solution typically in seconds and even for intentionally bad initializations (scenario "A" starts with orders of magnitude too large mean) in minutes. The optimization was carried out in location-scale parameterization, using 
Adam optimizer with learning rate $=0.1$ and estimating $\hat{E}[Y]$ with $S_y = 3$ samples of observed variable for each of $S_1 = 100$ samples of latent variables.

Figure~\ref{fig:convergence} (right) shows similar convergence curves for hierarchical PMF \citep{gopalan2015scalable}, optimizing for both mean and variance targets provided by the user with $\discrepancy=(\hat{\E}[Y]-10)^2+(\hat{\V}[Y]-100)^2$, illustrating that the algorithm works for more complex MF models and discrepancies as well. Again we find the solution irrespective of the initial value (Table~\ref{tab:hpf_initializations} in the Appendix), but convergence is naturally slower due to higher dimensionality (2 more hyperparameters) and the hierarchical structure.  The convergence of the optimizer is highly insensitive to the relative weighting of the two statistics, resulting only in minor speed differences (not shown here).

Finally, Figure~\ref{fig:convergence} (middle) demonstrates the behavior of the algorithm for two PMF cases optimized for mean and variance, one with feasible and one with infeasible target moments. Besides users specifying contradicting goals in the prior elicitation case, the latter case can also happen with the empirical Bayes scenario if the data does not exactly match the model. For PMF we must have $\var[Y_{ij}] \ge \E[Y_{ij}]$ 
and hence here the target $(\text{E}^*, \text{V}^*) = (100, 10)$ cannot be reached. The algorithm correctly handles this, by converging to the border of the feasible region and returning a non-zero discrepancy, providing valuable feedback for the user on possible model mismatch or poorly specified targets. The weighting of the two statistics determines what kind of compromise the algorithm makes and allows freedom for the user to indicate e.g. reliability of their given targets.

\section{Discussion and Conclusion}

The hyperparameters of Bayesian machine learning models
are often overlooked and chosen based on vague intuitions instead of more rigorous practices, e.g., gamma priors with small mean and large variance are used to induce sparsity. 
This is understandable as strict Bayesian paradigm requires specifying the prior independently of the data \citep{garthwaite:2005,ohagan:2006}, whereas standard optimization procedures require fitting the model for some data to evaluate the quality.
Even if accepting that priors are to be adapted based on available data (which is admissible for latent-variable models whose priors do not really encode subjective prior knowledge), optimizing for the hyperparameters is challenging as posterior inference is slow and the optimization surface can be difficult.

We provided a framework for selecting hyperparameters for Bayesian  models that: (a) better fits within the Bayesian paradigm by only relying on simple data summaries or by entirely ignoring the data and only using expected moments provided by a user; (b) is efficient in contrast to the traditional alternatives by avoiding posterior inference when specifying priors. The framework builds on the prior predictive distribution, previously used for prior predictive elicitation \citep{kadane:1980} and model checking \citep{schad:2019, gelman2020bayesian}, converted here into a tool for ML practitioners working with large models. From another perspectice, the approach generalizes methods of moments and maximum marginal likelihood.

We demonstrated the approach for Bayesian matrix factorization models, for which the choice of priors and hyperparameters is difficult. For Poisson MF \citep{DBLP:journals/cin/Cemgil09} and Compound PMF \citep{BasbugE16} we derived analytic expressions that provide immediate solutions (i.e., hyperparameters optimal in terms of PPD matching the desired statistics)
and, in particular, we now have an analytic expression for determining the number of latent factors. 
Finally, 
we also presented a model-independent optimization algorithm that was here demonstrated for a hierarchical MF model \citep{gopalan2015scalable}, 
but can be applied for a more general class of models. In this work, we only demonstrated the approach in context of matrix factorization to keep the scope of the paper more limited, but the algorithm applies to any probabilistic model as long as we can sample from the model in a differentiable manner. Empirical evaluation of the approach in context of other model families remains an interesting future direction.

The overall methodology is generic, supporting in principle arbitrary choice of the statistics and discrepancy measures and not taking any stance on how the values for the target statistics are determined. Such a flexibility means the method requires some care when used. We believe that as long as the set of statistics are limited to a few low-order moments it is safe to estimate the target values directly from data, but cannot say in detail what would happen if using considerably higher order moments or other more detailed statistics. Our practical recommendation is to start with the obvious statistics of mean and variance, and then add additional ones if considered relevant for the problem or if the solution is still too undetermined. Besides correlation considered in this work, we see potential in using kurtosis, skewness, extreme values and selected quantiles of the data distribution, already used in prior elicitation protocols.

Perhaps the most important aspect requiring expertise from the user concerns model mis-spefication. We empirically studied three different forms of mis-specification to highlight varying degrees of robustness. For instance, for PMF the approach is robust for minor over-dispersion and zero-inflation, but suggests too low number of factors for cases where the margin counts do not follow the prior assumption of uniformity. However, it can still work well in practice even in such cases, as demonstrated by good performance in posterior quality comparisons on real recommender engine data with clearly non-uniform margins.
In general, the results for models that fit the data very poorly are not necessarily even sensible. Our best recommendation here is for practitioners to use this results as a starting point for the prior specification in the modeling workflow, but not to necessarily directly rely on the solution. For example, if our method suggests using only a few factors for a recommender engine task commonly modeled with dozens of factors, our recommendation is to carefully consider whether the specific probabilistic model matches the particular data at hand or whether, e.g., changing the likelihood or model structure would be in order. Alternatively, the solution can be used as an initial point for Bayesian optimization.

Our results demonstrate practical feasibility of the idea, but further improvements would help make it a routine part of the Bayesian ML development pipeline. Promising directions for improving the model-independent algorithm include:
(1) more careful analysis of the moment and discrepancy measure choices, possibly building on the existing work done for approximate Bayesian computation, (2) automatic choice of good parameterization for the hyperparameter space, possibly along the lines of \citet{gorinova2019automatic}, to support generic probabilistic programs, (3) better empirical estimators for the gradients \citep{mohamed2019monte} to reduce optimization noise, (4) support for discrete hyperparameters, and (5) better means for identifying failure modes without relying on external validation or theoretical analysis.

\acks{The work was supported by Academy of Finland, under grant 1313125, as well as the Finnish Center for Artificial Intelligence (FCAI), a Flagship of the Academy of Finland.
We also thank the Finnish Grid and Cloud Infrastructure (urn:nbn:fi:research-infras-2016072533) for providing computational resources. We also thank Norwegian Open AI Lab - NTNU for the support with computational resources and funding. We thank Helge Langseth, Heri Ramampiaro (NTNU/Trondheim) and Ant\'onio G\'ois (MILA) for useful discussions and critical review of the publication.}

\appendix

\begin{appendices}

\section{Setup of the experiments}\label{append:experiments}

Tables~\ref{tab:pmf_initializations} and \ref{tab:hpf_initializations} contain initial configurations used in experiments with respectively PMF and HPF. The legends in Figures \ref{fig:sensitivity}, \ref{fig:surfaces} and \ref{fig:nsamples} refer to these letters.

Table~\ref{tab:priors_initialization} show the configurations of the priors used in the experiments that evaluated the methods with multiple priors. For each prior distribution of the parameters $\theta$ and $\beta$, hyperparam\_1 and hyperparam\_2 represent the location and scale, according the standard parameterization used in \texttt{numpy}\footnote{\url{https://numpy.org/doc/stable/}} library implementation. Table~\ref{tab:distributions_params} show for each distribution what is the location and scale parameterization corresponding to the values of hyperparam\_1 and hyperparam\_2, with its respective mean and variance. Table~\ref{tab:distributions_params} include also the distributions used for the experiment with different observations models, namely Gumbel, Laplace and Normal distribution. In this experiment the location parameter was set such that the conditional mean $\E[Y_{ij} | \theta, \beta]=\sum_{k=1}^{K}\theta{ik}\beta_{jk}$ and the scale (for each distribution) with values $\{0.5, 1.0, \dots, 3.5, 4.0, 4.5, \dots, 7.5, 8.0\}$, leading to the expected conditional variance $\E[\var(Y_{ij} | \theta, \beta)]$ having values according to the formulas for the variance in Table~\ref{tab:distributions_params} (for Gumbel, Laplace and Normal distributions).

\begin{table}[t]
    \caption{Considered sets of initial PMF hyperparameters.}
    \centering
    \begin{tabular}{c|llll|llll|rr}
 & $a$ & $b$ & $c$ & $d$ & $\mu_\theta$ & $\sigma_\theta$ & $\mu_\beta$ & $\sigma_\beta$ & $\E[Y]$ & $\var[Y]$ \\
\hline
\hline
A  &  10 & 1 & 10 & 1  &  10.0  & 3.16  & 10.0  & 3.16   &  2500.00  & 55000.00 \\
B  &  10 & 2 & 10 & 2  &  5.0  & 1.58  & 5.0  & 1.58   &  625.00  & 3906.25 \\
\hline
C  &  0.001 & 0.01 & 0.01 & 0.1  &  0.1  & 3.16  & 0.1  & 1.0   &  0.25  & 253.00 \\
D  &  0.1 & 1 & 0.1 & 1  &  0.1  & 0.32  & 0.1  & 0.32   &  0.25  & 0.55 \\
\hline
E  &  0.1 & 0.1 & 0.1 & 0.1  &  1.0  & 3.16  & 1.0  & 3.16   &  25.00  & 3025.00 \\
F  &  1 & 1 & 0.1 & 0.1  &  1.0  & 1.0  & 1.0  & 3.16   &  25.00  & 550.00 \\
G  &  1000 & 1000 & 1000 & 1000  &  1.0  & 0.03  & 1.0  & 0.03   &  25.00  & 25.05 \\
    \end{tabular}
    \label{tab:pmf_initializations}
\end{table}

\begin{table}[t]
    \caption{Considered sets of initial HPF hyperparameters.}
    \centering
    \begin{tabular}{c|llllll|rr}
 & $a$ & $a'$ & $b'$ & $c$ & $c'$ & $d'$  & $\E[Y]$ & $\var[Y]$ \\
\hline
\hline
K  &  1.0 & 100.0 & 10.0 & 1.0  &  100.0  & 10.0    &  0.26  & 0.26 \\
L  &  0.1 & 100.0 & 1.0 & 1.0  &  100.0  & 1.0    &  2.55  & 8.25 \\
M  &  50.0 & 5000.0 & 10.0 & 1.0  &  5000.0  & 1.0    &  125.05  & 781.42 \\
N  &  1.0 & 100.0 & 1.0 & 10.0  &  10.0  & 1.0    &  280.57  & 15309.26 \\
O  &  450.0 & 4500.0 & 100.0 & 10.0  &  400.0  & 1.0    &  1128.10  & 9833.55 \\
P  &  50.0 & 50.0 & 1.0 & 1.0  &  50.0  & 1.0    &  1301.71  & 146074.01 \\
    \end{tabular}
    \label{tab:hpf_initializations}
\end{table}

\begin{table}[t]
    \caption{Set of hyperparameters configurations for the experiments with multiple priors.}
    \centering
\begin{tabular}{ll|ll|l}
$\theta$ hyperparam\_1 & $\theta$ hyperparam\_2 & $\beta$ hyperparam\_1 & $\beta$ hyperparam\_2 & prior    \\ 
\hline
\hline
1.0                       & 1.0                       & 1.0                        & 1.0                        & gamma    \\
10.0                      & 1.0                       & 10.0                       & 1.0                        & gamma    \\
1.0                       & 10.0                      & 1.0                        & 10.0                       & gamma    \\
0.1                       & 1.0                       & 1.0                        & 0.1                        & gamma    \\
\hline
1.0                       & 1.0                       & 1.0                        & 1.0                        & beta     \\
1.0                       & 2.0                       & 1.0                        & 1.0                        & beta     \\
1.0                       & 2.0                       & 1.0                        & 2.0                        & beta     \\
1.0                       & 2.0                       & 2.0                        & 1.0                        & beta     \\
\hline
0.0                       & 1.0                       & 0.0                        & 1.0                        & halfnorm \\
0.0                       & 2.0                       & 0.0                        & 2.0                        & halfnorm \\
0.5                       & 2.0                       & 1.0                        & 2.0                        & halfnorm \\
0.5                       & 0.5                       & 2.0                        & 1.0                        & halfnorm \\
\hline
1.0                       & 1.0                       & 0.0                        & 1.0                        & lognorm  \\
1.0                       & 2.0                       & 0.0                        & 2.0                        & lognorm  \\
1.0                       & 2.0                       & 0.5                        & 0.5                        & lognorm  \\
1.0                       & 0.5                       & 0.5                        & 2.0                        & lognorm  \\ 
\end{tabular}
    \label{tab:priors_initialization}
\end{table}

\begin{table}[t]
    \caption{Parameterization of different distributions}
    \centering
    \begin{tabular}{c|c|c|c}
 distribution & location, scale & $\E$ & $\var$ \\
\hline
\hline
$\gam(a,b)$  &   $a,b$ (shape and scale)  &  $ab$  & $ab^2$ \\
\hline
$\text{Beta}(a,b)$  &   $a,b$  &  $\frac{a}{a+b}$  & $\frac{ab}{(a+b)^2(a+b+1)}$ \\
\hline
$\text{HalfNormal}(\mu,\sigma)$  &   $\mu,\sigma$  &  $\mu+\sigma\sqrt{\frac{2}{\pi}}$  & $\sigma^2(1-\frac{2}{\pi})$ \\
\hline
$\text{LogNormal}(\mu,\sigma)$  &   $\mu,\sigma$  &  $\text{exp}(\mu+\frac{\sigma^2}{2})$  & $(\text{exp}(\sigma^2)-1)\text{exp}(2\mu+\sigma^2)$ \\

\hline
$\text{Gumbel}(\mu,\sigma)$  &   $\mu,\sigma$  &  $\mu+\gamma\sigma$  & $\frac{\pi^2}{2}\sigma^2$\\

\hline
$\text{Laplace}(\mu,b)$  &   $\mu,\sigma$  &  $\mu$  & $\frac{\pi^2}{2}\sigma^2$\\
\hline
$\text{Normal}(\mu,\sigma^2)$  &   $\mu,\sigma$  &  $\mu$  & $\sigma^2$ \\

\hline
$\text{Poisson}(\eta)$  &   $\eta$  &  $\eta$  & $\eta$

    \end{tabular}
    \label{tab:distributions_params}
\end{table}

\section{Derivation of Analytic Solution for Poisson Matrix Factorization}

We consider the model defined in Eq.~\ref{eq:genmf}. We reinstate that for the sake of notational convenience we will adopt the following notation $ \lambda \eqdef \{ \mu_{\theta},\sigma_{\theta}^2, \mu_{\beta},\sigma_{\beta}^2 \} $.  We are interested in computing $\E[Y_{ij} ; \lambda ]$, $\var[Y_{ij} ; \lambda ]$, $\cov[Y_{ij}, Y_{tl}  ; \lambda ]$ and $\rho[Y_{ij}, Y_{tl}  ; \lambda ]$ denotating here the virtual statistics derived from PDD. We want to marginalize out the latent variables to be able to calculate a mathematical relation between the hyperparameters and the virtual statistics.\footnote{For notation brevity assume the hyperparameter $\lambda$ implicitily, for example $\E[Y_{ij}] \eqdef \E[Y_{ij} ; \lambda ] = \int Y_{ij} p(Y_{ij}|\sum_k \theta_{ik}\beta_{jk} )\prod_k p(\theta_{ik} ; \lambda) p(\beta_{jk} ; \lambda) d\theta_{ik} d\beta_{jk} d Y_{ij} $. }
\subsection{Preliminaries}

We will make use of Total Expectation Law, Total Variance Law and Total Covariance Law and the following propositions without proofs:

\begin{align}
    \E[Y]&=\E[\E[Y | X] ] \label{eq:tote}\\
    \var[Y]&=\E[\var[Y | X] ]+\var[\E[ Y | X] ] \label{eq:totvar} \\
    \cov[X,Y]&=\E[\cov[X, Y | Z]] +\cov[\E[ X | Z], \E[ Y | Z] ] \label{eq:totcov} \\
    \var[XY]&=\E[X^2]\E[Y^2]-\E[X]^2\E[Y]^2 \label{eq:varmult}\text{ if X and Y are independent} \\
    \var[XY]&=\E[X]^2\var[Y]+\E[Y]^2\var[X]+\var[X]\var[Y] \label{eq:varmult2} \text{ if X and Y are independent} \\
    \E[X^2]&=\var[X]+\E[X]^2 \label{eq:ex2} \\
    \var[\sum_k X_k ]&=\sum_k \var[X_k]+2\sum_{k<k'}\cov[X_k,X_{k'}]. \label{eq:varsum}
\end{align}

These relations will be useful in the task of marginalizing out the latent parameters of the hierarchical model, as we shall see in the following section.
\subsection{Intermediate results}
We will start by proving some intermediate results that are useful in different steps for the final results. 
\begin{proposition}\label{prop:cov1}
For any combination of valid values for the indexes $i$,$j$,$t$ and $l$, if the latent indexes $k \neq k'$, then $\cov[\theta_{ik}\beta_{jk}, \theta_{tk'}\beta_{lk'}] = 0$.
\end{proposition}
\begin{proof}
By definition of the covariance $$\cov[\theta_{ik}\beta_{jk}, \theta_{tk'}\beta_{lk'}] = \E[\theta_{ik}\beta_{jk} \theta_{tk'}\beta_{lk'}]-\E[\theta_{ik}\beta_{jk}]\E[ \theta_{tk'}\beta_{lk'}].$$
Given that $k \neq k'$, this implies (for any combination of the other indices) $\E[\theta_{ik}\beta_{jk} \theta_{tk'}\beta_{lk'}]=\E[\theta_{ik}\beta_{jk}]\E[ \theta_{tk'}\beta_{lk'}]$.
\end{proof}

\begin{proposition}\label{prop:var1}
For any combination of valid values for the indexes $i$,$j$ and $k$ the following equations hold:
\begin{enumerate}
    \item $\E[\sum_k \theta_{ik}\beta_{jk}] = K\mu_{\theta}\mu_{\beta}$;
    \item $\var[\sum_k \theta_{ik}\beta_{jk}] = K[(\mu_{\beta}\sigma_{\theta})^2+(\mu_{\theta}\sigma_{\beta})^2+ (\sigma_{\theta}\sigma_{\beta})^2]$;
    \item $\cov[\sum_k \theta_{ik}\beta_{jk},\sum_k \theta_{tk}\beta_{lk}] = K[\delta_{it}(\mu_{\beta}\sigma_{\theta})^2+\delta_{jl}(\mu_{\theta}\sigma_{\beta})^2+ \delta_{it}\delta_{jl}(\sigma_{\theta}\sigma_{\beta})^2)]$.
\end{enumerate}
\end{proposition}
\begin{proof}:

\begin{enumerate}
    \item For the first equation we apply the summation property of the expected value and the fact that $\theta_{ik}$ and $\beta_{jk}$ are independent. 
    \item For $\var[\sum_k \theta_{ik}\beta_{jk}]$, we start by using Eq.~\ref{eq:varsum}, thus resulting in $\var[\sum_k \theta_{ik}\beta_{jk}] = \sum_k \var[\theta_{ik}\beta_{jk}]+2\sum_{k<k'}\cov[\theta_{ik}\beta_{jk},\theta_{ik'}\beta_{jk'}]$. However, from Proposition~\ref{prop:cov1} we know the covariance terms where the indexes $k$ and $k'$ are not the same should be zero, resulting in  $\var[\sum_k \theta_{ik}\beta_{jk}] = \sum_k \var[\theta_{ik}\beta_{jk}]$. Now using \eqref{eq:varmult2} for the variance of the product of random variables we obtain \begin{align*}
   \var[\theta_{ik}\beta_{jk}] & = \E[\theta_{ik}]^2\var[\beta_{jk}]+\E[\beta_{jk}]^2\var[\theta_{ik}]+\var[\theta_{ik}]\var[\beta_{jk}]\\
   & = \mu_\theta^2\sigma_\beta^2+\mu_\beta^2\sigma_\theta^2+\sigma_\beta^2\sigma_\theta^2\\
   \Rightarrow \var[\sum_k \theta_{ik}\beta_{jk}] &= K[(\mu_{\beta}\sigma_{\theta})^2+(\mu_{\theta}\sigma_{\beta})^2+ (\sigma_{\theta}\sigma_{\beta})^2].
\end{align*}
   \item For the last equation we start with the definition of covariance: 
\begin{align*}
\cov[\sum_k \theta_{ik}\beta_{jk},\sum_k \theta_{tk}\beta_{lk}] &= \E[\sum_{k,k'} \theta_{ik}\beta_{jk}\theta_{tk'}\beta_{lk'}]-\E[\sum_{k} \theta_{ik}\beta_{jk}]\E[\sum_{k'} \theta_{tk'}\beta_{lk'}] \\
&= \sum_{k,k'} \underset{\cov[\theta_{ik}\beta_{jk},\theta_{tk'}\beta_{lk'}]}{\underbrace{\E[\theta_{ik}\beta_{jk}\theta_{tk'}\beta_{lk'}]- \E[\theta_{ik}]\E[\beta_{jk}]\E[\theta_{tk'}]\E[\beta_{lk'}]}}. 
\end{align*}
Considering Proposition~\ref{prop:cov1}, we know that only the shared indices $k$ are non zero, thus simplyfiying to
\begin{align}
\cov[\sum_k \theta_{ik}\beta_{jk},\sum_k \theta_{tk}\beta_{lk}]
= \sum_{k} \E[\theta_{ik}\beta_{jk}\theta_{tk}\beta_{lk}]- \E[\theta_{ik}]\E[\beta_{jk}]\E[\theta_{tk}]\E[\beta_{lk}]  \label{eq:cov_inter1} \\
 = \sum_k \cov[\theta_{ik}\beta_{jk},\theta_{tk}\beta_{lk}] \nonumber.
\end{align}

Now, we can calculate $ \cov[\theta_{ik}\beta_{jk},\theta_{tk}\beta_{lk}]$ for four different cases:
    \begin{enumerate}
        \item if $i \neq t \And j \neq l$: because of independence of all variables, we obtain $$\cov[\theta_{ik}\beta_{jk},\theta_{tk}\beta_{lk}]=0;$$
        \item if $i = t \And j \neq l$:
        \begin{align}
            \cov[\theta_{ik}\beta_{jk},\theta_{ik}\beta_{lk}] &= \E[\theta_{ik}^2\beta_{jk}\beta_{lk}]- \E[\theta_{ik}]^2\E[\beta_{jk}]\E[\beta_{lk}] \nonumber \\
            &= \E[\theta_{ik}^2]\E[\beta_{jk}]E[\beta_{lk}]- \E[\theta_{ik}]^2\E[\beta_{jk}]\E[\beta_{lk}] \nonumber \\
            &= \E[\beta_{jk}]E[\beta_{lk}](\E[\theta_{ik}^2]- \E[\theta_{ik}]^2) %
            = \mu_\beta^2\sigma_\theta^2 ;\nonumber
        \end{align}
        \item if $i \neq t \And j = l$:
        \begin{align}
            \cov[\theta_{ik}\beta_{jk},\theta_{tk}\beta_{jk}] &= \E[\beta_{jk}^2]\E[\theta_{ik}]E[\theta_{tk}]- \E[\beta_{jk}]^2\E[\theta_{ik}]E[\theta_{tk}] %
            = \mu_\theta^2\sigma_\beta^2; \nonumber
        \end{align}
        \item if $i = t \And j = l$:
        \begin{align}
        \cov[\theta_{ik}\beta_{jk},\theta_{ik}\beta_{jk}] =\var[\theta_{ik}\beta_{jk}] %
        =(\mu_{\beta}\sigma_{\theta})^2+(\mu_{\theta}\sigma_{\beta})^2+ (\sigma_{\theta}\sigma_{\beta})^2. \nonumber
        \end{align}
    \end{enumerate}
    Putting all together using Kronecker delta for the indices in the different cases we obtain
    \begin{align}
        \cov[\theta_{ik}\beta_{jk},\theta_{tk}\beta_{lk}] = \delta_{it}(\mu_{\beta}\sigma_{\theta})^2+\delta_{jl}(\mu_{\theta}\sigma_{\beta})^2+ \delta_{it}\delta_{jl}(\sigma_{\theta}\sigma_{\beta})^2). \label{eq:cov_inter2}
    \end{align}
    
    We obtain the final results combining \eqref{eq:cov_inter1} and \eqref{eq:cov_inter2}
    
  $$
 \Rightarrow \cov[\sum_k \theta_{ik}\beta_{jk},\sum_k \theta_{tk}\beta_{lk}]
=  K[ \delta_{it}(\mu_{\beta}\sigma_{\theta})^2+\delta_{jl}(\mu_{\theta}\sigma_{\beta})^2+ \delta_{it}\delta_{jl}(\sigma_{\theta}\sigma_{\beta})^2) ] .$$
\end{enumerate}
\end{proof}

\subsection{Expected values and variance}

We can now proceed to derive the prior predictive expected value and variance. 

\begin{proposition}\label{prop:expvar}
For any combination of valid values for the indexes $i$,$j$ and the following equations hold:
\begin{enumerate}\label{prop:kvare}
    \item $\E[Y_{ij}] = K\mu_{\theta}\mu_{\beta}$;
    \item $\var[Y_{ij}] =K[\mu_{\theta}\mu_{\beta}+ (\mu_{\beta}\sigma_{\theta})^2+(\mu_{\theta}\sigma_{\beta})^2+ (\sigma_{\theta}\sigma_{\beta})^2]$.
\end{enumerate}
\end{proposition}
\begin{proof}
By the law of total expectation, $\E[Y_{ij}] = \E[  \E[Y_{ij} | \sum_k \theta_{ik}\beta_{jk}] ] = \E[\sum_k \theta_{ik}\beta_{jk}]$, already calculated in Proposition~\ref{prop:var1}.
For the second equation we use the law of total variance (\eqref{eq:totvar}) $\var[Y_{ij}] = \E[\var[Y_{ij}|\sum_k\theta_{ik}\beta_{jk}] ]+\var[\E[ Y_{ij}|\sum_k\theta_{ik}\beta_{jk} ] ]$, and because we have a Poisson likelihood we know that $ \var[Y_{ij}|\sum_k\theta_{ik}\beta_{jk}] = \E[Y_{ij}|\sum_k\theta_{ik}\beta_{jk}]=\sum_k\theta_{ik}\beta_{jk} $. Now putting both together and using Proposition~\ref{prop:var1} we obtain
\begin{align*}
    \var[Y_{ij}] = \E[\sum_k\theta_{ik}\beta_{jk}]+\var[\sum_k\theta_{ik}\beta_{jk}] %
    = K[\mu_{\theta}\mu_{\beta}+(\mu_{\beta}\sigma_{\theta})^2+(\mu_{\theta}\sigma_{\beta})^2+ (\sigma_{\theta}\sigma_{\beta})^2].
\end{align*}

\end{proof}

\subsection{Covariance and correlation}

Finally, combining the previous results we can obtain the covariance and correlation for the prior predictive distribution

\begin{proposition}\label{prop:cov}
The prior predictive covariance is given by
$$ \cov[Y_{ij},Y_{tl}] = K[  \delta_{it}(\mu_{\beta}\sigma_{\theta})^2+\delta_{jl}(\mu_{\theta}\sigma_{\beta})^2+ \delta_{it}\delta_{jl}(\mu_{\theta}\mu_{\beta}+(\sigma_{\theta}\sigma_{\beta})^2)].$$
\end{proposition}
\begin{proof}
Using the law of total covariance (\eqref{eq:totcov}):
\begin{align*}
    \cov[Y_{ij},Y_{tl}] & = \E[\cov[Y_{ij},Y_{tl}|\theta_{i.},\beta_{j.},\theta_{t.},\beta_{l.}]]+\cov[\E[Y_{ij}|\theta_{i.},\beta_{j.}],\E[Y_{tl} |\theta_{t.},\beta_{l.}] ] \\
    &= \E[\delta_{it}\delta_{jl}\var[Y_{ij}|\theta_{i.},\beta_{j.}]]+\cov[\sum_k \theta_{ik}\beta_{jk},\sum_k \theta_{tk}\beta_{lk}] \\
     &= \E[\delta_{it}\delta_{jl}\sum_k \theta_{ik}\beta_{jk}]+ K[ \delta_{it}(\mu_{\beta}\sigma_{\theta})^2+\delta_{jl}(\mu_{\theta}\sigma_{\beta})^2+ \delta_{it}\delta_{jl}(\sigma_{\theta}\sigma_{\beta})^2) ]  \\
     &= K[  \delta_{it}(\mu_{\beta}\sigma_{\theta})^2+\delta_{jl}(\mu_{\theta}\sigma_{\beta})^2+ \delta_{it}\delta_{jl}(\mu_{\theta}\mu_{\beta}+(\sigma_{\theta}\sigma_{\beta})^2)].
\end{align*}
\end{proof}

\begin{proposition}\label{prop:corr}
The prior predictive correlation is given by
$$ \rho[Y_{ij},Y_{tl}] = \frac{\delta_{it}(\mu_{\beta}\sigma_{\theta})^2+\delta_{jl}(\mu_{\theta}\sigma_{\beta})^2+ \delta_{it}\delta_{jl}(\mu_{\theta}\mu_{\beta}+(\sigma_{\theta}\sigma_{\beta})^2)}{\mu_{\theta}\mu_{\beta}+ (\mu_{\beta}\sigma_{\theta})^2+(\mu_{\theta}\sigma_{\beta})^2+ (\sigma_{\theta}\sigma_{\beta})^2}  $$
or alternatively
$$ \rho[Y_{ij},Y_{tl}] = \begin{cases} 
0,\text{if }i \neq t \And j \neq l\\
1,\text{if }i = t \And j = l \\
\rho_1 = \frac{(\mu_{\beta}\sigma_{\theta})^2}{\mu_{\theta}\mu_{\beta}+ (\mu_{\beta}\sigma_{\theta})^2+(\mu_{\theta}\sigma_{\beta})^2+ (\sigma_{\theta}\sigma_{\beta})^2},\text{if }i = t \And j \neq l \\
\rho_2=\frac{(\mu_{\theta}\sigma_{\beta})^2}{\mu_{\theta}\mu_{\beta}+ (\mu_{\beta}\sigma_{\theta})^2+(\mu_{\theta}\sigma_{\beta})^2+ (\sigma_{\theta}\sigma_{\beta})^2},\text{if }i \neq t \And j = l
\end{cases}.$$
\end{proposition}
\begin{proof}
From the definition of correlation we have
\begin{align*}
    \rho[Y_{ij},Y_{tl}] &= \frac{\cov[Y_{ij},Y_{tl}] }{\sqrt{\var[Y_{ij}] \var[y_{tlj}] }} %
    =\frac{\cov[Y_{ij},Y_{tl}] }{\sqrt{\var[Y_{ij}]^2  }}\\
    &=\frac{\delta_{it}(\mu_{\beta}\sigma_{\theta})^2+\delta_{jl}(\mu_{\theta}\sigma_{\beta})^2+ \delta_{it}\delta_{jl}(\mu_{\theta}\mu_{\beta}+(\sigma_{\theta}\sigma_{\beta})^2)}{\mu_{\theta}\mu_{\beta}+ (\mu_{\beta}\sigma_{\theta})^2+(\mu_{\theta}\sigma_{\beta})^2+ (\sigma_{\theta}\sigma_{\beta})^2}. 
\end{align*}
\end{proof}
\subsection{Finding the hyperparameters given the moments}
\begin{proposition}\label{prop:multvarprior}
Given that we know $K$, $\E[Y_{ij}]$, $\var[Y_{ij}]$, $\rho_1$ and $\rho_2$ the following equations hold and can be used for determining the hyperparameters:
\begin{align}
    \sigma_\theta \sigma_\beta &= \frac{\var[Y_{ij}]}{\E[Y_{ij}]}\sqrt{\rho_1 \rho_2} \label{eq:multvarprior}\\
    \left(\frac{\sigma_\beta}{\mu_\beta}\right)^2 &= K\frac{\var[Y_{ij}]}{\E[Y_{ij}]^2}\rho_2 \\
    \left(\frac{\sigma_\theta}{\mu_\theta}\right)^2 &= K\frac{\var[Y_{ij}]}{\E[Y_{ij}]^2}\rho_1 \\
    \rho_1\left(\frac{\sigma_\beta}{\mu_\beta}\right)^2 &= \rho_2 \left(\frac{\sigma_\theta}{\mu_\theta}\right)^2.
\end{align}
\end{proposition}
\begin{proof}
We can rewrite the columns correlation $\rho_1$ and row correlation $\rho_2$ equations from Proposition~\ref{prop:corr} as
\begin{align}
    \rho_1 \frac{\var[Y_{ij}]}{K}&=(\mu_{\beta}\sigma_{\theta})^2 \label{eq:varrho_1} \\
    \rho_2 \frac{\var[Y_{ij}]}{K}&=(\mu_{\theta}\sigma_{\beta})^2. \label{eq:varrho_2}
\end{align}
Multiplying them together we obtain
\begin{align}
\rho_1 \rho_2  \left( \frac{\var[Y_{ij}]}{K} \right)^2 &= \underset{\frac{\E[Y_{ij}]}{K}}{\underbrace{(\mu_{\beta}\mu_{\theta})}}^2(\sigma_{\beta}\sigma_{\theta})^2 \nonumber \\ 
\implies \rho_1 \rho_2  \left( \frac{\var[Y_{ij}]}{\E[Y_{ij}]} \right)^2 &=(\sigma_{\beta}\sigma_{\theta})^2. \label{eq:multvarprior2}
\end{align}
Taking the root of \eqref{eq:multvarprior2} completes the proof for \eqref{eq:multvarprior}. 

Now, using \eqref{eq:multvarprior2}, \eqref{eq:varrho_1} and \eqref{eq:varrho_2}, we will obtain the value of $ \frac{(\sigma_{\beta}\sigma_{\theta})^2}{(\mu_{\beta}\sigma_{\theta})^2}$ and $ \frac{(\sigma_{\beta}\sigma_{\theta})^2}{(\mu_{\theta}\sigma_{\beta})^2}$:
\begin{align}
    \frac{(\sigma_{\beta}\sigma_{\theta})^2}{(\mu_{\beta}\sigma_{\theta})^2} = \frac{\sigma_{\beta}^2}{\mu_{\beta}^2} &= \rho_1 \rho_2  \left( \frac{\var[Y_{ij}]}{\E[Y_{ij}]} \right)^2 \frac{K}{\rho_1 \var[Y_{ij}]} = \rho_2 K \frac{\var[Y_{ij}]}{\E[Y_{ij}]^2}  \\ 
    \frac{(\sigma_{\beta}\sigma_{\theta})^2}{(\mu_{\theta}\sigma_{\beta})^2} = \frac{\sigma_{\theta}^2}{\mu_{\theta}^2} &= \rho_1 \rho_2  \left( \frac{\var[Y_{ij}]}{\E[Y_{ij}]} \right)^2 \frac{K}{\rho_2 \var[Y_{ij}]} = \rho_1 K  \frac{\var[Y_{ij}]}{\E[Y_{ij}]^2}.
\end{align}

Finally, dividing \eqref{eq:varrho_1} by \eqref{eq:varrho_1} we obtain the last result that completes the proof.
\end{proof}

\begin{proposition}\label{prop:estimatek}
Given that we know $\E[Y_{ij}]$, $\var[Y_{ij}]$, $\rho_1$ and $\rho_2$, we can obtain the number of latent factors $K$ and coefficient of variation ($\frac{\sigma}{\mu}$) of the priors of the Poisson factorization model that would generate data to match those moments:

\begin{align}
    K  &= \frac{(1-(\rho_1+\rho_2))\var[Y_{ij}]-\E[Y_{ij}]}{\rho_1\rho_2} \left( \frac{\E[Y_{ij}]}{\var[Y_{ij}]} \right)^2 \label{eq:latent_k} \\
    \left(\frac{\sigma_\theta}{\mu_\theta}\right)^2 &= \frac{(1-(\rho_1+\rho_2))\var[Y_{ij}]-\E[Y_{ij}]}{\rho_2 \var[Y_{ij}] } \label{eq:latent_cv1} \\
    \left(\frac{\sigma_\beta}{\mu_\beta}\right)^2 &= \frac{(1-(\rho_1+\rho_2))\var[Y_{ij}]-\E[Y_{ij}]}{\rho_1 \var[Y_{ij}] } \label{eq:latent_cv2}.
\end{align}
\end{proposition}
\begin{proof}
From Proposition~\ref{prop:kvare}, we can rewrite the expression for the variance as
$$ \var[Y_{ij}] = \E[Y_{ij}]+K\underset{\rho_2 \frac{\var[Y_{ij}]}{K}}{(\underbrace{\mu_{\theta}\sigma_{\beta}})^2}+K\underset{\rho_1 \frac{\var[Y_{ij}]}{K}}{(\underbrace{\mu_{\beta}\sigma_{\theta}})^2}+K(\sigma_{\theta}\sigma_{\beta}).
$$
Now, using \eqref{eq:varrho_1} and \eqref{eq:varrho_1} to substitute in the previous equation we obtain:
$$ \var[Y_{ij}] = \E[Y_{ij}]+(\rho_1+\rho_2)\var[Y_{ij}]+K(\sigma_{\theta}\sigma_{\beta})^2.
$$

Using the squared version of \eqref{eq:multvarprior} from Proposition~\ref{prop:multvarprior}, we know that $K(\sigma_{\theta}\sigma_{\beta})^2 = K\left( \frac{\var[Y_{ij}]}{\E[Y_{ij}]} \right)^2 \rho_1 \rho_2$. This results in \begin{align}
    K\left( \frac{\var[Y_{ij}]}{\E[Y_{ij}]} \right)^2 \rho_1 \rho_2 &= (1-(\rho_1+\rho_2))\var[Y_{ij}]-\E[Y_{ij}] \nonumber \\
    \implies K &= \frac{(1-(\rho_1+\rho_2))\var[Y_{ij}]-\E[Y_{ij}]}{\rho_1 \rho_2 }\left( \frac{\E[Y_{ij}]}{\var[Y_{ij}]} \right)^2. \label{eq:result_latent}
\end{align}
The remaining results are obtained by substituting \eqref{eq:result_latent} in \eqref{eq:varrho_1} and \eqref{eq:varrho_2}.
\end{proof}

\subsubsection{Gamma priors}

For gamma priors parameterized with shape (a,c) and rate (b,d) we have:
$$
    \mu_\theta = \frac{a}{b} ;     \sigma_\theta^2 = \frac{a}{b^2} ; \mu_\beta = \frac{c}{d} ;     \sigma_\beta^2 = \frac{c}{d^2}
$$
and the cofficient of variation is given by
\begin{align}
    \frac{\sigma_\theta^2}{\mu_{\theta}^2} &= \frac{a}{b^2}\frac{b^2}{a^2}=\frac{1}{a}  \label{eq:gamcv1}\\
    \frac{\sigma_\beta^2}{ \mu_\beta^2} &=\frac{d^2}{c^2} \frac{c}{d^2}=\frac{1}{c}. \label{eq:gamcv2}
\end{align}
Thus, Eq.~\ref{eq:gamcv1} and \ref{eq:gamcv1} establish a close form relationship between the shape hyperparameters of Gamma distributed latent variables in Poisson MF and moments of the marginal distribution of the data. This means that any assumption that the expert might have about those moments on the data can be readily translated into appropriate values for the prior specification.

In conclusion, given the chosen moments, the prior especification of Gamma-Poisson MF model reduces to one degree of freedom, given that the latent dimensionality, and shape hyperparameters are determined. The only two hyperparameters left are the rate/scale, although they would be restriced to be obey a relationship with functional form 
$$b \propto \frac{1}{d}.$$

\begin{proposition}\label{prop:gamma priors}
Given that we know the moments $\E[Y_{ij}]$, $\var[Y_{ij}]$, $\rho_1$ and $\rho_2$, we can obtain the scale parameters of the Gamma priors speficied as $F(\mu_{\theta},\sigma_{\theta}^2)=\gam(a,b)$ and $F(\mu_{\beta},\sigma_{\beta}^2)=\gam(c,d)$ in the Gamma-Poisson factorization model such that the prior predictive moments would match those given moments.

\begin{align}
    \frac{1}{a} &= \frac{(1-(\rho_1+\rho_2))\var[Y_{ij}]-\E[Y_{ij}]}{\rho_2 \var[Y_{ij}] } \label{eq:latent_cv12} \\
    \frac{1}{c} &= \frac{(1-(\rho_1+\rho_2))\var[Y_{ij}]-\E[Y_{ij}]}{\rho_1 \var[Y_{ij}] } \label{eq:latent_cv22}
\end{align}
\end{proposition}
\begin{proof}
Immediate from Proposition~\ref{prop:estimatek} and the parameterization of Gamma distribution discussed.
\end{proof}

Also we can rewrite Eq.~\ref{eq:multvarprior} with Gamma parameterization to obtain 
\begin{align}
    \frac{a}{b^2}\frac{c}{d^2} &= \sigma_\theta^2 \sigma_\theta^2 =\left( \frac{\var[Y_{ij}]}{\E[Y_{ij}]} \right)^2 \rho_1 \rho_2  \nonumber \\
    \implies (bd)^2 &= \left( \frac{\E[Y_{ij}] }{\var[Y_{ij}]} \right)^2 \frac{ac}{\rho_1 \rho_2} \nonumber \\
    \implies bd &=  \frac{\E[Y_{ij}] }{\var[Y_{ij}]} \sqrt{ \frac{ac}{\rho_1 \rho_2}}.
\end{align}

\section{Derivation of Analytic Solution for Compound Poisson Matrix Factorization}

We will work with the Exponential Dispersion models (EDM) family of observation that makes Compound Poisson matrix factorization models. Keeping the same notation of the previous section, but adding variable $N_{ui}$ as a Poisson distributed latent count factor of the ED model. With abuse of notation, for example this model allow for observations of the type $Y_{ij} = \sum_{i=1}^{N_{ij}} \mathcal{N}(1,1)$, where $N_{ij}$ is a Poisson random variable, extending Poisson factorization to the domain of real valued observations. Also, from the additive properties ~\cite{JorgensenEDM,BasbugE16} of EDM models, $Y_{ij} = \sum_{i=1}^{N_{ij}} Y_{ijk} $ with $ Y_{ijk} \overset{\text{iid}}{\sim} \ed(w,\kappa)$ is equivalent to $Y_{ij} \sim \ed(w,\kappa N_{ij}) $. Thus, the Compound Poisson Matrix Factorization (CPMF) model we use is defined as
\begin{align*}
    Y_{ij} & \sim \ed(w,\kappa N_{ij})\\
    N_{ij} & \sim \poi(\sum_k \theta_{ik}\beta_{jk}) \\
    \theta_{ik} &\sim F(\mu_{\theta},\sigma_{\theta}^2)\\ 
    \beta_{jk} &\sim F(\mu_{\beta},\sigma_{\beta}^2),
\end{align*}
where $p(Y_{ij} |N_{ij} ; w,\kappa )=\text{exp}(Y_{ij}w-\kappa N_{ij}\psi(w))h(Y_{ij},\kappa N_{ij})$, $\E[Y_{ij} | N_{ij} ; w,\kappa ]=\kappa N_{ij} \psi'(w)$ and $\var[Y_{ij} |  N_{ij} ; w,\kappa]=\kappa N_{ij} \psi{''}(w)$.

\subsection{Mean, variance, covariance and correlation}

\begin{proposition}\label{prop:edmexpvar}
For any combination of valid values for the indexes $i$,$j$, the following equations hold:
\begin{enumerate}\label{prop:kvare2}
    \item $\E[Y_{ij}] = \kappa \psi{'}(w) K \mu_{\theta}\mu_{\beta}$;
    \item $\var[Y_{ij}] =\kappa \psi{''}(w) K \mu_{\theta}\mu_{\beta}+(\kappa \psi{'}(w))^2K[\mu_{\theta}\mu_{\beta}+ (\mu_{\beta}\sigma_{\theta})^2+(\mu_{\theta}\sigma_{\beta})^2+ (\sigma_{\theta}\sigma_{\beta})^2]$.
\end{enumerate}
\end{proposition}
\begin{proof}
By the law of total expectation and the properties of the mean of ED family, $\E[Y_{ij}] = \E[ \kappa   \psi'(w) N_{ij} ] $, which simplifies to $\E[Y_{ij}] = \kappa   \psi'(w) \E[N_{ij} ] $ and from Proposition~\ref{prop:expvar} we know the expected value of the latent Poisson count $N_{ij}$, concluding that $\E[Y_{ij}] = \kappa \psi{'}(w) K \mu_{\theta}\mu_{\beta}$.
Using the law of total variance $\var[Y_{ij}] = \E[\var[Y_{ij}|N_{ij} ] ]+\var[\E[ Y_{ij}|N_{ij}]$, that simplifies to 
$$\var[Y_{ij}] = \kappa \psi{''}(w)  \E[ N_{ij}  ]+[\kappa \psi'(w)]^2 \var[ N_{ij} ],$$ and again substituting Proposition~\ref{prop:expvar} completes the proof. 
\end{proof}

\begin{proposition}\label{prop:edmcovcorr}
The prior predictive correlation is given by:

$$ \rho[Y_{ij},Y_{tl}] = \begin{cases} 
0,\text{if }i \neq t \And j \neq l\\
1,\text{if }i = t \And j = l \\
\rho_1 ,\text{if }i = t \And j \neq l \\
\rho_2 ,\text{if }i \neq t \And j = l
\end{cases},$$
with
\begin{align*}
\rho_1 &= \frac{K [\kappa \psi'(w)]^2}{\var[Y_{ij}]}(\mu_{\beta}\sigma_{\theta})^2 \\
\rho_2 &= \frac{K [\kappa \psi'(w)]^2}{\var[Y_{ij}]}(\mu_{\theta}\sigma_{\beta})^2 .
\end{align*}
\end{proposition}
\begin{proof}
Starting with the covariance, we apply the law of total covariance to obtain
\begin{align}
	\cov[Y_{ij},Y_{tl}] = \delta_{it}\delta_{jl}\kappa \psi{''}(w)  \E[ N_{ij}]+[\kappa \psi{'}(w)]^2\cov[N_{ij},N_{tl}]. \label{eq:edcov}
\end{align}
When all the indices coincide this will be equal to the variance, thus leading to a correlation of 1, when all the indices are different this will lead to correlation of zero. This means that the main difference between the prior predictive correlation structure of CPMF and PMF will be where there is rows and columns correlation, that we will be able to calculate because we know the covariance $\cov[N_{ij},N_{tl}]$ from Proposition~\ref{prop:cov}, namely
\begin{align*}
	\rho_1 &= \frac{[\kappa \psi{'}(w)]^2\cov[N_{ij},N_{il}]}{\var[Y_{i,j}]} %
	       = \frac{K [\kappa \psi{'}(w)]^2(\mu_{\beta}\sigma_{\theta}^2 )}{\var[Y_{i,j}]},\\
	\rho_2 &= \frac{[\kappa \psi{'}(w)]^2\cov[N_{ij},N_{tj}]}{\var[Y_{i,j}]} %
	       = \frac{K [\kappa \psi{'}(w)]^2(\mu_{\theta}\sigma_{\beta}^2 )}{\var[Y_{i,j}]}.
\end{align*}
\end{proof}

\subsection{Finding the hyperparameters given the moments}

\begin{proposition}\label{prop:compound_estimatek}
For Compound Poisson MF, given that we know $\E[Y_{ij}]$, $\var[Y_{ij}]$, $\rho_1$ and $\rho_2$, we can obtain the number of latent factors $K$ of the model that would generate data to match those moments: 

\begin{align}
    K  &= \frac{(1-(\rho_1+\rho_2))\var[Y_{ij}]-\left(\kappa \psi'(w)+\frac{\psi{''}(w)}{\psi'(w)} \right) \E[Y_{ij}]}{\rho_1\rho_2} \left( \frac{\E[Y_{ij}]}{\var[Y_{ij}]} \right)^2. \label{eq:edm_latent_k_1} 
\end{align}
\end{proposition}
\begin{proof}
We will start by showing that 
\begin{align}
    \sigma_\theta \sigma_\beta &= \frac{\var[Y_{ij}]}{\E[Y_{ij}]\kappa \psi'(w)}\sqrt{\rho_1 \rho_2}. \label{eq:edm_mult_std}
\end{align}
Take $\rho_1$ and $\rho_2$ and multiply them to obtain:
\begin{align*}
    \rho_1 \rho_2 &= \left( \frac{K [\kappa \psi'(w)]^2}{\var[Y_{ij}]}\right)^2 (\mu_{\theta}\mu_{\beta})^2(\sigma_{\theta}\sigma_{\beta})^2. 
\end{align*}
From Proposition~\ref{prop:edmexpvar}, we know $K\mu_{\theta}\mu_{\beta}=\frac{\E[Y_{ij}] }{ \kappa \psi{'}(w)}$, so we can substitute that on the previous equation obtaining 
\begin{align*}
    \rho_1 \rho_2 &= \left( \frac{ [\kappa \psi'(w)]^2}{\var[Y_{ij}]}\right)^2 \left( \frac{\E[Y_{ij}] }{ \kappa \psi{'}(w)} \right)^2 (\sigma_{\theta}\sigma_{\beta})^2 %
        = [\kappa \psi'(w)]^2 \left( \frac{\E[Y_{ij}] }{ \var[Y_{ij}]} \right)^2 (\sigma_{\theta}\sigma_{\beta})^2. 
\end{align*}

Now let us turn our attention to $\var[Y_{ij}]$ and re-write it using the previous results together with Proposition~\ref{prop:edmcovcorr} for the correlations, and Proposition~\ref{prop:edmexpvar} for the mean:
\begin{align*}
   \var[Y_{ij}] &=  \psi{''}(w) \underformula{\kappa K \mu_{\theta}\mu_{\beta}}{\frac{\E[Y_{ij}]}{\psi{'}(w)}}
   +\underformula{K[\kappa \psi{'}(w)]^2\mu_{\theta}\mu_{\beta}}{ \kappa \psi{'}(w)\E[Y_{ij}] }
   + K[\kappa \psi{'}(w)]^2[(\mu_{\beta}\sigma_{\theta})^2+(\mu_{\theta}\sigma_{\beta})^2+ (\sigma_{\theta}\sigma_{\beta})^2] \\
   &= \left( \frac{\psi{''}(w)}{\psi{'}(w)}+\kappa \psi{'}(w) \right) \E[Y_{ij}]
   +\underformula{K[\kappa \psi{'}(w)]^2(\mu_{\beta}\sigma_{\theta})^2}{\rho_1 \var[Y_{ij}]}
   \\
   &+\underformula{K[\kappa \psi{'}(w)]^2(\mu_{\theta}\sigma_{\beta})^2}{\rho_2 \var[Y_{ij}]}
   +K[\kappa \psi{'}(w)]^2(\sigma_{\theta}\sigma_{\beta})^2 \\
   &=\left( \frac{\psi{''}(w)}{\psi{'}(w)}+\kappa \psi{'}(w) \right) \E[Y_{ij}]+(\rho_1+\rho_2)\var[Y_{ij}]
   +K \underformula{[\kappa \psi{'}(w)]^2(\sigma_{\theta}\sigma_{\beta})^2}{ \rho_1 \rho_2  \left( \frac{\var[Y_{ij}]] }{ \E[Y_{ij} } \right)^2}\\
   &=\left( \frac{\psi{''}(w)}{\psi{'}(w)}+\kappa \psi{'}(w) \right) \E[Y_{ij}]+(\rho_1+\rho_2)\var[Y_{ij}]
   +K \rho_1 \rho_2  \left( \frac{\var[Y_{ij}]] }{ \E[Y_{ij} } \right)^2.
\end{align*}
Reorganizing the terms and isolating $K$ we obtain the final formula
\begin{align*}
    K  &= \frac{(1-(\rho_1+\rho_2))\var[Y_{ij}]-\left(\kappa \psi'(w)+\frac{\psi{''}(w)}{\psi'(w)} \right) \E[Y_{ij}]}{\rho_1\rho_2} \left( \frac{\E[Y_{ij}]}{\var[Y_{ij}]} \right)^2 \label{eq:edm_latent_k}. 
\end{align*}

\end{proof}

\section{Generic Probabilistic Matrix Factorization}
\label{sec:genericpmf}

Consider the model defined in Eq.~\ref{eq:generic_pmf}, with priors $F(\mu_{\theta},\sigma_{\theta}^2)$ and $ F(\mu_{\beta},\sigma_{\beta}^2) $, and observation model $ F_Y$, defined as:
\begin{align*}
    \theta_{ik} &\sim F(\mu_{\theta},\sigma_{\theta}^2), \quad \beta_{jk} \sim F(\mu_{\beta},\sigma_{\beta}^2) \nonumber \\ 
    Y_{ij} & \sim F_Y(\sum_{k = 1}^K \theta_{ik}\beta_{jk}), \text{ with } \E[Y_{ij}]=\sum_{k = 1}^K \theta_{ik}\beta_{jk}.
\end{align*}

\begin{proposition}\label{prop:gen_moments}
For any entry of the matrix $ \mathbf{Y} =\{ Y_{ij} \}$, the mean and variance is given by:
\begin{align}
    \E[Y_{ij}] &= K\mu_{\theta}\mu_{\beta}  \\
    \var[Y_{ij}] &= \E[\var[Y_{ij}|\theta, \beta]] \nonumber \\
    &+K[(\mu_{\beta}\sigma_{\theta})^2 +(\mu_{\theta}\sigma_{\beta})^2+ (\sigma_{\theta}\sigma_{\beta})^2] 
\end{align}
\end{proposition}
\begin{proof}
The equation for $\E[Y_{ij}]$ is obtained from the same steps shown for the case PMF and CPMF, and the fact that $\E[Y_{ij}]=\sum_{k = 1}^K \theta_{ik}\beta_{jk}$ by definition of the model. 

The equation for $\var[Y_{ij}]$ is obtained also from the same steps shown before, but with the additional restriction that $ \E[\var[Y_{ij}|\theta, \beta]]$ is model dependent and can not be simplified further, once we apply the law of total variance.
\end{proof}

\begin{proposition}\label{prop:gen_moments_2}
For any pair of entries $ Y_{ij} $ and $Y_{tl}$ of the  matrix $\mathbf{Y}$, their correlation is given by:
\begin{align}
     \rho[Y_{ij},Y_{tl}] &= \begin{cases} 
0,\text{if }i \neq t \And j \neq l\\
1,\text{if }i = t \And j = l \\
\rho_1 ,\text{if }i = t \And j \neq l \\
\rho_2 ,\text{if }i \neq t \And j = l
\end{cases}
\end{align}
\begin{align}
    \text{ with }\rho_1 &= \frac{K(\mu_{\beta}\sigma_{\theta})^2}{\var[Y_{ij}]} \nonumber \\
    \rho_2 &=\frac{K(\mu_{\theta}\sigma_{\beta})^2}{\var[Y_{ij}]} \nonumber
\end{align}
\end{proposition}
\begin{proof}
Using the law of total covariance:
\begin{align*}
    \cov[Y_{ij},Y_{tl}] & = \E[\cov[Y_{ij},Y_{tl}|\theta,\beta]]+\cov[\E[Y_{ij}|\theta,\beta],\E[Y_{tl} |\theta,\beta] ] \\
    &= \delta_{it}\delta_{jl}\E[\var[Y_{ij}|\theta,\beta]]+\cov[\sum_k \theta_{ik}\beta_{jk},\sum_k \theta_{tk}\beta_{lk}] \\
     &=\delta_{it}\delta_{jl}\E[\var[Y_{ij}|\theta,\beta]]+ K[  \delta_{it}(\mu_{\beta}\sigma_{\theta})^2+\delta_{jl}(\mu_{\theta}\sigma_{\beta})^2+ \delta_{it}\delta_{jl}(\sigma_{\theta}\sigma_{\beta})^2].
\end{align*}

Analyzing this expression for the different cases we obtain:
\begin{itemize}
    \item $i \neq t \And j \neq l$: $\cov[Y_{ij},Y_{tl}]=0$
    \item $i = t \And j = l$: $\cov[Y_{ij},Y_{tl}]=\cov[Y_{ij},Y_{ij}]=\var[Y_{ij}]$
    \item $i = t \And j \neq l$: $\cov[Y_{ij},Y_{tl}]=K(\mu_{\beta}\sigma_{\theta})^2$
    \item $i \neq t \And j = l$: $\cov[Y_{ij},Y_{tl}]=K(\mu_{\theta}\sigma_{\beta})^2$
\end{itemize}

Applying the definition of correlation $\rho[Y_{ij},Y_{tl}]=\frac{\cov[Y_{ij},Y_{tl}]}{\sqrt{\var[Y_ij]\var[Y_tl]}}=\frac{\cov[Y_{ij},Y_{tl}]}{\var[Y_ij]}$ we obtain the final equation in the proposition.
\end{proof}

Given Propositions \ref{prop:gen_moments} and \ref{prop:gen_moments_2} and some target values for the moments, we can directly solve for the number of latent factors $K$. Denoting
$ \tau = 1-(\rho_1+\rho_2)$, we obtain our main result

\begin{theorem}\label{prop:num_latent_gen}
  \begin{align}
      K  &= \frac{ \tau\var[Y_{ij}]-\E[\var(Y_{ij}|\theta, \beta)]}{\rho_1\rho_2} \left( \frac{\E[Y_{ij}]}{\var[Y_{ij}]} \right)^2.
  \end{align}
\end{theorem}
\begin{proof}
We can rewrite the expression of variance using the terms $\rho_1$ and $\rho_2$ resulting in
\begin{align*}
\var[Y_{ij}] &= \E[\var[Y_{ij}|\theta, \beta]]+(\rho_1+\rho_2)\var[Y_{ij}]+K(\sigma_{\theta}\sigma_{\beta})^2
\end{align*}

And given that $ K(\sigma_{\theta}\sigma_{\beta})^2 = = K\rho_1\rho_2\left( \frac{\var[Y_{ij}]}{\E[Y_{ij}]} \right)^2 $ and $ \tau = 1-(\rho_1+\rho_2)$ we obtain 

\begin{align*}
K\rho_1\rho_2\left( \frac{\var[Y_{ij}]}{\E[Y_{ij}]} \right)^2 &= (1-(\rho_1+\rho_2))\var[Y_{ij}]-\E[\var[Y_{ij}|\theta, \beta]] \\
\implies K&=\frac{\tau\var[Y_{ij}]-\E[\var[Y_{ij}|\theta, \beta]]}{\rho_1\rho_2}\left( \frac{\E[Y_{ij}]}{\var[Y_{ij}]} \right)^2
\end{align*}

\end{proof}

\section{Differentiable Moment's Estimators for Hierarchical Bayesian Models}
\label{sec:generalmethod}

This appendix provides the details on how to compute the gradients for the gradient-based approach described in Section~\ref{sec:model_independent}. To optimize \eqref{eq:discrepancy} with stochastic gradient descent, we require
that $\discrepancy(\cdot)$ and $\hat \T$ are differentiable w.r.t their arguments, and that we can propagate gradient $\nabla_\lambda$ through $\E[g(Y)]$. Next we show how this can be done for a rather general structure of hierarchical Bayesian models with outputs $Y$ and latent variables $\Z$. See Figure~\ref{fig:gradient_computation} for a conceptual illustration of the assumed model structure and the procedure for computing the gradient.
The procedure is based on recursively applying the law of total expectation. 
The unconditional expectation of $g(Y)$ can be obtained by integrating out latent variables $Z$, but since an analytical form of it is not available, 
we proceed by performing a numerical approximation,
where each of the integrals over latent variables $Z_1, \dots Z_l \dots Z_L$ is replaced by a sum over samples from respective (conditional) distributions.
An estimate of 
the required gradient $\nabla_\lambda \E[g(Y)]$ is then obtained by propagating estimates of the gradients
$\nabla_\lambda \E[g(Y)|\z_l]$ and $\nabla_\lambda \log p(\z_l| \dots; \lambda)$ backward through the computation graph.

\begin{figure}[t]
    \centering
\[\begin{tikzcd}
	&& {\text{Bayesian model}} &&& {\text{Gradient computation}} \\
	&&&&& {\nabla_\lambda E[g(Y)]} \\
	{Z_1} & \bullet & \bullet & \bullet && {\nabla_\lambda E[g(Y) | Z_{1}] \nabla_\lambda \text{log}(p(Z_{1}))} \\
	& \vdots & \vdots & \vdots && \cdots \\
	{Z_{L-1}} & \bullet && \bullet && {\nabla_\lambda E[g(Y) | Z_{L-1} ] \nabla_\lambda \text{log}(p(Z_l))} \\
	{Z_L} & \bullet & \bullet & \bullet && {\nabla_\lambda E[g(Y) | Z_L] \nabla_\lambda \text{log}(p(Z_L))} \\
	&& Y &&& {\nabla_\lambda g(Y)}
	\arrow[from=6-2, to=7-3]
	\arrow[from=6-3, to=7-3]
	\arrow[from=6-4, to=7-3]
	\arrow[from=3-2, to=4-2]
	\arrow[from=3-2, to=4-3]
	\arrow[from=3-2, to=4-4]
	\arrow[from=3-3, to=4-2]
	\arrow[from=3-3, to=4-3]
	\arrow[from=3-4, to=4-4]
	\arrow[from=3-4, to=4-3]
	\arrow[from=3-3, to=4-4]
	\arrow[from=3-4, to=4-2]
	\arrow[Rightarrow, from=7-6, to=6-6]
	\arrow[Rightarrow, from=3-6, to=2-6]
	\arrow[from=5-2, to=6-4]
	\arrow[from=5-2, to=6-2]
	\arrow[from=5-2, to=6-3]
	\arrow[from=5-4, to=6-4]
	\arrow[from=5-4, to=6-3]
	\arrow[from=5-4, to=6-2]
	\arrow[Rightarrow, from=6-6, to=5-6]
	\arrow[Rightarrow, from=5-6, to=4-6]
	\arrow[Rightarrow, from=4-6, to=3-6]
\end{tikzcd}\]
    \caption{Conceptual illustration of how the gradients for arbitrary moments can be estimated for Bayesian models with hierarchical structure of latent variables.}   
    \label{fig:gradient_computation}
\end{figure}

\textbf{Top level variables.}
Assuming all the remaining variables $Z_{-1}$ integrated out, the expectation of $g(Y)$ can be expressed by conditioning only on variables $\Z_1$ having no parents, i.e., latent variables being on top in the hierarchy of such a Bayesian model.
For discrete $\Z_1$ 
\begin{equation}
\E[g(Y)] = \sum_{\z_1 \in \ZSPACE_1} \underbrace{\E[g(Y)|\z_1]}_{f^1_\lambda(Y, \z_1) } \cdot p(\z_1; \lambda)
\label{eq:noparents_discrete}
\end{equation}
where  we named
$
f^1_\lambda(Y, \z_1) \equiv \E[g(Y)|\z_1]$ to emphasize that the expectations here behave like ordinary functions. 
The gradient of the expectation is then given as 
$$
\nabla_\lambda \E[g(Y)] = \sum_{\z_1 \in \ZSPACE_1} \nabla_\lambda f(Y, \z_1) \cdot  p(\z_1; \lambda) + f(Y,\z_1) \cdot \nabla_\lambda p(\z_1; \lambda).
$$
Whenever the number of possible discrete values, i.e., $|\ZSPACE_1|$, is too large or infinite (like for example, for Poisson distribution), the exact sum over all possible outcomes in \eqref{eq:noparents_discrete} is replaced with a set of samples
\begin{equation}
\E[g(Y)] \approx \frac{1}{S_1} \sum_{\z_1 \sim p(\z_1; \lambda)} f^1_\lambda(Y, \z_1) \cdot p(\z_{1}; \lambda),
\label{eq:approx_discrete}
\end{equation}
where $S_1$ denotes number of samples.
Then, log derivative trick in (e.g., the DiCE incarnation by \citealt{foerster2018dice}) is used to obtain unbiased estimates for the gradient 
\begin{equation}
\nabla_\lambda \E[g(Y)] \approx \frac{1}{S_1} \sum_{\z_1 \sim p(\z_1; \lambda)} \nabla_\lambda f(Y, \z_1) \cdot  p(\z_1; \lambda) + \underbrace{f(Y,\z_1) \cdot \nabla_\lambda \log p(\z_1; \lambda)}_{\text{log derivative trick}},
\label{eq:approx_discrete_gradient}
\end{equation}
The second term may incurr large variance of the gradient estimator but (even though we do not account for it in our experiments) this can be helped by variance reduction techniques~\citep{mnih2014,mnih2016,tucker2017rebar}.

For continuous $\Z_1$ the expectation 
is approximated using MC
\begin{equation}
\E[g(Y)] \approx \frac{1}{S_1} \sum_{\z_1 \sim p(\z_1;\lambda)} f^1_\lambda(Y, \z_1)
\label{eq:noparents_continous}
\end{equation}
and the gradient can be estimated by reparameterizing $\z_1 := \z_1(\epsilon_1, \lambda)$, i.e., expressing samples $z_1$ as a deterministic function of $\lambda$ and another random variable $\epsilon_1$ coming from a zero-parameter distribution $p_0$. The estimate is then  
\begin{equation}
\nabla_\lambda \E[g(Y)] 
\approx 
\nabla_\lambda \left[
\frac{1}{S_1} \sum_{\epsilon_1 \sim p_0(\epsilon_1)} f^1_\lambda\left(Y,\z_1(\epsilon_1, \lambda)\right) \right]
=
\frac{1}{S_1} \sum_{\epsilon_1 \sim p_0(\epsilon_1)}
\nabla_\lambda f^1_\lambda\left(Y,\z_1(\epsilon_1, \lambda)\right)
\label{eq:noparents_continous_reparametrized}
\end{equation}
where we omit the change-of-variables Jacobian for brevity. 
Even though the technical implementations of the path-wise (reparametrized) gradients~\citep{figurnov2018implicit} may differ from the conceptual presentation above, it does not affect our reasoning
-- usually the details are hidden by interface of an automatic-differentation library.

\textbf{Inner latent variables.}
For both of the above cases ($Z_1$ discrete and continous) we need to be able to calculate the inner gradients of the conditional expectations
$\nabla_\lambda f(Y,\z_1) \equiv \nabla_\lambda \E[g(Y)|\z_1]$, and in general, for any intermediate latent variables $\Z_l$ we need $\nabla_\lambda f(Y,\z_l) \equiv \nabla_\lambda \E[g(Y)|\z_l]$. It can be obtained by further expanding the expectations. For each $l$, we unwrap recursively the expectation by conditioning on variables' $\Z_l$ parents $\Z_{l-1}$. 
In particular, 
for discrete $\Z_l$ we again have 
\begin{equation}
\E[g(Y)|\z_{l-1}] = \sum_{\z_{l} \in \ZSPACE_{l}} \underbrace{\E[g(Y)|\z_{l}]}_{f^l_\lambda(Y, \z_l)} \cdot p(\z_{l}|\z_{l-1}; \lambda).
\label{eq:recursive_discrete}
\end{equation}
where
\eqref{eq:recursive_discrete} corresponds and has the same form as \eqref{eq:noparents_discrete} and therefore the gradient $\nabla_\lambda \E[g(Y)|\z_{l-1}]$ can be obtained similar as in \eqref{eq:approx_discrete_gradient} (with the exception that 
the probability mass function $p(\z_{l}|\z_{l-1}; \lambda)$ depends here also on samples $\z_{l-1}$ of the parents $\Z_{l-1}$).
Respectively, for continuous $\Z_{l}$:
\begin{equation}
\E[g(Y)|\z_{l-1}] \approx 
\frac{1}{S_l} \sum_{\z_l \sim p(\z_l;\lambda)} \underbrace{\E[g(Y)|\z_{l}]}_{f^l_\lambda(Y, \z_l)} 
=
\frac{1}{S_{l}} \sum_{\epsilon_{l} \sim p_0} \E[g(Y)|\z_{l}(\epsilon_{l}, \lambda, \z_{l-1})],
\label{eq:recursive_continous}
\end{equation}
where
\eqref{eq:recursive_continous} corresponds and has the same form as  \eqref{eq:noparents_continous} and allows obtaining its gradient similar as in \eqref{eq:noparents_continous_reparametrized}.
The only notable difference is again that 
the reparametrization $\z_l := \z_l(\epsilon_l, \lambda, z_{l-1})$
depends also on $z_{l-1}$.

\textbf{Direct parents of observed variables.}
By recursively exploring the computation graph,
we eventually arrive at the expectation of $g(Y)$ conditioned directly on its parents. For discrete $Y$ it takes an exact form of
\begin{equation}
\E[g(Y)|\z_L] = \sum_{y \in \YSPACE} \underbrace{g(y)}_{f^L_\lambda(y)} \cdot p(y|\z_L; \lambda),
\label{eq:y_discrete}
\end{equation}
where again we find the same structure as in 
\eqref{eq:noparents_discrete}
and \eqref{eq:recursive_discrete}, and the gradient can be obtained similar to \eqref{eq:approx_discrete_gradient}. Note however that here
$\nabla_\lambda f^L_\lambda(y) = \nabla_\lambda g(y)$ where we assume $\nabla_\lambda g(y)$ to be a known function.
For continuous outputs $Y$, we again resort to the MC approximation
\begin{equation}
\E[g(Y)|\z_L] \approx  \frac{1}{S_y} \sum_{\epsilon_y \sim p_0} g(y(\epsilon_y, \lambda, \z_L)),
\label{eq:y_continous}
\end{equation}
where differentiability w.r.t $\lambda$ is achieved by reparametrizing $y := y(\epsilon_y, \lambda, \z_L)$. 
$S_y$ is a number of MC samples and
$p_0$ is a simple distribution free of hyperparameters $\lambda$. 
Finally, the reparametrized gradient similar to
\eqref{eq:noparents_continous_reparametrized} is given as 
$$
\nabla_\lambda
\E[g(Y)|\z_L] \approx  \frac{1}{S_y} \sum_{\epsilon_y \sim p_0} 
\nabla_\lambda g(y(\epsilon_y, \lambda, \z_L)).
$$

\textbf{Mixed-type variables.}
If the set of variables denoted by $\Z_l$ consist of both continuous and discrete nodes, the above expressions need to be combined by summing over (or sampling from) the discrete variables and using MC approximation for the continuous ones.

\subsection{Example: Derivation for PMF and HPF}
\label{sec:model_indep_pmf}

To demonstrate the rather generic presentation above and to link it to the BMF use-case, we show as an example the MC estimate for the expectation $\E[Y_{ij}]$ of the PMF model
\begin{equation}
\E[Y] \approx \frac{1}{S_\theta \cdot S_\beta} \sum_{\epsilon_\theta \sim p_0} \sum_{\epsilon_\beta  \sim p_0} \E[Y | \theta(\epsilon_\theta, \lambda)^T \beta(\epsilon_\beta, \lambda)],
\label{eq:pmf_ey_uncond}
\end{equation}
where we reparametrize both $\beta$ and $\theta$.
For clarity, we also dropped indices $i$ and $j$ in $Y_{ij}$, $\theta_{i}$ and $\beta_{j}$. 
The internal conditional expectation in \eqref{eq:pmf_ey_uncond} we expand as
\begin{equation}
\E[Y| \theta^T \beta] \approx \frac{1}{C} \sum_{y \sim \poi(\theta^T \beta)} y \cdot \textit{Poisson}(y|\theta^T \beta),
\label{eq:pmf_ey}
\end{equation}
where we applied \eqref{eq:approx_discrete} and $\textit{Poisson}$ denotes Poisson probability mass distribution. 
Similarly, variance of $Y$ we estimate with $\var[Y] = \E[Y^2] - \E[Y]^2$, 
 where  the estimator of $\E[Y^2]$
 we obtain by substituting $y$ with $y^2$ in \eqref{eq:pmf_ey}.

To demonstrate the flexibility of the model-independent algorihm we also apply it on hierarchical Poisson factorization (HPF) model of 
\citet{gopalan2015scalable}, for which we do not have closed-form expressions for the moments. This model adds one level of hierarchy to PMF, and hence matches our general formulation:
\begin{align*}
\theta \sim \gam(a, \xi), \quad \xi \sim \gam(a', a'/b') \\
\beta \sim \gam(c, \eta), \quad \eta \sim \gam(c', c'/d'),
\end{align*}
where the new continuous variables $\xi$, $\eta$ we  reparametrize (and sample) as follows:
\begin{align*}
  \theta:=\theta(\epsilon_\theta, a, \xi), \quad  \xi := \xi(\epsilon_\xi, a', a'/b')\\
  \beta:=\beta(\epsilon_\beta, c, \xi), \quad  \eta := \eta(\epsilon_\eta, c', c'/d'),
\end{align*}

Then, for HPF, \eqref{eq:pmf_ey_uncond} takes the form
$$
\E[Y] \approx
\frac{1}{S_\xi \cdot S_\eta} 
\sum_{\epsilon_\xi \sim p_0} \sum_{\epsilon_\eta  \sim p_0}
\underbrace{
\frac{1}{S_\theta \cdot S_\beta} \sum_{\epsilon_\theta \sim p_0} \sum_{\epsilon_\beta  \sim p_0} \E[Y | \theta(\epsilon_\theta, \lambda, \underline{\eta(\epsilon_\eta, \lambda)})^T 
\beta(\epsilon_\beta, \lambda, \underline{\eta(\epsilon_\eta, \lambda))}]
}_{\E[Y|\xi, \eta]},
$$

\section{Parametrization of PMF}
\label{sec:parametrization}

\begin{figure*}[t]
    \centering
    \includegraphics[width=0.3\textwidth]{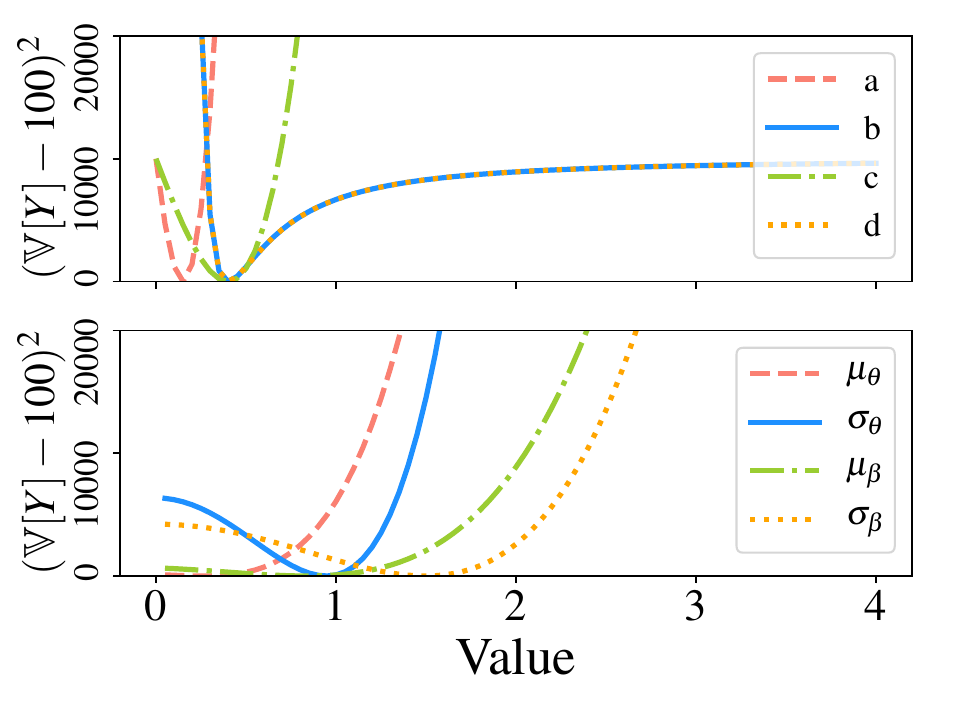}
    \includegraphics[width=0.3\textwidth]{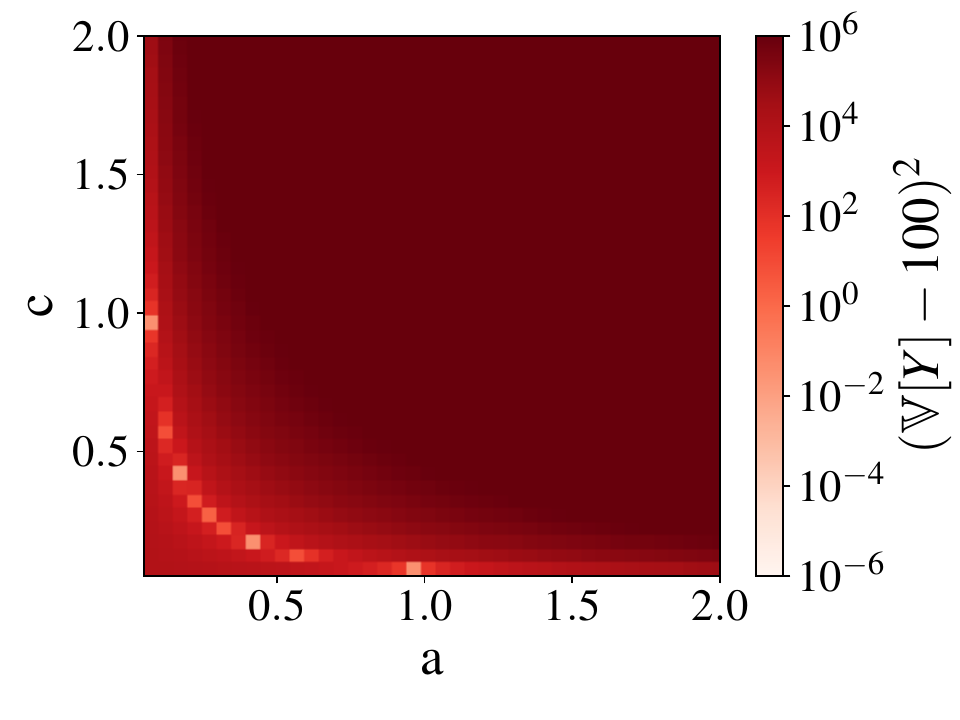}
    \includegraphics[width=0.3\textwidth]{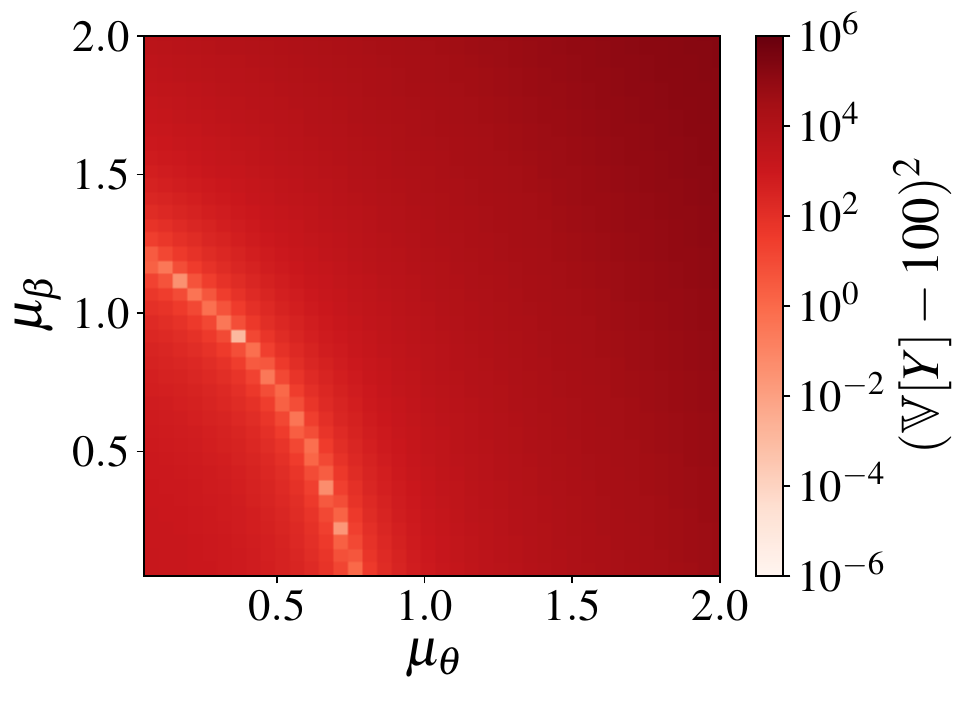}
    \caption{PMF ($K=25$) parametrization using concentrations  $a,c$ and rates $b,d$ vs. means $\mu_\theta, \mu_\beta$ and variances $\sigma_\theta^2, \sigma_\beta^2$:
    1D and 2D projections 
    of the optimization surface for matching $\var[Y]=100$ 
    in neighborhood of the optimal point ($a = 0.16$, $b=0.4$, $c=0.4$, $d=0.4$). 
    }
    \label{fig:surfaces}
\end{figure*}

The performance of the stochastic optimization algorithm depends on the model parametrization, i.e.,
structure of the associated optimization space. Figure~\ref{fig:surfaces} compares two alternative parametrizations for PMF, one in terms of concentrations ($a$, $c$) and rates ($b$, $d$)
and the other in terms of means ($\mu_\theta$, $\mu_\beta$) and variances 
($\sigma_\theta^2$,  $\sigma_\beta^2$). In this case, the latter results in diagonal Fisher information matrix and reduced correlations between the hyperparameters.
This makes optimization easier, seen also as smoother optimization surface.

The question of optimal parameterization for general cases is clearly non-trivial problem, as it is for all optimization problems. One direction for improved parameterization of the prior distributions could build on the approach of \citet{hartmann2019}, \citet{DBLP:conf/icml/TangR19} and \citet{gorinova2019automatic}, going towards a generic invariant to the parameterization of probabilistic programming models such as in~\citet{JMLR:v18:14-467}.

\section{Bias and Variance of Generic Estimators}
\label{sec:experiment_bias_variance}

\begin{figure*}[t]
    \centering
      \includegraphics[width=0.45\textwidth]{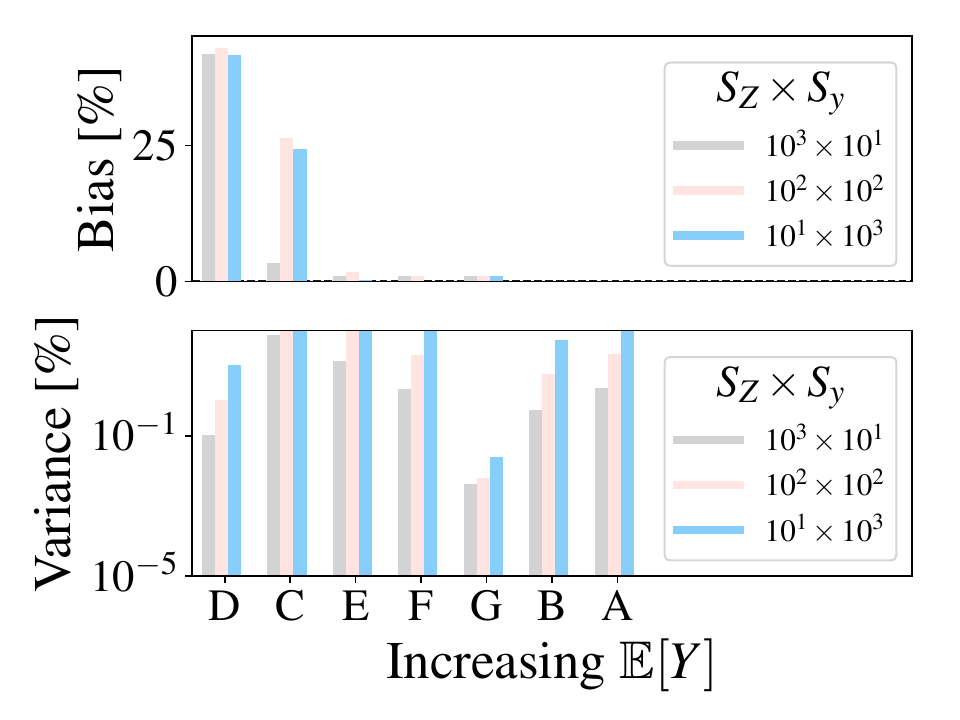}
      \includegraphics[width=0.45\textwidth]{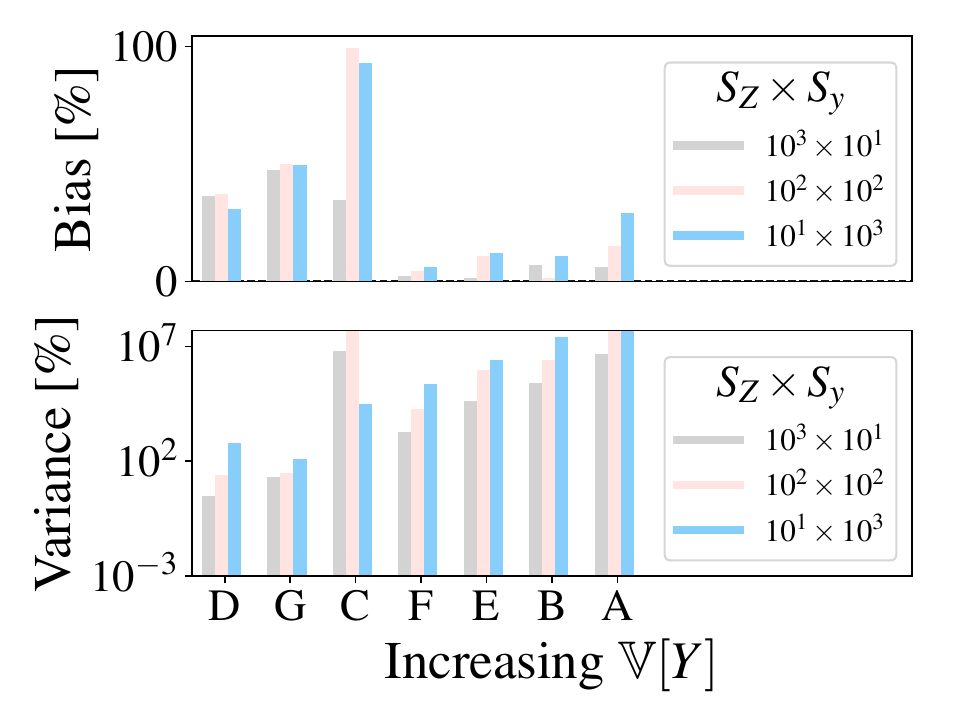}
    \caption{Biases and variances of the expected value (left) and variance (right) estimators for PMF ($K=25$) with different hyperparameters' configurations (Table~\ref{tab:pmf_initializations}).
    $S_Z$ denotes the number of latent variable samples and $S_y$ is the number of samples of observed variable.}      
    \label{fig:nsamples}
\end{figure*}

The model-independent stochastic algorithm relies on Monte Carlo estimates for the virtual statistics as explained in Section 5, instantiated for example cases in Eq. 19 and 20. Such estimates are not necessarily accurate, and hence we evaluate the bias and variance of the estimators against the closed-form solutions of Eq. 3 and 4, based on $10^3$ independent runs.

Figure~\ref{fig:nsamples} compares variances and biases for 
the mean $\hat \E[Y]$ (top) and 
the variance $\hat \var[Y]$ (bottom) estimators for a range of hyperparameter configurations (Tables~\ref{tab:pmf_initializations} and \ref{tab:hpf_initializations} in the Appendix), analyzing the effect of the number of samples used for estimating the statistics. We fix the total number of samples at $S_Z \times S_y = 10^4$, but vary the ratio between 
the number of samples for the latent variables ($S_Z$) and for the observed ones ($S_y$). Both
variance and bias are minimized when more computational resources are spent on sampling the latent variables, suggesting the use of $S_Z \times S_y := 10^3 \times 10$.
The mean estimator has slight bias which is noticeable only for very small values ($\E[Y]<1$), and the variance is usually lower than $50\%$ of the mean estimate.
The variance estimator, however, may have significant bias and variance for values of $\var[Y]$, potentially disturbing the convergence for initializations C, D and G and raising a question regarding better estimators.

\section{Adjustment for model mismatch}
\label{app:mismatch}
The model used in the experiment of Figure~\ref{fig:consitent_rows_rescale} in Section~\ref{sec:mismatch} has elements for each column controlled by a parameter, which is used to induce each columns to sum to a certain number. This model has a multiplicative factor $\gamma_j$ for each latent rate that can be adjusted to induced the desired sum, which can be interprested as allocation factors controlling how much the latent rates are allocated to each column. The original equations for PMF are not applicable in this case, but it is possible to adjust the estimators for empirical mean, variance and covariance, such that we could still use the original equations obtained for PMF. In this section we show details about this result, as well as a simple experiment validating the approach.

Define $Y_{ij} \sim \poi(\eta_{ij}) $ with $\eta_{ij}=\sum_{k = 1}^K \theta_{ik}\beta_{jk}$ as the original PMF matrix, and $\tilde{Y}_{ij} \sim \poi(\gamma_j\eta_{ij})$ with $\bm{\gamma}=[\gamma_1, \ldots, \gamma_M]$ as the modified PMF model with the multiplicative factors. The conditional mean and variance of the modified PMF is  $\E[\tilde{Y}_{ij} | \theta \beta]= \var[\tilde{Y}_{ij} | \theta \beta]= \gamma_j\eta_{ij}=\gamma_j\E[Y_{ij} | \theta \beta]$. The marginal expected and variance are obtained using the previous results resulting in $\E[\tilde Y_{ij}] = e_j = K\gamma_j\mu_\theta\mu_\beta=\gamma_j\E[Y_{ij}]$ and $\var[Y_{ij}] = v_j = e_j+\gamma_j^2\var[\eta_{ij}]=\gamma_j\E[Y_{ij}]+\gamma_j^2\var[\eta_{ij}]$.

 Furthermore we can calculate the correlations between $\tilde Y_{ij}$ and $\tilde Y_{tl}$ , and obtain a similar correlation structure that is zero when all the indices are different ($\rho(Y_{ij},Y_{tl})=0$) and 1 when all indices are the same ($\rho(\tilde Y_{ij},\tilde Y_{ij})=1$), but a column varying correlation for the remaining cases: for the fixed row $i=t$ and different columns $\rho(\tilde Y_{ij},\tilde Y_{tl})=\rho_{jl}^{(1)}=\frac{K\gamma_j\gamma_l}{\sqrt{v_jv_l}}(\mu_\beta\sigma_\theta)^2=\gamma_j\gamma_l\frac{\cov(Y_{ij},Y_{tl})}{\sqrt{v_jv_l}}$, and for fixed columns $j=l$ and different rows  $\rho(\tilde Y_{ij},\tilde Y_{tl})=\rho_{j}^{(2)}=\gamma_j^2\frac{\cov(Y_{ij},Y_{tl})}{v_j}$. Making empirical estimates of these moments are challenging, specially if we intend to use a single observation matrix to estimate those quantities. It is possible to recover the original terms of the equations derived for PMF by
 observing the relationship between the obtained formulas and the moments for PMF.
 
 Define $\hat Y_{ij}^{(1)}=\frac{\tilde Y_{ij}}{\gamma_j}$ and $\hat Y_{ij}^{(2)}=\frac{\tilde Y_{ij}}{\gamma_j^2}$. Now, $\E[ \hat Y_{ij}^{(1)} ]=\hat e=\E[Y_{ij}]=K\mu_\theta\mu_\beta$ which is the equation for the original PMF. Similarly we observe that for PMF, $\var[Y_{ij}]=\E[Y_{ij}]+\var[\eta_{ij}]$, while in the modified model we obtain $\var[\tilde Y_{ij}]=\gamma_j\E[Y_{ij}]+\gamma_j^2\var[\eta_{ij}]$, which implies that $\var[\hat Y_{ij}^{(1)}]=\frac{\var[\tilde Y_{ij}]}{\gamma_j^2}=\frac{\E[Y_{ij}]}{\gamma_j}+\var[\eta_{ij}]$. Since $\E[\hat Y_{ij}^{(2)}]=\frac{\E[Y_{ij}]}{\gamma_j}$, we can write $\var[\hat Y_{ij}^{(1)}]-\E[\hat Y_{ij}^{(2)}]=\var[\eta_{ij}]$, which is one of the terms in the original PMF equation for the variance, as well as being a term that does not vary with the column index. Furthermore, combining with the expected value $\E[\hat Y_{ij}^{(1)}]$ we obtain the result $\hat v=\var[\hat Y_{ij}^{(1)}]-\E[\hat Y_{ij}^{(2)}]+\E[\hat Y_{ij}^{(1)}]=\E[Y_{ij}]+\var[\eta_{ij}]$ which is precisely the equation for $\var[Y_{ij}]$ in the PMF model.
 
Given the bilinearity of the covariance, we have $\cov(Y_{ij}^{(1)}, Y_{tl}^{(1)})= \frac{\cov(\tilde Y_{ij}, \tilde Y_{tl})}{\gamma_j\gamma_l} $, which cancels with the terms $\gamma_j\gamma_l$ in the numerator when applied to the calculate the correlation coefficients. Finally, we obtain the original correlation of PMF using $\frac{\cov(Y_{ij}^{(1)}, Y_{tl}^{(1)})}{\hat v}$. We can now use this equations to obtain empirical estimator using  $\hat Y_{ij}^{(1)}$ and $\hat Y_{ij}^{(2)}$, and apply those in the same formula of the latent dimensionality $K$ obtained for PMF.

The empirical validation setup consist of fixed priors $\text{Gamma}(1,10)$ and $\text{Gamma}(1,0.5)$, latent dimensionality $K\in \{100,250,500,750,1000\}$, and 10 repeated runs for each $K$. In each run, we
sample a matrix $\{Y_{ij}\} \in \mathbb{R}^{1000 \times 1000}$ using the PMF model and another matrix $\{ \tilde Y_{ij} \} \in \mathbb{R}^{1000 \times 1000}$ with the multiplicative factors $\bm{\gamma}=[1,\ldots,\gamma_{max}]$ ($\gamma_{max}=3.5$), using the same prior distribution (but resampling the latent factor for each case) and $K$. For the PMF matrix $\{Y_{ij}\}$ we can directly calculate empirical estimates of the mean, variance and correlations, and estimate the latent dimensionality $\hat K$ using \eqref{eq:latent_var_est}. For the matrix $\{ \tilde Y_{ij} \}$ of the modified model we calculate $\hat Y_{ij}^{(1)}=\frac{\tilde Y_{ij}}{\gamma_j}$ and $\hat Y_{ij}^{(2)}=\frac{\tilde Y_{ij}}{\gamma_j^2}$, and use empirical estimates of variance, expected value and covariance (that is used to calculate the correlations) according to the equations presented in the previous paragraph, finally calculating the estimated latent dimensionality $\hat K$ using the same \eqref{eq:latent_var_est}. Figure~\ref{fig:adjusted} compares the results and shows that in both cases the estimate $\hat K$ recover the true $K$ with similar degree of variability.

\begin{figure}[t]
    \centering
     \includegraphics[width=0.45\textwidth]{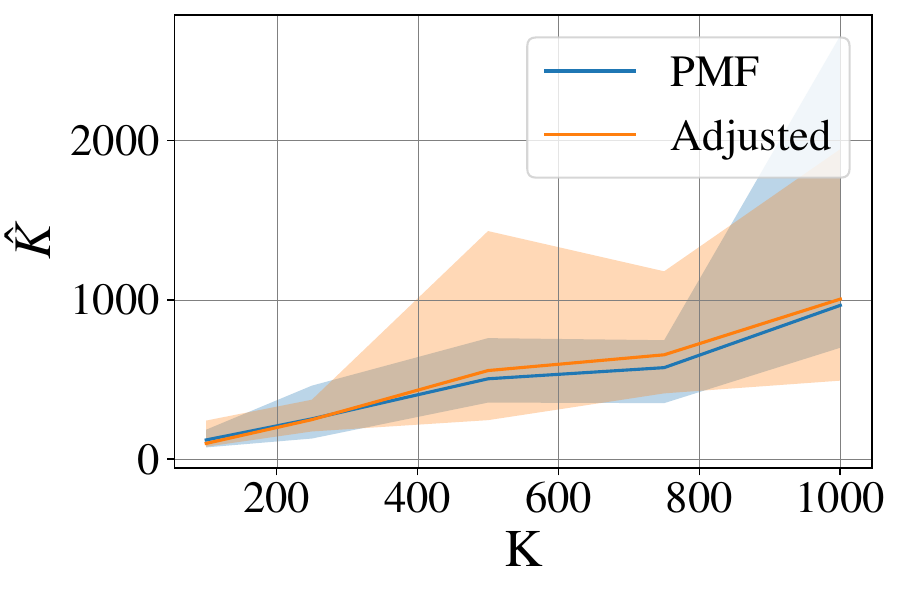}
    \caption{Comparing the estimated latent dimensionality $\hat K$ for PMF and the modified PMF with multiplicative factors, using the adjusted empirical estimates for variance, expected value and correlation (median over 10 runs, with CI=95\%)} 
    \label{fig:adjusted}
\end{figure}

\end{appendices}

\bibliography{biblio}   %

\end{document}